\documentclass[reprint,amsmath,amssymb,aps]{revtex4-2}
\usepackage{array}[=2016-10-06]
\usepackage[export]{adjustbox} 
\usepackage{graphicx}
\usepackage{dcolumn}
\usepackage{bm}
\usepackage{booktabs}
\usepackage{dcolumn}
\usepackage{bm}
\usepackage[bookmarks, colorlinks=true, urlcolor=blue, linkcolor=black, citecolor=blue]{hyperref}
\usepackage{cleveref} 
\crefname{appendix}{Supplemental Material}{Supplemental Materials}
\usepackage{mathtools}
\usepackage{physics}
\usepackage{enumitem}
\usepackage[dvipsnames]{xcolor}
\usepackage{float}
\usepackage{braket}
\usepackage{bbold}
\usepackage{multirow}
\usepackage{makecell} 
\usepackage{amsthm}
\usepackage[normalem]{ulem}

\usepackage{titletoc}
\titlecontents{section}
[0pt]
{\vspace{1.0em}}
{\thecontentslabel . }
{}
{\titlerule*[1pc]{.}\contentspage}
\titlecontents{subsection}
[10pt]
{\vspace{0.2em}}
{\thecontentslabel . }
{}
{\titlerule*[1pc]{.}\contentspage}
\titlecontents{subsubsection}
[20pt]
{\vspace{0.1em}}
{\thecontentslabel . }
{}
{\titlerule*[1pc]{.}\contentspage}

\usepackage{tikz}
\usetikzlibrary{quantikz}
\usetikzlibrary{positioning, shapes.symbols, calc, automata, arrows}

\theoremstyle{plain}
\newtheorem{definition}{Definition}

\newtheorem{theorem}{Theorem}
\newtheorem{lemma}{Lemma}

\newtheorem{corollary}{Corollary}
\newcommand{\cesaro}[1]{\left\langle #1  \right\rangle}

\usepackage[textsize=footnotesize]{todonotes}
\newcounter{todocounter}
\newcommand\lrstacksmall{
	{\mathrel{\vcenter{\mathsurround0pt
				\ialign{##\crcr
					\noalign{\nointerlineskip}\scalebox{0.5}{$\rightarrow$}\crcr
					\noalign{\nointerlineskip}\scalebox{0.5}{$\leftarrow$}\crcr
				}
		}}\hspace{0.08em}}
}

\newcommand{\kb}{{k_\mathrm{B}}}

\definecolor{mblue}{rgb}{0.37,0.51,0.71}
\definecolor{morange}{rgb}{0.88,0.61,0.14}
\definecolor{mgreen}{rgb}{0.56,0.69,0.19}
\definecolor{mred}{rgb}{0.92,0.39,0.21}

\begin{document}

	\title{Information Thermodynamics of Agents: The Work Capacity of Channels with Memory}
	
	\author{Lukas J. Fiderer}
	\email{lukasjfiderer@gmail.com}
	\author{Paul C. Barth}
	\author{Isaac D. Smith}
	\author{Hans J. Briegel}
	\affiliation{Universit\"{a}t Innsbruck, Institut f\"{u}r Theoretische Physik, Technikerstraße 21a, A-6020 Innsbruck, Austria}
	\begin{abstract}
		Predicting future observations plays a central role in machine learning, biology, economics, and many other fields. It lies at the heart of organizational principles such as the variational free energy principle and, based on the second law of thermodynamics, has even been shown to be necessary for reaching the fundamental energetic limits of information processing on a tape. While the usefulness of the predictive paradigm is undisputed, complex adaptive systems that {\em interact} with their environment are more than just predictive machines: they have the power to act upon their environment and cause change. In this work, we develop a framework to analyze the thermodynamics  of information processing in percept--action loops, a model of agent–environment interaction, allowing us to investigate the thermodynamic implications of actions and percepts on equal footing. To this end, we introduce the concept of work capacity, defined as the maximum rate at which an agent can expect to extract work from its environment. Our results reveal that work-efficient agents must balance prediction and forgetting.  This highlights a fundamental departure from the thermodynamics of passive observation, suggesting that prediction and energy efficiency may be at odds in active learning systems.
	\end{abstract}
	\maketitle

	\section{Introduction}
	Percept–action loops are cycles in which an agent perceives its environment, processes and stores information, and acts to influence future perception. They underlie adaptive behavior in both biological and artificial systems. Such loops can be observed across various domains, from humans learning chess, to animals foraging, to artificial intelligence models engaging in dialogue.

	Energetic aspects of processes underlying percept--action loops have been highlighted in diverse contexts. In the brain, energy consumption associated with neural signaling is minimized through efficient coding strategies \cite{lennie2003cost,yu2017energy}. At the molecular level, ribosomes have been shown to perform simple decoding computations at energy costs within an order of magnitude of Landauer’s bound, significantly outperforming even the most advanced supercomputers \cite{kempes2017thermodynamic}. In artificial intelligence, the energetic cost of supercomputers is becoming an increasing concern, particularly in the training of large neural networks \cite{thompson2020computational}, resulting in performance-power trade-offs in large language models \cite{mcdonald2022great}.

	All such processes are ultimately constrained by the second law of thermodynamics, which sets fundamental limits on the efficiency of information processing. Experimental studies on Landauer erasure \cite{berut2012experimental,hong2016experimental} and information engines with measurement-based feedback \cite{koski2014experimental,paneru2018lossless} have begun to probe these limits. Theoretically, they have been analyzed using nonequilibrium fluctuation relations and generalized second laws for information processing with feedback \cite{seifert2012stochastic,parrondo2015thermodynamics,cao2009thermodynamics,sagawa2012nonequilibrium,ito2013information,horowitz2014thermodynamics,hartich2014stochastic,shiraishi2015fluctuation}.
	
	While these works form an important part of the general foundations of information thermodynamics, little is known about the distinctive thermodynamic properties of percept--action loops. This work is motivated by two central questions:
	
	What are the energetic limits of adaptive information processing in percept--action loops? And how should agents be designed to operate efficiently under such constraints?

	\begin{figure}
		\centering		\includegraphics[width=0.8\linewidth]{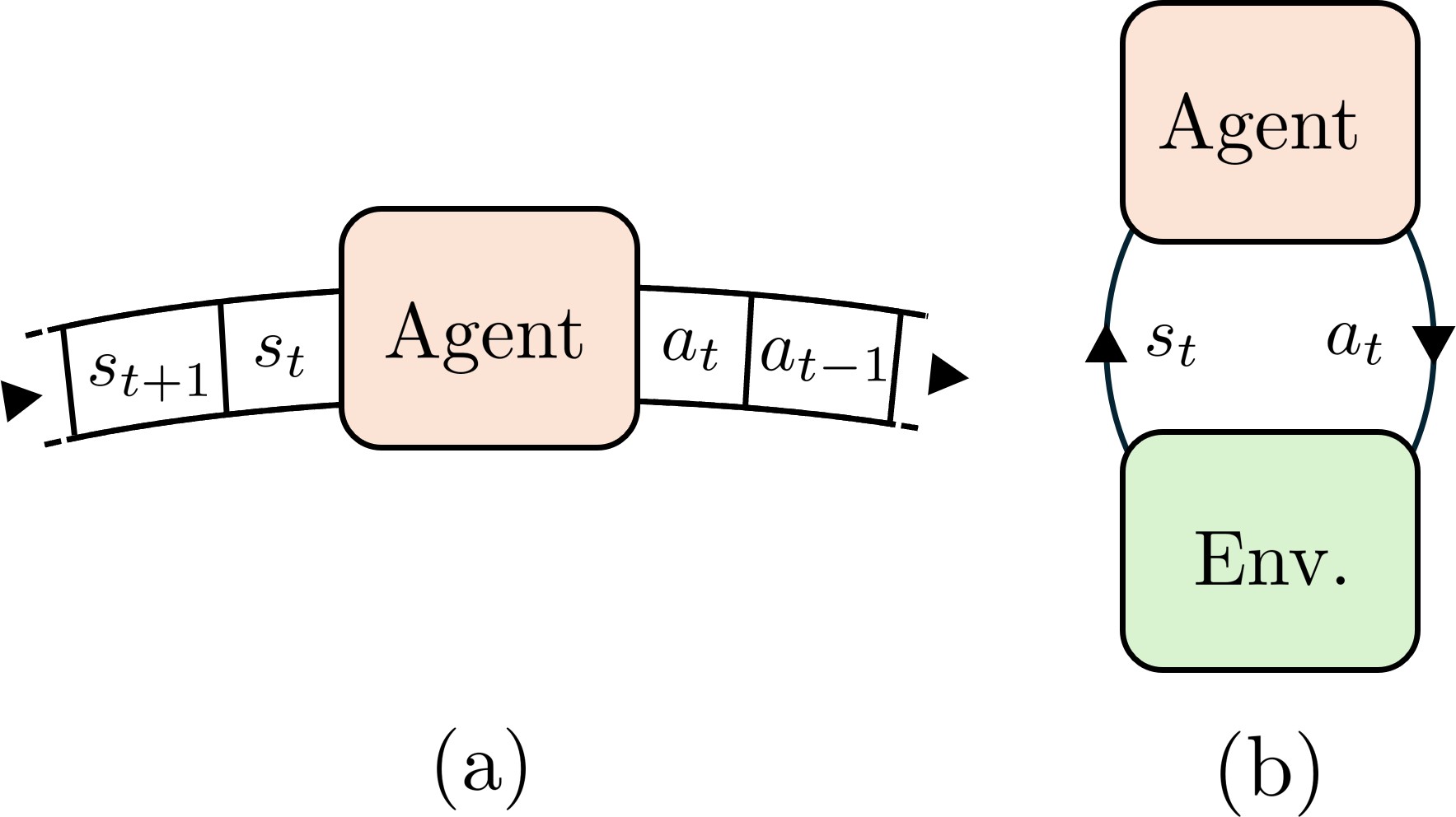}
		\caption{Tape setting (a) and percept--action-loop setting (b). In the tape setting, (a), the agent receives an information carrier encoding an incoming symbol $s_{t-1}$ and processes it by replacing this encoding with that of an outgoing symbol $a_t$ on the same information carrier. Subsequent information carriers are supplied by a pre-existing tape and are therefore not influenced by the outgoing symbols. In the percept--action-loop setting, (b), the agent performs the same physical mapping, from an incoming symbol $s_{t-1}$ to an outgoing symbol $a_t$, on a single information carrier. Here, however, the tape is replaced by an environment (Env.) which receives the information carrier encoding $a_t$ and supplies the next information carrier. Consequently, $s_t$ may depend not only on previous incoming symbols, but also on the current and previous outgoing symbols, thereby closing the feedback loop. Both the agent and the environment may have memory, allowing these dependencies to extend over multiple rounds. For the assumptions underlying the thermodynamic implementation, see Sec.~IV.}
		\label{fig:1}
	\end{figure}
	
	These questions can be addressed by reducing the problem to an information-theoretic model of percept--action loops. By abstracting away implementation-dependent details, we derive energetic bounds that arise solely from the intrinsic cost of information processing, as analyzed through nonequilibrium thermodynamics~\cite{seifert2012stochastic,parrondo2015thermodynamics}. 
	
	As a point of comparison we consider the \emph{tape setting} (see \Cref{fig:1}), where, 
	inspired by Maxwell’s demon, agents sequentially process and modify symbols on a pre-existing tape \cite{mandal2012work, mandal2013maxwell, barato2013autonomous, barato2014unifying, hoppenau2014energetics, merhav2015sequence, boyd2017correlation, garner2017thermodynamics, boyd2018thermodynamics, garner2021fundamental, he2022information, boyd2022thermodynamic, elliott2022quantum, huang2023engines, thompson2025energetic}. In the tape setting, predictable correlations in the input serve as an energetic resource, while generating correlations in the output incurs an energy cost. However, works in the tape setting exclude feedback between agent and environment (but see \cite{zambon2025quantum} for a recent exception, investigating memoryful quantum processes with agent and environment).

	In this work, we model both agent and environment as hidden Markov channels. By combining results from stochastic thermodynamics \cite{seifert2012stochastic,parrondo2015thermodynamics} with the information theory of hidden Markov models \cite{ephraim2002hidden}, we obtain a framework which goes beyond prior work situated in the tape setting by incorporating feedback between agent and environment. This framework is the primary contribution of this work.
	
	The central quantity of this work is the work capacity, defined as the optimal rate of work extraction achievable by any agent. Analogous to communication capacity, it is an intrinsic information-theoretic property of the environment channel.

	A key insight from our investigation of work capacity is that, in percept--action loops, predictive agents are generally inefficient. This is surprising in light of earlier results from the tape setting, where work-efficient agents must be predictive \cite{boyd2018thermodynamics}. We show that this difference persists even when the agent’s memory is constrained to arbitrary finite size.

	In the following sections, we first introduce the percept--action loop framework (\Cref{sec:2}) and define what it means for an agent to be predictive (\Cref{sec:3}). We then present our main results on the work capacity of channels (\Cref{sec:4}) and the design principles of work-efficient agents (\Cref{sec:5}).
	These insights are further developed through a detailed example (\Cref{sec:6}).
	Finally, we discuss directions for future research and conclude by situating our findings in a broader context (\Cref{sec:7}).

	\section{Framework} \label{sec:2}
	Throughout this work, random variables are denoted by capital letters, their realizations by lowercase letters, their alphabets by calligraphic letters, and sequences of random variables---interpreted as random variables on a product space---by $ X_{t:n} = (X_t, X_{t+1}, \dotsc, X_{n-1}) $. Infinite sequences, known as stochastic processes, are written as $ \bm{X} = X_{0:\infty} $, with analogous notation for their realizations, $x_{t:n} = (x_t,x_{t+1},\dotsc, x_{n-1})\in \mathcal{X}^{n-t}$ and $ \bm{x} = x_{0:\infty} \in\mathcal{X}^{\mathbb{N}_0}$.
	
	We consider an agent interacting with an environment in discrete time steps (in the following called \emph{rounds}) indexed by $ t \in \mathbb{N}_0 $, where $ \mathbb{N}_0 $ denotes the nonnegative integers. In each round $ t $, the agent selects an action $ a_t $ and subsequently receives a percept $ s_t $.
	Since both the agent and environment may be stochastic, actions and percepts are modeled as random variables $ A_t $ and $ S_t $  taking values in finite alphabets $ \mathcal{A} $ and $ \mathcal{S} $ (see \Cref{fig:1}b). Embedding the smaller alphabet into the larger lets us set $ \mathcal{A} = \mathcal{S} $, which will be assumed in the following.

	A concept of central importance in describing environments and agents is that of a \emph{channel}. A channel from an alphabet $ \mathcal{X} $ to an alphabet $ \mathcal{Y} $ is a conditional probability distribution $ \mu_{Y|X} $, that is, it assigns to every input $ x \in \mathcal{X} $ a probability distribution $ \mu_{Y|X}(y|x) $ over outputs $ y \in \mathcal{Y} $.
	
	The environment and the agent are described by the channels $ \nu^{\mathrm{env}}_{\bm{S}|\bm{A}} $ and $ \eta^{\mathrm{agt}}_{\bm{A}|\bm{S}} $, respectively. An environment channel assigns to every (infinite) sequence of actions $ \bm{a} = (a_0, a_1, \dotsc) $ a probability distribution $ \nu^{\mathrm{env}}_{\bm{S}|\bm{A}}(\bm{s}|\bm{a}) $ over percept sequences $ \bm{s} = (s_0, s_1, \dotsc) \in \mathcal{S}^{\mathbb{N}_0} $ (and analogously for agent channels, described by $\eta^{\mathrm{agt}}_{\bm{A}|\bm{S}}$).
	
	To qualify as an agent or environment channel, we impose two additional assumptions, which we briefly outline before providing a formal definition. First, the channel must be \emph{causal}, meaning it respects temporal order---future inputs cannot influence past outputs. Second, it must be \emph{finite}, that is, realizable with systems of finite dimension. Like causality, the finiteness assumption is physically motivated: manipulating infinite-dimensional systems can entail unphysical effects such as infinite energy costs. It is also mathematically convenient, as it avoids complications associated with infinite-dimensional spaces (see \Cref{fig:3}). Finally, for our purposes this assumption does not pose a practical limitation, since the performance of optimal infinite-state solutions can be approximated arbitrarily well using finite resources.
	
	These two constraints define what is commonly referred to in the literature as a \textit{finite-state}~\cite{gallager1968information} or \textit{hidden Markov channel}~\cite{ephraim2002hidden} (see  
	\Cref{supp:3}
	of the Supplemental Material for a more in-depth exposition). We now proceed to define both the environment and the agent as such hidden Markov channels:
	\begin{definition}\label{def:environment_channel}
		A channel $\nu^{\mathrm{env}}_{\bm{S}|\bm{
				A}}$ is an \emph{\textbf{environment channel}}, denoted as
		\begin{align}
			\normalfont \texttt{env}\coloneqq\nu^{\mathrm{env}}_{\bm{S}|\bm{A}}, \label{eq:def_env}
		\end{align}
		if there exists a finite set of states $\mathcal{Z}$, a distribution $p^{\mathrm{env}}_{Z_{0}}$ over $\mathcal{Z}$, and a
		transition matrix $\Phi^{\mathrm{env}} = \left(\phi^{\mathrm{env}}(j|i)\right)_{j,i}$ with $i\in\mathcal{A}\times \mathcal{Z}$ and $j\in\mathcal{S}\times \mathcal{Z}$ such that
		\begin{align}
			\nu^{\mathrm{env}}_{\bm{S}|\bm{
					A}}(\bm{s}|\bm{
				a})=\sum_{\bm{z}}\nu^{\mathrm{env}}_{\bm{S}\bm{Z}|\bm{
					A}}(\bm{s},\bm{z}|\bm{a})
		\end{align} 
		where the sum runs over all $\bm{z}\in\mathcal{Z}^{\mathbb{N}_0}$ and  
		\begin{align}
			\nu^{\mathrm{env}}_{\bm{S}\bm{Z}|\bm{
					A}}(\bm{s},\bm{z}|\bm{a})\coloneqq p^{\mathrm{env}}_{Z_{0}}(z_{0})\prod_{t=0}^{\infty}\phi^{\mathrm{env}}
			\left(s_t,z_{t+1}|a_t,z_t\right).\label{eq:environment_model_channel}
		\end{align}
		Then, the tuple
		\begin{align}
			\normalfont{\texttt{env}}_{\mathcal{Z}} \coloneqq (\Phi^{\mathrm{env}}, p^{\mathrm{env}}_{Z_{0}})
		\end{align}
		is an input-output hidden Markov model, called an \emph{\textbf{environment model}}, the elements $z\in\mathcal{Z}$ are called its \emph{\textbf{hidden states}}, and $\nu^{\mathrm{env}}_{\bm{S}\bm{Z}|\bm{A}}$ is called the \textbf{model channel} of $\normalfont{\texttt{env}_{\mathcal{Z}}}$.
	\end{definition}
	While an environment channel describes only the input-output behavior, its input-output hidden Markov model provides an explicit memory-based mechanism that generates temporal correlations. Input-output hidden Markov models are also known as transducers \cite{barnett2015computational,rosas2025ai}.
	
	Agents are defined analogously.  
	The key distinction from environment channels is that the agent initiates the percept--action loop by selecting a first action $ A_0 $ (see \Cref{fig:2}):
	
	\begin{figure}
		\centering
		\includegraphics[width=0.99\linewidth]{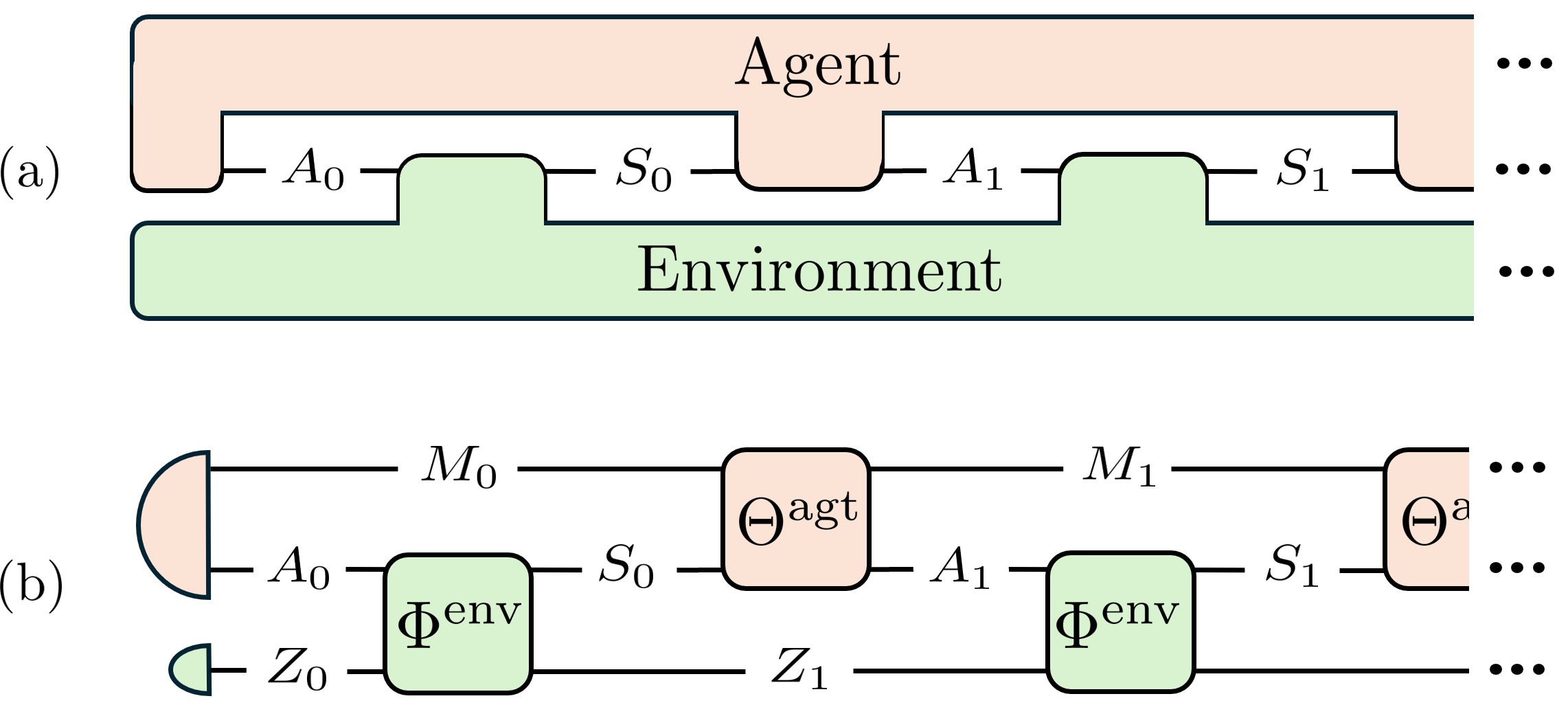}
		\caption{Circuit representation of percept--action loops, with time flowing from left to right;  the horizontal line represents the logical percept--action register shared between them.  (a) The agent and environment are modeled as channels with memory. (b) The agent and environment are represented by their hidden Markov models, characterized by finite adaptive memories $ M_t $ and $ Z_t $. The transition matrices $ \Theta^{\mathrm{agt}} $ and $ \Phi^{\mathrm{env}} $ remain fixed over time.}
		\label{fig:2}
	\end{figure}
	
	\begin{definition}\label{def:agent_channel}
		A channel $\eta^{\mathrm{agt}}_{\bm{A}|\bm{S}}$
		is an \emph{\textbf{agent channel}}, denoted as
		\begin{align}
			\normalfont \texttt{agt}\coloneqq\eta^{\mathrm{agt}}_{\bm{A}|\bm{S}},
		\end{align}
		if there exists a finite set of states $\mathcal{M}$, a distribution $p^{\mathrm{agt}}_{M_0 A_{0}}$ over $\mathcal{M}\times \mathcal{A}$, and a
		transition matrix $\Theta^{\mathrm{agt}} = \left(\theta^{\mathrm{agt}}(j|i)\right)_{j,i}$ with $i\in\mathcal{S}\times \mathcal{M}$ and $j\in\mathcal{A}\times \mathcal{M}$ such that
		\begin{align}
			\eta^{\mathrm{agt}}_{\bm{A}|\bm{
					S}}(\bm{a}|\bm{
				s})=\sum_{\bm{m}}\eta^{\mathrm{agt}}_{\bm{M} \bm{A} |\bm{S}}(\bm{m}, \bm{a}|\bm{s}),
		\end{align} 
		where the sum runs over all $\bm{m}\in\mathcal{M}^{\mathbb{N}_0}$ and
		\begin{align}
			\eta^{\mathrm{agt}}_{\bm{M} \bm{A} |\bm{S}}(\bm{m}, \bm{a}|\bm{s})&\coloneqq p^{\mathrm{agt}}_{M_0 A_0}(m_0, a_0) \notag\\
			&\qquad \prod_{t=0}^{\infty}\theta^{\mathrm{agt}}
			\left(a_{t+1},m_{t+1}|s_t,m_t\right). \label{eq:agent_model_channel}
		\end{align}
		Then, the tuple
		\begin{align}
			\normalfont \texttt{agt}_{\mathcal{M}}\coloneqq (\Theta^{\mathrm{agt}}, p^{\mathrm{agt}}_{M_0A_{0}})
		\end{align}			
		is an input-output hidden Markov model, called an \emph{\textbf{agent model}}, the elements $m\in\mathcal{M}$ are called its \emph{\textbf{memory states}}, and $\eta^{\mathrm{agt}}_{\bm{M} \bm{A} |\bm{S}}$ is called the \textbf{model channel} of $\normalfont{\texttt{agt}_{\mathcal{M}}}$.
	\end{definition}
	An agent can be understood as having two types of memory:  
	(i) 
	\emph{read-only memory} (ROM), which remains fixed over time and encodes the agent’s transition matrix $ \Theta^{\mathrm{agt}} $ (analogous to the programmed rules or hardware features of an artificial agent or the DNA of a biological organism), and
	(ii) \emph{read-write} or \emph{adaptive memory} $ \mathcal{M} $, which the agent can update to retain information about the past percept--action history and which, through the operation of $ \Theta^{\mathrm{agt}} $, influences future actions.
	
	Unless stated otherwise, we assume that the percept--action register contains no ROM---that is, all parts of it can be modified by the agent.

	With the definitions of agent and environment channels, as well as their hidden Markov models, we define percept--action loops 
	as tuples of an agent (model) and an environment (model), giving rise to four different types. In the following, it is assumed that agent and environment are compatible in the sense that their percept and action alphabets match. 
	\begin{table}
		\centering
		\renewcommand{\arraystretch}{1.6} 
		\setlength{\tabcolsep}{11pt}      
		\begin{tabular}{p{0.1cm} c c}
			\toprule
			& percept--action Loop
			& Distribution \\
			\midrule
			\textbf{(i)} & $\normalfont{(\texttt{agt}_{\mathcal{M}}, \texttt{env}_{\mathcal{Z}} )}$ & $p_{\bm{M}\bm{A}\bm{S}\bm{Z}} \coloneqq \nu^{\mathrm{env}}_{\bm{S}\bm{Z}|\bm{A}} \eta^{\mathrm{agt}}_{\bm{M}\bm{A}|\bm{S}}$ \\
			\textbf{(ii)} &        $\normalfont{(\texttt{agt}, \texttt{env}_{\mathcal{Z}} )}$ & $p_{\bm{A}\bm{S}\bm{Z}} \coloneqq \nu^{\mathrm{env}}_{\bm{S}\bm{Z}|\bm{A}} \eta^{\mathrm{agt}}_{\bm{A}|\bm{S}}$ \\
			\textbf{(iii)} &       $\normalfont{(\texttt{agt}_{\mathcal{M}}, \texttt{env} )}$ & $p_{\bm{M}\bm{A}\bm{S}} \coloneqq \nu^{\mathrm{env}}_{\bm{S}|\bm{A}} \eta^{\mathrm{agt}}_{\bm{M}\bm{A}|\bm{S}}$ \\
			\textbf{(iv)} &       $\normalfont{(\texttt{agt}, \texttt{env})}$ & $p_{\bm{A}\bm{S}} \coloneqq \nu^{\mathrm{env}}_{\bm{S}|\bm{A}}\eta^{\mathrm{agt}}_{\bm{A}|\bm{S}}$\\
			\bottomrule
		\end{tabular}
		\caption{Percept--action loops and their associated distributions. The four types (i) to (iv) differ in whether the agent and environment are represented as hidden Markov channels ($\texttt{agt}$ and $\texttt{env}$) or modeled with explicit memory states ($\texttt{agt}_{\mathcal{M}}$ and $\texttt{env}_{\mathcal{Z}}$). Their distributions are defined by the corresponding channels (see \Cref{def:agent_channel,def:environment_channel}).	}
		\label{tab:pals}
	\end{table}
	
	\begin{definition}\label{def:pal}
		\textbf{Percept--action loops} are classified into types (i)–(iv), defined by the tuples listed in the first column of \Cref{tab:pals}. Their corresponding distributions are given in the second column.
	\end{definition}
	
	A systematic way to understand percept--action loops is through type (i), which defines a \emph{global distribution} 
	\begin{align}
		&p_{\bm{MASZ}}  (\bm{m}, \bm{a}, \bm{s}, \bm{z}) =  p^{\mathrm{agt}}_{M_0 A_0}(m_0, a_0) p^{\mathrm{env}}_{Z_{0}}(z_{0}) \notag\\
		&\quad \times\prod_{t=0}^{\infty}\phi^{\mathrm{env}}
		\left(s_t,z_{t+1}|a_t,z_t\right)
		\theta^{\mathrm{agt}}
		\left(a_{t+1},m_{t+1}|s_t,m_t\right) \label{eq:global_distribution}
	\end{align}
	as the product of \cref{eq:environment_model_channel,eq:agent_model_channel}. The distributions for types (ii)–(iv) can then be obtained by marginalizing over the hidden states. \Cref{fig:2} illustrates percept--action loops of types (i) and (iv).
	
	We denote by $\normalfont \mathbb{A}^{\lrstacksmall \texttt{env}}(d)$ the set of all agent models with memory size $\abs{\mathcal{M}} = d$ that are compatible with a given environment \texttt{env}, and define the full set of compatible agents as
	\begin{align}
		\normalfont \mathbb{A}^{\lrstacksmall \texttt{env}} \coloneqq \bigcup_{d \in \mathbb{N}} \mathbb{A}^{\lrstacksmall \texttt{env}}(d).
	\end{align}
	Importantly, the global distribution $ p_{\bm{MASZ}} $ defines a finite-state Markov chain. This constitutes  the \emph{global Markov chain} of a percept--action loop (see  
	\Cref{supp:4} 
	of the Supplemental Material for a proof):  
	\begin{align}
		M_0A_0S_0Z_0 \rightarrow M_1A_1S_1Z_1 \rightarrow \dotsm.
	\end{align}
	\begin{figure}
		\centering
		\includegraphics[width=0.98\columnwidth]{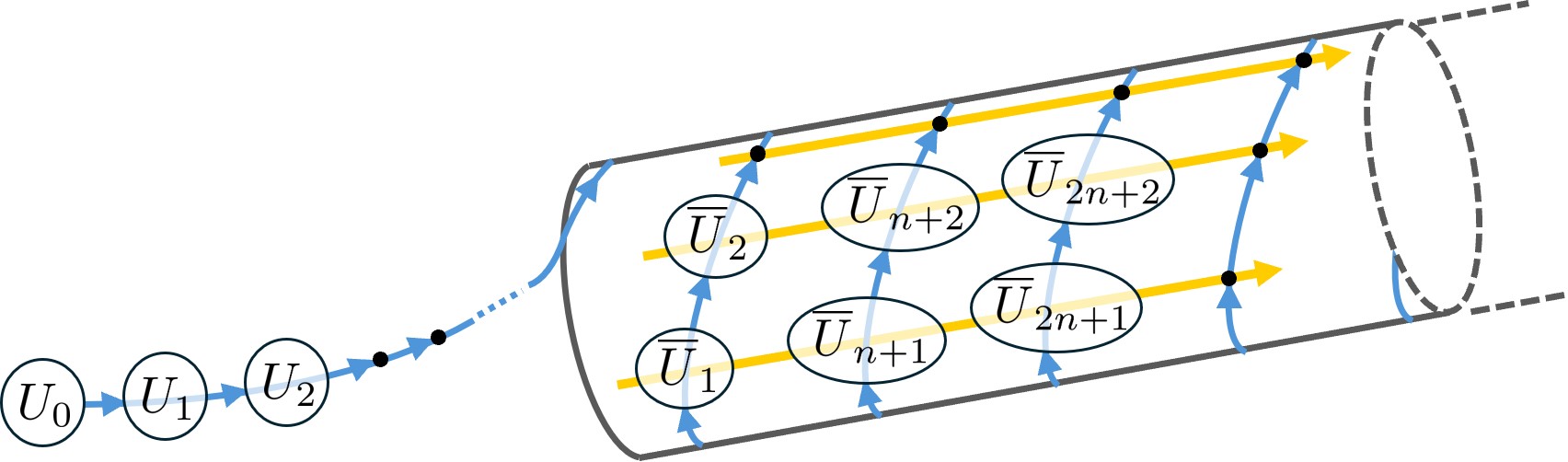}
		\caption{Asymptotic behavior of percept--action loops on a finite state space. We show that the global process $ \bm{U} $ of a percept--action loop, where $ U_t = (M_t, A_t, S_t, Z_t) $ includes the agent's memory, actions, percepts, and the environment's hidden memory, forms a finite-state Markov chain (the blue line visually guides the eye along this process). As time progresses, any such Markov chain converges to a cyclostationary process: one that periodically cycles through $ n $ stationary processes (as indicated by the yellow arrows), for some finite $ n $. This ensures that the asymptotic behavior of percept--action loops is well-behaved, in particular implying convergence of time averages.
		} \label{fig:3}
	\end{figure}
	This Markovian property allows us to leverage existing results on finite-state Markov chains, ensuring that the asymptotic dynamics of percept--action loops are well-behaved \cite{ash2008basic, iosifescu2014finite} (see \Cref{fig:3} and  
	\Cref{supp:2}
	of the Supplemental Material for more details).

	\section{Predictive  Agents} \label{sec:3}
	
	\begin{figure}
		\centering
		\resizebox{0.9\columnwidth}{!}{
			\begin{tikzpicture}[>=stealth',shorten >=1pt,node distance=2.5cm, scale= 0.83,on grid,auto, state/.style={circle, draw, minimum size=0.8cm, inner sep=1pt},font=\tiny]
				
				\def \n {1}
				
				\def \S {S}
				\def \A {A}
				\def \M {M}
				\def \Z {Z}
				\def \t {t}
				
				\def\XCol{{0,2,0,0,0,0}}
				\def\SCol{{1,3,,0,0,0}}
				\def\ACol{{1,1,0,0,0,0}}
				\def\ZCol{{0,0,0,0,0,0}}

				\foreach \step in {0,...,\n} {
					\pgfmathtruncatemacro{\ms}{-\step}
					\pgfmathtruncatemacro{\aftestate}{-\step+1}
					\pgfmathtruncatemacro{\prevstate}{\step-1}
					
					\ifnum \step= 0
					\pgfmathsetmacro{\xc}{\XCol[\n]}
					\pgfmathsetmacro{\sc}{\SCol[\n]}
					\pgfmathsetmacro{\ac}{\ACol[\n]}
					\pgfmathsetmacro{\zc}{\ZCol[\n]}
					
					\node[state, fill={\ifnum\xc=1 mred!50\else\ifnum\xc=2 mblue!50\else\ifnum\xc=3 mgreen!50\else white\fi\fi\fi}] (Xn-\step) at (0.75,0.75) {$\M_\t$};
					\node[state, fill={\ifnum\ac=1 mred!50\else\ifnum\ac=2 mblue!50\else\ifnum\ac=3 mgreen!50\else white\fi\fi\fi}] (An-\step) at (0,-1.5) {$\A_\t$};
					\node[state, fill={\ifnum\sc=1 mred!50\else\ifnum\sc=2 mblue!50\else\ifnum\sc=3 mgreen!50\else white\fi\fi\fi}] (Bn-\step) at (1.5,-1.5) {$\S_\t$};
					\node[state, fill={\ifnum\zc=1 mred!50\else\ifnum\zc=2 mblue!50\else\ifnum\zc=3 mgreen!50\else white\fi\fi\fi}] (Zn-\step) at (-0.75,-3.75) {$\Z_\t$};
					\node[circle, draw=black, fill=black!40, inner sep=2pt] (Tn-0) at (0.75,-2.25) {};
					\node[circle, draw=black, fill=black!40, inner sep=2pt] (Fn-0) at (-0.75,-0.75) {};
					\else
					
					\pgfmathsetmacro{\xc}{\XCol[\n+\step]}
					\pgfmathsetmacro{\sc}{\SCol[\n+\step]}
					\pgfmathsetmacro{\ac}{\ACol[\n+\step]}
					\pgfmathsetmacro{\zc}{\ZCol[\n+\step]}
					
					\node[state, fill={\ifnum\xc=1 mred!50\else\ifnum\xc=2 mblue!50\else\ifnum\xc=3 mgreen!50\else white\fi\fi\fi}] (Xn-\step) [right of = Xn-\prevstate] {$\M_{\t+\step}$};
					\node[state, fill={\ifnum\ac=1 mred!50\else\ifnum\ac=2 mblue!50\else\ifnum\ac=3 mgreen!50\else white\fi\fi\fi}] (An-\step) [right of = An-\prevstate] {$\A_{\t+{\step}}$};
					\node[state, fill={\ifnum\zc=1 mred!50\else\ifnum\zc=2 mblue!50\else\ifnum\zc=3 mgreen!50\else white\fi\fi\fi}] (Zn-\step) [right of = Zn-\prevstate] {$\Z_{\t+\step}$};
					\node[circle, draw=black, fill=black!40, inner sep=2pt] (Fn-\step) [right of = Fn-\prevstate] {};
					
					\ifnum \step< \n
					\node[state, fill={\ifnum\sc=1 mred!50\else\ifnum\sc=2 mblue!50\else\ifnum\sc=3 mgreen!50\else white\fi\fi\fi}] (Bn-\step) [right of = Bn-\prevstate] {$\S_{\t+{\step}}$};
					\node[circle, draw=black, fill=black!40, inner sep=2pt] (Tn-\step) [right of = Tn-\prevstate] {};
					\fi
					
					\pgfmathsetmacro{\xc}{\XCol[\n-\step]}
					\pgfmathsetmacro{\sc}{\SCol[\n-\step]}
					\pgfmathsetmacro{\ac}{\ACol[\n-\step]}
					\pgfmathsetmacro{\zc}{\ZCol[\n-\step]}
					\node[state, fill={\ifnum\xc=1 mred!50\else\ifnum\xc=2 mblue!50\else\ifnum\xc=3 mgreen!50\else white\fi\fi\fi}] (Xn-\ms) [left of = Xn-\aftestate] {$\M_{\t-\step}$};
					\node[state, fill={\ifnum\ac=1 mred!50\else\ifnum\ac=2 mblue!50\else\ifnum\ac=3 mgreen!50\else white\fi\fi\fi}] (An-\ms) [left of = An-\aftestate] {$\A_{\t-{\step}}$};
					\node[state, fill={\ifnum\sc=1 mred!50\else\ifnum\sc=2 mblue!50\else\ifnum\sc=3 mgreen!50\else white\fi\fi\fi}] (Bn-\ms) [left of = Bn-\aftestate] {$\S_{\t-{\step}}$};
					\node[state, fill={\ifnum\zc=1 mred!50\else\ifnum\zc=2 mblue!50\else\ifnum\zc=3 mgreen!50\else white\fi\fi\fi}] (Zn-\ms) [left of = Zn-\aftestate] {$\Z_{\t-\step}$};
					\node[circle, draw=black, fill=black!40, inner sep=2pt] (Tn-\ms) [left of = Tn-\aftestate] {};
					\node[circle, draw=black, fill=black!40, inner sep=2pt] (Fn-\ms) [left of = Fn-\aftestate] {};
					
					\path[->, thick] (Tn-\prevstate) edge (Zn-\step);
					\path[->, thick] (Tn-\prevstate) edge (Bn-\prevstate);
					\path[->, thick] (Fn-\step) edge (Xn-\step);
					\path[->, thick] (Fn-\step) edge (An-\step);
					\path[->, thick] (Tn-\ms) edge (Zn-\aftestate);
					\path[->, thick] (Tn-\ms) edge (Bn-\ms);
					\path[->, thick] (Zn-\ms) edge (Tn-\ms);
					\path[->, thick] (An-\ms) edge (Tn-\ms);
					\path[->, thick] (Xn-\ms) edge (Fn-\aftestate);
					\path[->, thick] (Bn-\ms) edge (Fn-\aftestate);
					\fi
					
					\path[->, thick] (Fn-\ms) edge (Xn-\ms);
					\path[->, thick] (Fn-\ms) edge (An-\ms);
				}
				
				\foreach \step in {1,...,\n} {
					\pgfmathtruncatemacro{\prevstate}{\step-1}
					\path[->, thick] (Xn-\prevstate) edge (Fn-\step);
					\path[->, thick] (Zn-\prevstate) edge (Tn-\prevstate);
					\path[->, thick] (Bn-\prevstate) edge (Fn-\step);
					\path[->, thick] (An-\prevstate) edge (Tn-\prevstate);
				}
				
				\path[dotted, very thick, black!50] (Xn-\n) edge ++(0.7,-0.7);
				\path[dotted, very thick, black!50] (An-\n) edge ++(0.7,-0.7);
				\path[dotted, very thick, black!50] (Zn-\n) edge ++(1,1);
				\path[dotted, very thick, black!50] (Fn--\n) edge ++(-0.7,-0.7);
				\path[dotted, very thick, black!50] (Fn--\n) edge ++(-0.7,0.7);
				\path[dotted, very thick, black!50] (Zn--\n) edge ++(-1,1);
				
			\end{tikzpicture}
		}
		\caption{Bayesian network for a percept--action loop. Shown is a fragment for rounds $t-1$, $t$, and the beginning of round $t+1$. This type of Bayesian network plays an important role in the information-theoretic framework underlying our results (see 
			\Cref{supp:5}
			of the Supplemental Material for details and \cite{ito2013information} for an application of Bayesian networks in the context of the thermodynamics of feedback control).
			Note that to faithfully represent the dynamics of the agent and environment, auxiliary nodes (gray and reduced in size) are included. The colorized nodes illustrate the condition for an agent to be maximally predictive in round $t$: the agent's memory (blue) must store all information from past actions and percepts $A_{0:t+1}S_{0:t}$ (red) that is relevant for predicting the current percept $S_t$ (green). }
		\label{fig:4}
	\end{figure} 
	
	For a given input-output behavior of the agent and environment, $(\texttt{agt}, \texttt{env})$, what does it mean for an agent to be as predictive as possible of its future percepts? To approach this question, it is helpful to begin with the following observation. In order to endow the agent with knowledge that reduces \emph{uncertainty} about future percepts, a natural first step is to encode $(\texttt{agt}_{\mathcal{M}}, \texttt{env}) = (\Theta^{\mathrm{agt}}, p^{\mathrm{agt}}_{M_0A_{0}}, \nu^{\mathrm{env}}_{\bm{S}|\bm{A}})$ into its fixed read-only memory. In what follows, we assume this is always the case.
	
	With this setup, the agent has access to the distribution of the underlying process, $p_{\bm{MAS}}$, which results in an uncertainty $H(S_t)$ about percept $S_t$, where $H(S_t)$ denotes Shannon entropy in units of bits. If, in addition, the agent takes its memory $M_t$ into account before observing $S_t$, this memory reduces the agent’s expected (with respect to memory states) uncertainty to $H(S_t|M_t)$. This reduction in uncertainty,  
	\begin{align}
		I\left[M_t; S_t\right] = H(S_t)-H(S_t|M_t), \label{eq:information_about_percept}
	\end{align}
	is simply the mutual information $I\left[M_t; S_t\right]$ between $S_t$ and $M_t$, quantifying how much $M_t$ enables the agent to predict $S_t$ (see  
	\Cref{supp:1}
	of the Supplemental Material for some background on information measures).

	To enhance its predictive capabilities, the agent can store information from past percepts and actions in its memory. In fact, the information the past $A_{0:t+1}S_{0:t}$ provides about the current percept $S_t$ is given by $I\left[A_{0:t+1}S_{0:t};S_t\right]$. Consequently, if the conditional mutual information $I\left[A_{0:t+1}S_{0:t};S_t|M_t\right]$ vanishes, the agent's memory $M_t$ contains all information from the past about the current percept. Otherwise, if $I\left[A_{0:t+1}S_{0:t};S_t|M_t\right]$ is positive, the agent must have forgotten about aspects of the past which are relevant for predicting the current percept. $I\left[A_{0:t+1}S_{0:t};S_t|M_t\right]$ thus quantifies a \emph{predictive gap}:
	\begin{definition} \label{def:predictive_main}
		Let $(\normalfont\texttt{agt}_{\mathcal{M}}, \texttt{env})$ be a percept--action loop. The \textbf{predictive gap} of ${\normalfont\texttt{agt}_{\mathcal{M}}}$ with respect to $\normalfont\texttt{env}$ in round $t$ is given by
		\begin{align}
			I\left[A_{0:t+1}S_{0:t};S_t|M_t\right], \label{eq:predictiveness_forgetfulness}
		\end{align}
		and the asymptotic mean (a.m.) predictive gap of ${\normalfont\texttt{agt}_{\mathcal{M}}}$ with respect to $\normalfont\texttt{env}$ is 
		\begin{align}
			\cesaro{I\left[A_{0:t+1}S_{0:t};S_t|M_t\right]}_t, \label{eq:am_predictiveness_forgetfulness}
		\end{align}
		where 
		\begin{align}
			\cesaro{\bullet }_t \coloneqq \lim\limits_{n\rightarrow \infty}\frac{1}n\sum_{t=0}^{n-1}\bullet \label{eq:cesaro}
		\end{align}
		denotes the Ces\`aro limit, the limit of the arithmetic mean.
		An agent model $\normalfont\texttt{agt}_\mathcal{M}$ is said to be \textbf{maximally predictive} if its a.m.~predictive gap vanishes.
	\end{definition}
	The condition for an agent to be maximally predictive can be illustrated within the Bayesian network underlying percept--action loops (see \Cref{fig:4}).
	For example, an agent is maximally predictive if \cref{eq:predictiveness_forgetfulness} vanishes for all times, or if \cref{eq:predictiveness_forgetfulness} decays sufficiently quickly.

	\section{Work capacity of channels}\label{sec:4}
	So far, we have treated agents and environments in terms of abstract logical operations on alphabets. However, as Landauer famously quipped, \emph{information is physical} \cite{landauer1991information}: any implementation of an agent must ultimately rely on physical memory and dynamics subject to thermodynamic laws. To analyze the energetic limits of such implementations, we adopt a framework from stochastic thermodynamics that models the agent’s information processing, described by its transition matrix 
	$\Theta^{\mathrm{agt}}$, as a physical process acting on memory \cite{seifert2012stochastic, parrondo2015thermodynamics}.
	The following \emph{isothermal information-processing (IIP) assumptions} have three parts: information carriers and encodings, physical protocols, and thermodynamic resource accounting. Together, they are chosen so that the resulting energetic bounds capture only the intrinsic cost of information processing. Before turning to the full percept--action loop, we first consider a single logical operation $\Theta:\mathcal{X}_{\mathrm{in}}\to\mathcal{X}_{\mathrm{out}}$.

	\subsection{Isothermal information-processing assumptions}
	\noindent \textbf{Information carriers and encodings}---Let $\mathfrak{S}_{d}$ denote an \emph{information carrier}, i.e., a physical system coupled to a thermal reservoir at temperature $T$. The system has $d$ metastable configurations satisfying the timescale separation $\tau_{\mathrm{others}} \ll \tau_{\mathrm{info}}$, where $\tau_{\mathrm{info}}$ is the equilibration time of the metastable configurations and $\tau_{\mathrm{others}}$ is the equilibration time of the remaining degrees of freedom.
	
	We assume that the input and output alphabets have the same cardinality, $|\mathcal{X}_{\mathrm{in}}| = |\mathcal{X}_{\mathrm{out}}| = d$.
	Both alphabets are encoded in the same information carrier $\mathfrak{S}_{d}$ by one-to-one assignments of alphabet symbols to metastable configurations. The input alphabet labels the metastable configurations before the protocol; the output alphabet labels the same metastable configuration space after the protocol. The initial distribution over $\mathcal{X}_{\mathrm{in}}$ is arbitrary.

	Information-processing architectures often store some variables in separate physical subsystems; for example, memory may be encoded in a separate subsystem. To model such modular architectures, we allow information carriers to be composed by taking component information carriers $\mathfrak{S}_{d_1},\dotsc,\mathfrak{S}_{d_n}$ coupled to the same thermal reservoir at temperature $T$. The metastable configuration space of the resulting composite information carrier $\mathfrak{S}_{d}$ is the product of the component metastable configuration spaces, so that $d = \prod_{i=1}^{n} d_i$. If each component information carrier $\mathfrak{S}_{d_i}$ encodes alphabet $\mathcal{X}_i$ one-to-one, then the composite information carrier naturally encodes the product alphabet $\mathcal{X}_1\times\dotsb\times\mathcal{X}_n$.

	\noindent \textbf{Physical protocol}---A physical protocol implements the logical operation $\Theta:\mathcal{X}_{\mathrm{in}}\to\mathcal{X}_{\mathrm{out}}$ by isothermal control of the information carrier $\mathfrak{S}_{d}$. The protocol is executed on a timescale $\tau_{\mathrm{protocol}}$ satisfying $\tau_{\mathrm{others}} \ll \tau_{\mathrm{protocol}} \ll \tau_{\mathrm{info}}$.
	For a composite information carrier, this condition is understood as applying to every component information carrier.
	
	The protocol may implement any isothermal control of the potential landscape over the metastable configuration space required by the logical operation under consideration. For a composite information carrier, the protocol may implement joint isothermal control over the composite metastable configuration space.
	
	\noindent \textbf{Thermodynamic resource accounting}---The following assumptions ensure that the work balance reflects the thermodynamics intrinsic to information processing, rather than changes in the energetic structure of the information carrier or in the external control apparatus.
	
	At the beginning and end of the protocol, the effective internal energies assigned to the metastable configurations of the information carrier are equal and unchanged. That is,  $U_i^{\mathrm{in}} = U_i^{\mathrm{out}} = U_0$
	for all $i\in\{1,\dots,d\}$.
	An information carrier with equal internal energies assigned to its metastable configurations is called an \emph{information reservoir} \cite{deffner2013information}.
	
	The isothermal control is cyclic: by the end of the protocol, all externally controlled potential landscapes are returned to their initial form, and the external control apparatus undergoes no net physical change. However, the probability distribution over metastable configurations may change, thereby implementing the desired logical operation $\Theta$ \cite{parrondo2015thermodynamics}.

	\begin{figure}
		\centering
		\includegraphics[width=0.75\columnwidth]{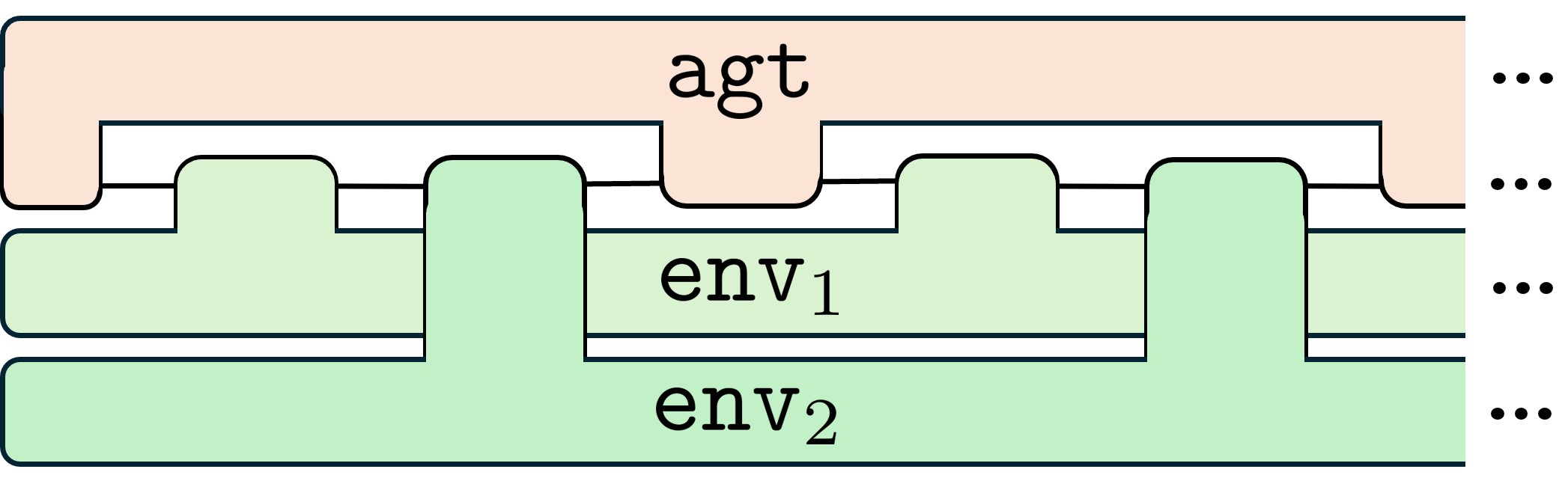}
		\caption{An agent \texttt{agt} interacting with the cascade of two environment channels $\texttt{env}_1$ and $\texttt{env}_2$.}
		\label{fig:5}
	\end{figure}

	\subsection{Second-law bound on extractable work}
	Let $W$ be the expected work extracted by applying a physical protocol that implements a logical operation $\Theta:\mathcal{X}_{\mathrm{in}}\to\mathcal{X}_{\mathrm{out}}$ on an information carrier. In the following, we use the sign convention that $W>0$ denotes work extracted by the protocol and $W<0$ work supplied to implement it. Note that we refer to the work as \emph{expected} because, before executing the protocol, it is the average work predicted from the distribution $p_{X_{\mathrm{in}}}$ over initial configurations of the information carrier.
	
	Under the IIP assumptions, the second law provides an upper bound on $W$ in terms of the change in nonequilibrium free energy of the information carrier between the beginning and end of the protocol \cite{seifert2012stochastic, parrondo2015thermodynamics}. The nonequilibrium free energy is the state function
	$F = U - \kb T \ln 2\, H(X)$, where $U=\sum_x p_X(x) U_x$, with $U_x$ the effective internal energy assigned to the information carrier's metastable configuration encoding $x$, and $\kb$ is the Boltzmann constant \cite{boyd2018thermodynamics}. Since the internal-energy contributions cancel under the IIP assumptions, this free-energy bound reduces to an entropy bound on the expected extractable work,
	\begin{align}
		W \leq H\left(X_{\mathrm{out}}\right) - H\left(X_{\mathrm{in}}\right),
		\label{eq:second_law}
	\end{align}
	when implementing the logical operation $\Theta$ that induces the transition from $X_{\mathrm{in}}$ to $X_{\mathrm{out}}$. Here and throughout, all work expressions are understood to be in units of $\kb T \ln 2$.

	The upper bound in \cref{eq:second_law}, imposed by the second law, is tight in the sense that it can, in principle, be saturated with protocols under idealized conditions. Concrete examples of such protocols are given in \cite{garner2017thermodynamics, boyd2018thermodynamics, boyd2022thermodynamic}. While our work is primarily concerned with the fundamental limits imposed by the second law, it should be noted that more realistic and resource-constrained assumptions can be incorporated \cite{kolchinsky2017dependence, shiraishi2018speed, kolchinsky2021work, wolpert2024stochastic}. Within our framework, this is most easily achieved when the extractable work can still be expressed through a state function, such as in \cite{kolchinsky2021work}, by replacing $F$ with the new state function.

	We now apply \cref{eq:second_law} to the agent's logical operation in a percept--action loop under the same IIP assumptions. That is, we set $\mathcal{X}_{\mathrm{in}}=\mathcal{M}\times\mathcal{S}$ and $\mathcal{X}_{\mathrm{out}}=\mathcal{M}\times\mathcal{A}$, and choose the logical operation $\Theta$ to implement the agent's transition matrix $\Theta^\mathrm{agt}$. In particular, $\abs{\mathcal{X}_{\mathrm{in}}}=\abs{\mathcal{X}_{\mathrm{out}}}$ since, by assumption, $\abs{\mathcal{A}}=\abs{\mathcal{S}}$. For this application of \cref{eq:second_law}, the alphabets $\mathcal{M}\times\mathcal{S}$ and $\mathcal{M}\times\mathcal{A}$ are encoded in the same information carrier. This information carrier may nevertheless be built from component information carriers representing the agent memory and a single register that carries the percept before the protocol and the action after it.
	
	For each round $t$, this gives an upper bound $W_t \leq H\left(M_{t+1}A_{t+1}\right) - H\left(M_tS_t\right)$
	on the expected work extractable by the agent's operation in that round.

	Taking the asymptotic mean over rounds yields an upper bound on the asymptotically expected work per round that an agent model $\texttt{agt}_{\mathcal{M}}$ can extract by implementing its transition matrix under the IIP assumptions. We denote this bound by
	\begin{align}
		W(\texttt{agt}_{\mathcal{M}}, \texttt{env})
		=
		\cesaro{H\left(A_t| M_t\right) - H\left(S_t|M_t\right)}_t,
		\label{eq:work_rate}
	\end{align}
	where $\cesaro{\bullet}_t$ is the Ces\`aro limit defined in \cref{eq:cesaro}; see 
	\Cref{supp:6.1}
	of the Supplemental Material for a derivation. Since the bound in \cref{eq:work_rate} is tight in principle under idealized protocols, and since we focus on fundamental limits, we refer to it in the following simply as the work rate.
	
	The IIP assumptions imply an implementation picture for the percept--action loop model in which the information carrier encoding the agent's memory remains under the agent's control throughout all rounds. For each round, the same information carrier must encode the percept before the agent's protocol and the action after it. The environment may either act on that information carrier to produce the next percept or provide a new information carrier carrying that percept. Since \cref{eq:work_rate} bounds the work extracted by the agent's operation, the environment channel provides the only resource for work extraction by the agent. The agent exploits this resource by retaining information from past rounds in its memory and thereby adapting to the statistical structure of the environment channel.
	
	The existence of work rate is not guaranteed for arbitrary processes, as it is possible that the Ces\`aro limit in \cref{eq:work_rate} does not exist \footnote{An illustrative example of a sequence $(a_t)_t$ where the Ces\`aro limit $\cesaro{a_t}_t$ fails to exist is $0110000\dotsc$, where one zero is followed by twice as many ones, followed again by twice as many zeros, and so on. This results in an oscillating arithmetic mean $1/n \sum_{t=1}^n a_t$ as $n \rightarrow \infty$.}. Note, however, that here the limit exists because the global Markov chain of the percept--action loop has a finite number of states and thus converges to a cyclostationary process (see
	\Cref{fig:3}).
	Note that if the logical percept--action register involves ROM---parts which can only be read but not updated by the agent---work rate becomes $\cesaro{\left(H\left(A_tS^{\mathrm{ROM}}_{t-1}| M_t\right) - H\left(S_tS^{\mathrm{ROM}}_{t}|M_t\right)\right)}_t$ where $ S^{\mathrm{ROM}}_{t} $ denotes the ROM provided by the environment in round $ t $, and $ S^{\mathrm{ROM}}_{-1} $ is an initial ROM state.
	
	\subsection{Work capacity}
	
	\Cref{eq:work_rate}  can be interpreted as a thermodynamic objective function for agents derived directly from the second law of thermodynamics and the information-theoretic structure of percept--action loops (see \cite{ortega2013thermodynamics} for a related but distinct approach with a thermodynamics-inspired objective for decision making). Maximizing this objective leads us to the following:
	\begin{definition} \label{def:work_capacity} 
		For a given environment channel {\normalfont\texttt{env}}, let
		\begin{align}
			\normalfont W^{(d)}(\texttt{env}) \coloneqq  \sup_{\texttt{agt}_{\mathcal{M}}\in\mathbb{A}^{\lrstacksmall \texttt{env}}(d)} W(\texttt{agt}_{\mathcal{M}}, \texttt{env}) \label{eq:work_capacity_d_main}
		\end{align}
		be the supremum of the work rates among all agent models with $\abs{\mathcal{M}} = d$. Then,
		the \emph{\textbf{work capacity}} $C^{\mathrm{work}}$ of $\normalfont\texttt{env}$ is defined as
		\begin{align}
			\normalfont C^{\mathrm{work}}(\texttt{env}) \coloneqq  \sup_{d\in\mathbb{N}}  W^{(d)}(\texttt{env}). \label{eq:work_capacity_main}
		\end{align}
	\end{definition}
	Combining \cref{eq:work_capacity_d_main,eq:work_capacity_main}, we can also write $C^{\mathrm{work}}(\texttt{env}) = \sup_{\texttt{agt}_{\mathcal{M}}\in\mathbb{A}^{\lrstacksmall \texttt{env}}}  W(\texttt{agt}_{\mathcal{M}}, \texttt{env})$.
	
	Intuitively, the work capacity captures the supremum rate at which an agent optimally tailored to the environment channel can expect to extract work, based on the second law of thermodynamics.
	
	For any fixed percept--action loop $(\texttt{agt}_{\mathcal M},\texttt{env})$, the protocols of Refs.~\cite{garner2017thermodynamics,boyd2018thermodynamics,boyd2022thermodynamic} can, in principle, attain the corresponding work	rate $W(\texttt{agt}_{\mathcal M},\texttt{env})$, thereby allowing the	supremum in \cref{eq:work_capacity_d_main} to be approached arbitrarily closely; see 
	\Cref{supp:6.2}
	of the Supplemental Material.

	Returning to the question posed at the beginning of this section, the energetic limits of agents, in terms of work rate, are determined by the work capacity of the environment channel.
	
	We next state some general properties of work capacity:
	\begin{theorem} \label{th:properties_of_work_capacity_main}
		For any environment channel $\normalfont \texttt{env}=\nu^{\mathrm{env}}_{\bm{S}|\bm{A}}$, work capacity $\normalfont C^{\mathrm{work}}(\texttt{env})$ has the following properties:
		\begin{enumerate}[label=(\roman*)]
			\item \emph{(Existence)} $\normalfont C^{\mathrm{work}}(\texttt{env})$ exists,
			\item \emph{(Bounds)} $\normalfont 0\leq C^{\mathrm{work}}(\texttt{env})\leq   \log_2 |\mathcal{S}|$,
			\item \emph{(Subadditivity under channel cascade, see  \Cref{fig:5})}
			\begin{align*}
				\normalfont C^{\mathrm{work}}(\texttt{env}_2\circ \texttt{env}_1)\leq  C^{\mathrm{work}}(\texttt{env}_1) + C^{\mathrm{work}}(\texttt{env}_2).
			\end{align*}
		\end{enumerate} 
	\end{theorem}
	See  
	\Cref{supp:6.3}
	 of the Supplemental Material for a proof. \Cref{th:properties_of_work_capacity_main}(i) is in fact a statement about the existence of work rate---which enters in the definition of work capacity---and follows from the fact that the global Markov chain is well behaved (see \Cref{fig:3}). Further, the bounds in \Cref{th:properties_of_work_capacity_main}(ii) follow from the canonical bounds on Shannon entropy.
	Finally, \Cref{th:properties_of_work_capacity_main}(iii) follows from the observation that agents interacting with a cascade $\texttt{env}_2\circ \texttt{env}_1$ form a subclass of those that alternately interact with $\texttt{env}_1$ and $\texttt{env}_2$. Specifically, this subclass consists of agents that implement an identity channel from the percepts of $\texttt{env}_1$ to the actions of $\texttt{env}_2$. The inequality then follows from the fact that using a non-identity channel can yield a higher work rate (see \Cref{fig:6} for an illustrative example of the subadditivity of work capacity).
	
	\begin{figure}
		\centering
		\begin{tikzpicture}
			\node[anchor=north west] at (0,0) {\includegraphics[width=0.8\linewidth]{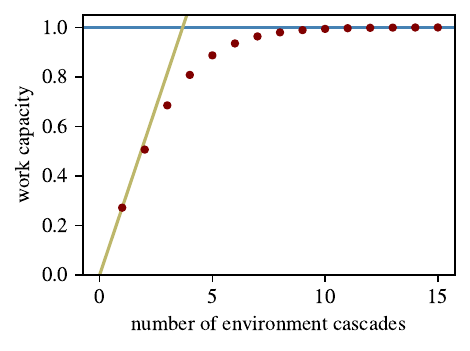}};
			\node[anchor=north west] at (3,-2.8) {\includegraphics[width=0.4\linewidth]{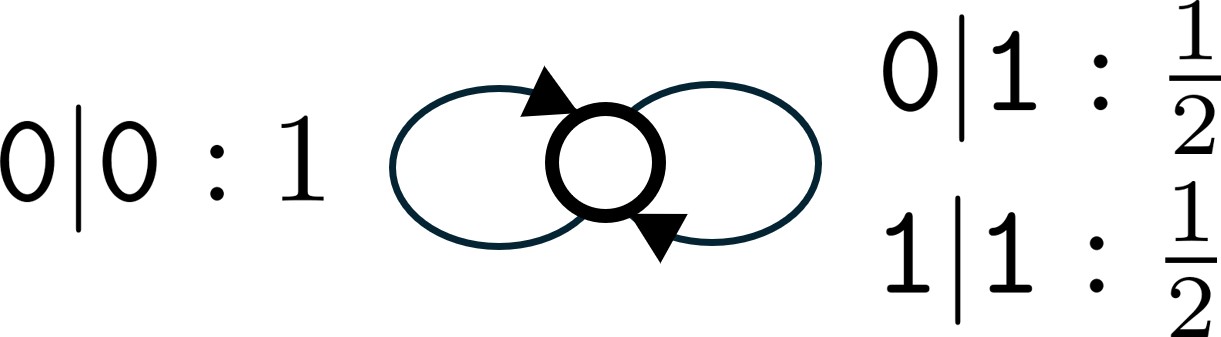}};
		\end{tikzpicture}
		\caption{Minimal example illustrating the subadditivity of work capacity. We consider a cascade of $n$ identical copies of a memoryless, invariant environment \texttt{env} with binary percept and action alphabets, $\mathcal{A} = \mathcal{S} = \{\texttt{0}, \texttt{1}\}$ (see inset; the transition labels follow the format \emph{percept}$\,|\,$\emph{action}$~:~$\emph{transition probability}).  The burgundy dots represent work capacity $C^{\mathrm{work}}(\texttt{env}^n)$, in units of $\kb T \ln 2$, of the $n$-fold cascade as a function of $n$ (see \Cref{fig:5} for an illustration of the case $n = 2$). The yellow line is the subadditivity bound from \Cref{th:properties_of_work_capacity_main}(iii) and the blue line at $\log_2 2$ is the upper bound from \Cref{th:properties_of_work_capacity_main}(ii). Note that for this environment, the outputs become increasingly deterministic, approaching a constant \texttt{0} as $ n \to \infty $. This allows the agent to better predict the outputs and extract more work as $n$ increases.
		}
		\label{fig:6}
	\end{figure}
	
	Because work capacity involves both a Cesàro limit and a maximization over agent models, it is generally difficult to compute. We now discuss special classes of environment channels for which the expression simplifies, as summarized in \Cref{tab:capacity}.
	
	\begin{table}
		\centering
		\renewcommand{\arraystretch}{1.2} 
		\setlength{\tabcolsep}{12pt}      
		\begin{tabular}{l c}
			\toprule
			\textbf{Environment Channel \texttt{env}} & $C^{\mathrm{work}}(\texttt{env})$ \\
			\midrule
			Memoryless Invariant & $\displaystyle \max_{p_{A_0}}\big[ H(A_0) - H(S_0) \big]$ \\
			Identity & $0$ \\
			Identity + Markovian ROM  & $\displaystyle \max_{p_{S_0}} I[S_0^{\mathrm{ROM}};S_0]$
			\\
			Product ($=$ tape setting) & $\log_2 |\mathcal{A}| - h(\bm{S})$ \\
			\bottomrule
		\end{tabular}
		\caption{Simplified expressions for the work capacity of memoryless invariant and product environment channels. The identity channel, with or without Markovian read-only memory (ROM), is a special case of a memoryless invariant channel. See \cref{eq:entropy_rate_main} for the definition of the entropy rate $ h(\bm{S}) $, and 
			\Cref{supp:6.3}
			of the Supplemental Material for a detailed proof.
		}
		\label{tab:capacity}
	\end{table}
	
	\paragraph{Memoryless invariant channels.}
	For \emph{memoryless invariant} environment channels, where
	\begin{align}
		\nu^{\mathrm{env}}_{\bm{S}|\bm{A}}(\bm{s}|\bm{a}) = \prod_{t=0}^{\infty} \phi^{\mathrm{env}}(s_t | a_t)
	\end{align}
	with the same transition matrix $\Phi^{\mathrm{env}}$ for all $t \in \mathbb{N}_0$, the absence of temporal memory allows the optimization in \cref{eq:work_capacity_main} to reduce to a single-step maximization over actions.
	
	A basic example is the \emph{identity channel} \footnote{Also referred to as a noiseless channel in communication theory.}, where $A_t = S_t$ for all $t$. In this case, an agent can predict the percept precisely to the extent it remembers its previous action, turning the trade-off between action and percept entropy into a zero-sum relation:
	\begin{align}
		H(A_t | M_t) - H(S_t | M_t) = 0,
	\end{align}
	and thus the work capacity vanishes.
	
	A richer example is an \emph{identity channel with read-only memory (ROM)}: the main channel transmits the input unchanged, while an additional side channel provides ROM that the agent can observe but not modify. Formally, the environment maps $ (A_t, A^{\mathrm{ROM}}_t)$ to $(S_t, S^{\mathrm{ROM}}_t)$ such that $S_t = A_t$, and the agent is constrained to respond such that $ A^{\mathrm{ROM}}_{t+1}=S^{\mathrm{ROM}}_t$ while  $A^{\mathrm{ROM}}_{0}$ is fixed. When the ROM is Markovian---that is, when $S^{\mathrm{ROM}}_t$ is a probabilistic function of $A_t$---the work capacity simplifies to the expression given in \cref{tab:capacity}.
	
	This setup parallels scenarios studied in stochastic thermodynamics, where a system is controlled by an agent using side information about the system’s state. It encompasses Maxwell's demon–type models, in which the agent can exploit this auxiliary information to extract work, while the environment bears the cost of providing the ROM.
	
	A canonical example is the Szilard box with measurement-based feedback control \cite{sagawa2012nonequilibrium}. Let $S_t$ and $A_{t+1}$ represent the state of the Szilard box before and after the agent’s transformation, and let $S^{\mathrm{ROM}}_t$ be the outcome of a (possibly noisy but feedback-free) measurement on the Szilard box  performed by the environment. Without ROM, the environment channel reduces to an identity channel with zero work capacity. However, when ROM is present, the agent can condition its transformation on $S^{\mathrm{ROM}}_t$ and achieve a work rate upper bounded by the work capacity $\max_{p_{S_0}}I \big[S^{\mathrm{ROM}}_0; S_0\big]$. For further generalizations involving non-Markovian ROM, see \cite{sagawa2012nonequilibrium, ito2013information, shiraishi2015fluctuation}.
	
	While the examples above either lack feedback or restrict it to ROM, the class of memoryless invariant channels also includes environments with \emph{genuine}, albeit memoryless, feedback. An example is shown in the inset of \Cref{fig:6}, with binary alphabets $\mathcal{A} = \mathcal{S} = \{\texttt{0}, \texttt{1}\}$ and transition matrix $\Phi^{\mathrm{env}}$ given by:
	\begin{align}
		\phi^{\mathrm{env}}(j|\texttt{0}) = \delta_{j,\texttt{0}}, \quad \phi^{\mathrm{env}}(j|\texttt{1}) = \frac{1}{2}, \quad j \in \{\texttt{0}, \texttt{1}\}.
	\end{align}
	For this environment, the work capacity is
	\begin{align}
		C^{\mathrm{work}}(\texttt{env}) = \frac{1}{2} \log_2\left(\frac{3}{4} + \frac{1}{\sqrt{2}}\right) \simeq 0.272
	\end{align}
	in units of $\kb T \ln 2$, and can be achieved by a memoryless agent that selects action \texttt{0} with probability $1/\sqrt{2}$ at each step.
	
	\paragraph{Product channels.}	
	For \emph{product environment channels}, percepts are not influenced by actions, placing us in the tape setting. Consequently, to maximize work rate, the optimal strategy is to maximize $H(A_t|M_t)$. 
	This implies two conditions which can be made explicit by rewriting conditional entropy using the definition of mutual information:  $H(A_t|M_t)=H(A_t)-I[A_t;M_t]$. By the nonnegativity of mutual information, it then follows that agent models must perfectly randomize their actions (maximizing $\cesaro{H(A_t)}_t$) without retaining memory of them ($ \cesaro{I[A_t;M_t]}_t = 0$). 
	This results in $H(A_t|M_t) = \log_2\abs{\mathcal{A}}$. The second term in the work capacity expression, as shown in \Cref{tab:capacity}, is the entropy rate of the percept process $\bm{S}$,
	\begin{align}
		h\left(\bm{S}\right) \coloneqq \lim\limits_{n \to \infty} \frac{H\left(S_{0:n}\right)}{n}, \label{eq:entropy_rate_main}
	\end{align}
	which was introduced by Shannon as the average uncertainty per symbol in a stochastic process \cite{shannon1948mathematical} \footnote{If the percept process is a stationary finite-state Markov chain, a closed-form expression for the entropy rate exists \cite{shannon1948mathematical}.}. It is also known, from the information-processing second law \cite{boyd2016identifying}, that $\log_2 |\mathcal{A}| - h(\bm{S})$  (in units of $\kb T \ln 2$) represents an upper bound on the rate of expected extractable work from a stochastic process.

	\section{Work-Efficient Agent Models}\label{sec:5}
	
	Determining work capacity is challenging, as it requires solving a nonlinear optimization problem (\cref{eq:work_capacity_main}). However, established design principles exist that reach the fundamental energetic limits of information processing in the tape setting, specifically for agent models in $\mathbb{A}^{\lrstacksmall \texttt{env}}$ \cite{boyd2018thermodynamics}. To generalize these results for agent models in $\mathbb{A}^{\lrstacksmall \texttt{env}}(d)$, which are restricted to finite memory of size $d$, we define the following subsets of agents:
	\begin{definition}\label{def:agent_subsets}
		For a given environment $\normalfont\texttt{env}$ and any $d \in \mathbb{N}$, 
		we define the following subsets of $\normalfont\mathbb{A}^{\lrstacksmall \texttt{env}}(d)$:
		\begin{align}
			\normalfont\mathbb{A}^{\lrstacksmall \texttt{env}}_{\mathrm{mea}}(d) &\coloneqq \displaystyle\operatorname*{argmax}_{\normalfont\texttt{agt}_{\mathcal{M}}\in\mathbb{A}^{\lrstacksmall \texttt{env}}(d)}\cesaro{H(A_t|M_t)}_t,\\
			\normalfont\mathbb{A}^{\lrstacksmall \texttt{env}}_{\mathrm{pred}}(d) &\coloneqq \displaystyle\operatorname*{argmin}_{\normalfont\texttt{agt}_{\mathcal{M}}\in\mathbb{A}^{\lrstacksmall \texttt{env}}(d)}\cesaro{I\left[A_{0:t+1}S_{0:t};S_t|M_t\right]}_t, \\
			\normalfont\mathbb{A}^{\lrstacksmall \texttt{env}}_{\mathrm{eff}}(d)&\coloneqq \displaystyle\operatorname*{argmax}_{\normalfont\texttt{agt}_{\mathcal{M}}\in\mathbb{A}^{\lrstacksmall \texttt{env}}(d)}W(\normalfont\texttt{agt}_{\mathcal{M}}, \texttt{env}).
		\end{align}
	\end{definition}
	Here and below, $\operatorname*{argmax}$ and
	$\operatorname*{argmin}$ denote the sets of agent models attaining the
	corresponding supremum or infimum. For fixed $d$, the parameter space
	of agent models is compact and, thus, the corresponding supremum or
	infimum is attained whenever the objective is continuous in the agent
	parameters.
	\begin{figure*}
		\noindent
		\begin{minipage}[t]{0.6\linewidth}
			\raggedright
			(a)\\
			\vspace{0.2cm}
			\includegraphics[width=\linewidth]{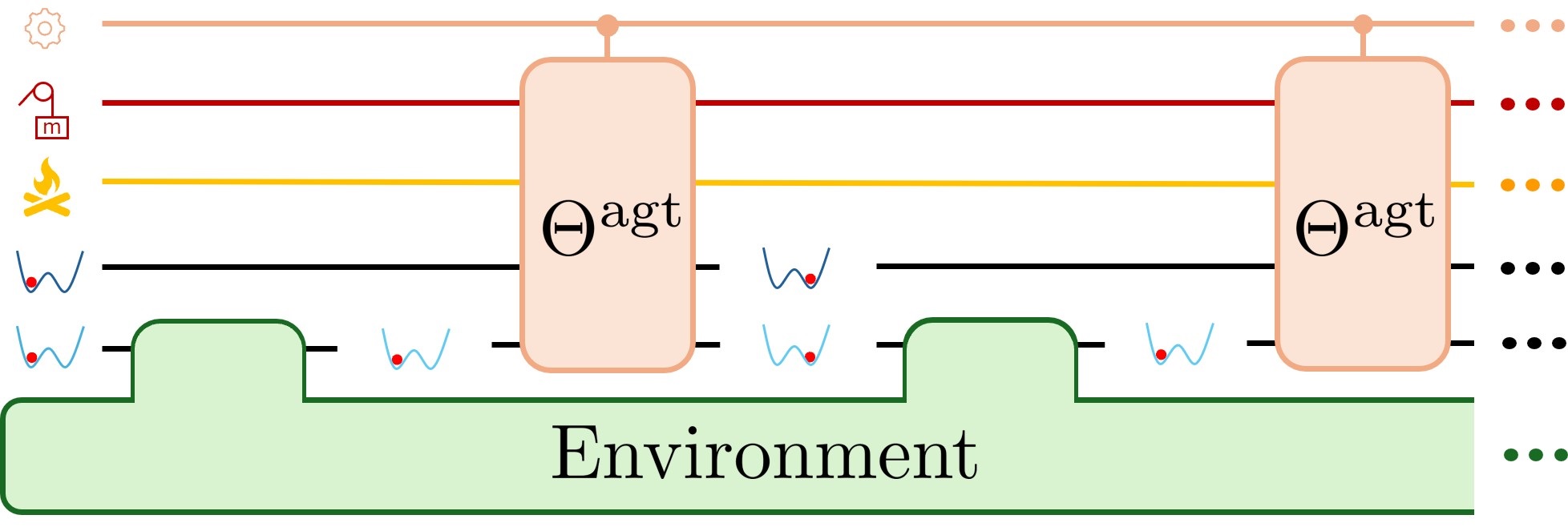}
		\end{minipage}
		\hspace{0.1\linewidth}
		\begin{minipage}[t]{0.25\linewidth}
			\raggedright
			(b)\\
			\vspace{0.2cm}
			\includegraphics[width=\linewidth]{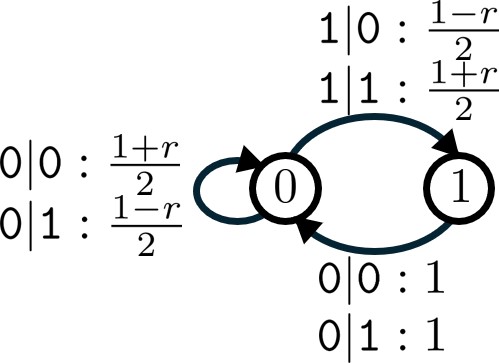}
		\end{minipage}
		\caption{Percept--action loop implementation used in the example. Panel (a) shows the corresponding circuit representation, with time flowing from left to right. The horizontal lines correspond, from top to bottom, to the agent's read-only memory, a work reservoir, a thermal reservoir, the agent's memory (two states), and the percept--action register (two states). For concreteness, we assume that the agent's memory and the percept--action register are physically encoded in information carriers given by overdamped Brownian particles in a controllable potential (as depicted). In each round, the agent's information processing, described by its transition matrix $ \Theta^{\mathrm{agt}} $, is realized through an isothermal control protocol, represented by the box labeled $ \Theta^{\mathrm{agt}} $. The protocol is designed so that the controllable potentials before and after the implementation of $ \Theta^{\mathrm{agt}} $ coincide (see \cite{boyd2018thermodynamics, boyd2022thermodynamic} for details). The agent's read-only memory stores a description of the percept--action loop $(\texttt{agt}_{\mathcal{M}}, \texttt{env}^{(r)})$, according to which the control protocol is tailored to the environment. The environment channel constitutes the only nonequilibrium resource available to the agent. Panel (b) shows the two-state finite-state machine of the environment, which specifies how actions are mapped to percepts (transition labels follow the format \emph{percept}$,|,$\emph{action}$,:,$\emph{transition probability}).  The \emph{feedback} parameter $r \in [0,1]$ controls the influence of actions on percepts: for $r = 0$, actions have no influence, corresponding to the tape setting. For all $r > 0$, percepts depend on previous actions, yielding a genuine percept--action loop.}
		\label{fig:8}
	\end{figure*}
	
	$\normalfont\mathbb{A}^{\lrstacksmall \texttt{env}}_{\mathrm{mea}}(d)$ is the subset of agent models in $\mathbb{A}^{\lrstacksmall \texttt{env}}(d)$ whose conditional action entropy $\cesaro{H(A_t|M_t)}_t$ reaches its theoretical maximum of $\log_2\abs{\mathcal{A}}$ in the Ces\`aro limit. 
	That is, agent models in $\mathbb{A}^{\lrstacksmall \texttt{env}}_{\mathrm{mea}}(d)$ must perfectly randomize their actions without retaining memory of them. The subscript \emph{mea} stands for \emph{maximum entropy actions}.
	
	$\mathbb{A}^{\lrstacksmall \texttt{env}}_{\mathrm{pred}}(d)$ is the subset of \emph{predictive} agent models; that is, they minimize their predictive gap (see \Cref{def:predictive_main}) subject to the constraint that their memory has size $d$.
	
	$\mathbb{A}^{\lrstacksmall \texttt{env}}_{\mathrm{eff}}(d)$ is the subset of \emph{work-efficient} agent models that maximize work rate (\cref{eq:work_rate}).

	We then have the following
	\begin{theorem}
		\label{th:work_efficient_predictive_main}
		For any product environment channel  $\normalfont\texttt{env}$ and any agent memory size $d\in\mathbb{N}$,
		\begin{align}
			\normalfont\mathbb{A}^{\lrstacksmall \texttt{env}}_{\mathrm{eff}}  (d)=\normalfont\mathbb{A}^{\lrstacksmall \texttt{env}}_{\mathrm{mea}} (d) \cap \normalfont\mathbb{A}^{\lrstacksmall \texttt{env}}_{\mathrm{pred}} (d).
		\end{align}
	\end{theorem}
	We note that whenever agent models can become \emph{maximally} predictive (see \Cref{def:predictive_main}), the results of Boyd et al.~\cite{boyd2018thermodynamics,boyd2022thermodynamic,boyd2025thermodynamic} apply.
	The proof of \Cref{th:work_efficient_predictive_main} (see  
	\Cref{supp:6.4}
	of the Supplemental Material)
	relies on the expression for work capacity in product environments given in \Cref{tab:capacity}. 
	Given the resources for an agent model with $\abs{\mathcal{M}}=d$ interacting with a product environment---that is, it observes a stochastic process---this result establishes two compatible principles for attaining or approaching the optimal work rate: (i) randomizing actions without retaining memory of them and  (ii) employing predictive memory. These two principles can be directly linked to the two terms of the work rate, \cref{eq:work_rate}, where the first principle ensures that $\cesaro{H(A_t|M_t)}_t$ is maximized and the second principle ensures that $\cesaro{H(S_t|M_t)}_t$ is minimized.

	A natural question then arises: what happens when actions influence future percepts---that is, in genuine percept--action loops? Based on \Cref{th:work_efficient_predictive_main}, one might expect the same design principles to apply in these scenarios. In particular, it is not immediately clear why an efficient agent model should not minimize its a.m.~predictive gap, as predicting its percepts reduces the uncertainty $H(S_t | M_t)$, which contributes negatively to the work rate. However, the following theorem demonstrates that there exist environments where neither of these previously identified design principles is compatible with work-efficient agent models. 
	\begin{theorem} \label{th:efficient_but_not_predictive_main}
		There exists an environment channel $\normalfont\texttt{env}$ such that for any $d\in\mathbb{N}$ the sets $\normalfont\mathbb{A}^{\lrstacksmall \texttt{env}}_{\mathrm{pred}}(d)$, $\normalfont\mathbb{A}^{\lrstacksmall \texttt{env}}_{\mathrm{mea}}(d)$, and $\normalfont\mathbb{A}^{\lrstacksmall \texttt{env}}_{\mathrm{eff}}(d)$ are mutually exclusive.
	\end{theorem}
	The proof (see 
	\Cref{supp:6.4}
	of the Supplemental Material)
	is based on a simple memoryless environment (see the inset of \Cref{fig:6}). This simplicity allows us to show that predictive agents retain complete information about their actions, $\cesaro{H(A_t|M_t)}_t = 0$, which results in a nonpositive work rate. In contrast, agents in $\mathbb{A}^{\lrstacksmall \texttt{env}}_{\mathrm{mea}}(d)$ and $\mathbb{A}^{\lrstacksmall \texttt{env}}_{\mathrm{eff}}(d)$ achieve distinct but constant work rates for all $d$. The incompatibility of these work rates implies that the corresponding sets of agent models must be mutually exclusive.

	\begin{figure*}
		\centering
		\begin{tikzpicture}
			\node at (0,0) {\includegraphics[width=\textwidth]{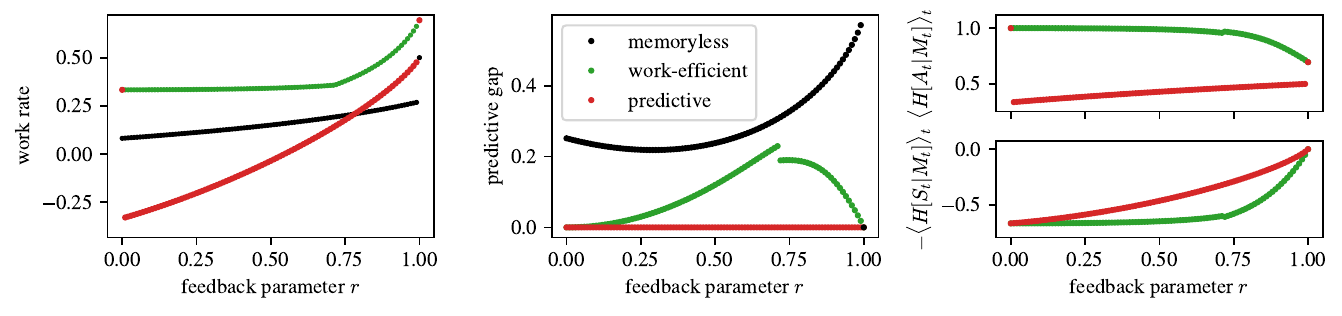}};
			
			\node at (-8.5,2.3) {(a)};
			\node at (-2.5,2.3) {(b)};
			\node at (3.2,2.3) {(c)};
		\end{tikzpicture}
		\caption{Work production in a percept--action loop with memory. The	environment under consideration is shown in \Cref{fig:8}b.
			(a) Numerically optimized work rates of general two-memory-state agents (green) and memoryless agents (black), together with the analytically determined maximal work rate among predictive agents,
			which is independent of the finite memory size $d\geq2$ (red), and 
			(b) their corresponding predictive gaps (see \Cref{def:predictive_main}). The observed discontinuity in the predictive gap of the numerically optimized general agents is associated with a qualitative change in their memory mechanism. 
			(c) Illustration of the trade-off between the percept term $H(S_t\mid M_t)$ and the action term $H(A_t\mid M_t)$ in the work rate (see \cref{eq:work_rate}). Compared with the numerically optimized general agent, the predictive agent has both a lower percept term and a lower action term. Because the reduction in the action term exceeds that in the percept
			term, the predictive agent attains a lower work rate.}
		\label{fig:9}
	\end{figure*}
	This is a central result of this work and highlights a fundamental distinction between the tape setting and the percept--action loop setting. While in the tape setting agents are work-efficient only if they are predictive, the percept--action loop setting admits regimes where prediction is not optimal for resource harvesting: work-efficient agents must \emph{not} be predictive---they must forget aspects of their past actions to operate efficiently.
	
	This phenomenon will be further investigated in the following example, which illustrates the transition from the tape setting to a genuine memoryful percept--action loop.
	\section{Example}  \label{sec:6}
	
	\subsection{Environment}
	We consider a family of environments $\texttt{env}^{(r)}$ indexed by a \emph{feedback parameter} $r \in [0,1]$ with binary actions and percepts, $\mathcal{A} = \mathcal{S} = \{\texttt{0}, \texttt{1}\}$. A two-state memory model of this environment is shown in \Cref{fig:8}b.
	
	\textbf{Tape setting ($r = 0$):}  
	{The environment's control protocol acts on the information carrier received from the agent, which carries the action. For $r=0$, however, the environment ignores the action encoded on that system when preparing the next percept. Instead, the percept is determined solely by the environment's internal rule: if the previous percept was \texttt{0}, the next percept is chosen uniformly at random, with $q(\texttt{0})=q(\texttt{1})=1/2$; if the previous percept was \texttt{1}, the next percept is deterministically \texttt{0}\,. A typical percept sequence might look like $\dotsc \texttt{010001010} \dotsc$.
		
		Because the environment ignores actions, this situation is equivalent to the tape setting of \Cref{fig:1}: rather than having agent and environment alternately apply control protocols to a single shared system, the environment prepares a stream of input systems carrying the percept symbols, and the agent processes these systems sequentially.
		
		This percept process, also known as the \emph{golden mean process} \footnote{The name stems from the following observation. Let $q$ be the probability that the next symbol is \texttt{0} given that the current symbol is \texttt{0}. Then the value of $q$ that maximizes the entropy rate of the golden-mean process equals $(\sqrt{5}-1)/2$, the reciprocal of the golden ratio, and the corresponding entropy rate is $\log_2((1+\sqrt{5})/2)$, the binary logarithm of the golden ratio~\cite{lind2021introduction}.}, exhibits temporal correlations that can be exploited by a vigilant agent equipped with memory. By adapting its memory to the structure of the process, the agent can on average reduce its subjective entropy about incoming percepts, $H(S_t \mid M_t)$, thereby increasing their thermodynamic value.  The agent then maps these percepts to thermal states in order to extract work. Although these thermal states formally constitute actions, they do not actually act on the environment since in the present $r=0$ setting they cannot influence future percepts.}

	The achievable work rate then equals the environment's work capacity:
	\begin{align}
		C^{\mathrm{work}}(\texttt{env}^{(r=0)}) = \frac{1}{3}
	\end{align}
	in units of $ \kb T \ln 2$. This corresponds to a $ (\log_2(27)-4)/3\,\kb T \ln 2 $ gap in work rate relative to the memoryless agent, in agreement with the results reported in \cite{boyd2018thermodynamics}.
	
	\textbf{percept--action loop setting ($r > 0$):}  
	For $r > 0$, the environment includes feedback: the probability $q$ now depends on the previous action $a$, with $q(\texttt{0}) = \frac{1 \pm r}{2}$, where $+$ (respectively $-$) applies to $a = \texttt{0}$ (respectively $a = \texttt{1}$). Consequently, $q(\texttt{1})$ equals $1 - q(\texttt{0})$. As in the tape setting, a percept \texttt{1} is always followed by a percept \texttt{0}. However, as $r$ increases, the agent gains more control over percept generation, reaching full determinism at $r = 1$.
	
	\subsection{Agent Models}
	In practice, for a given percept--action loop $(\texttt{agt}_{\mathcal{M}},\texttt{env}^{(r)})$, one possibility for an agent model to achieve the work rate in \cref{eq:work_rate} is by (i) precomputing the optimal thermodynamic protocols \cite{boyd2018thermodynamics,boyd2022thermodynamic} for the cyclostationary process resulting from $(\texttt{agt}_{\mathcal{M}},\texttt{env}^{(r)})$, (ii) storing these optimal protocol parameters in its read-only memory, and (iii) implementing them during the percept--action loop.
	
	\Cref{fig:9}a shows a numerical estimate of $W^{(2)}(\texttt{env}^{(r)})$, the optimal work rate among all agent models with memory size $d=2$ (green points). We also find that the numerically optimized
	values of $W^{(d)}(\texttt{env}^{(r)})$ are indistinguishable for
	$d=2,3,4$, indicating that no improvement occurs over the tested memory sizes. For comparison, we consider the numerically optimized
	work rate of memoryless agent models (black points), and an analytic solution for the optimal work rate among predictive agent models with $d= 2$ memory states, $\mathbb{A}^{\lrstacksmall \texttt{env}^{(r)}}_{\mathrm{pred}}(2)$  (red points). For predictive agent models, additional memory cannot increase the work rate (see 
	\Cref{supp:7}
	 of the Supplemental Material for details).

	For $r=0$, predictive agents attain the maximum work rate
	$W^{(2)}(\texttt{env}^{(0)})$, consistent with
	\Cref{th:work_efficient_predictive_main}. For $0<r<1$, the optimal
	predictive work rate drops sharply, and our numerical results indicate
	that predictive agent models cease to be work-efficient. This illustrates the separation between prediction and work
	efficiency established by \Cref{th:efficient_but_not_predictive_main}.
	At $r=1$, the environment becomes deterministic in response to
	actions, allowing predictive agent models to once again attain the maximum
	work rate.
	
	As shown in \Cref{fig:9}b, for $0<r<1$ the numerically optimized general agents
	strike a balance between prediction and forgetting: they are less
	predictive than maximally predictive agents but more predictive than
	memoryless ones. Notably, for every $r\in[0,1]$ and every finite $d\geq2$, agents in
	$\mathbb{A}^{\lrstacksmall\texttt{env}^{(r)}}_{\mathrm{pred}}(d)$
	are \emph{maximally predictive}, meaning their predictive gap vanishes, $\cesaro{I[A_{0:t+1}S_{0:t}; S_t | M_t]}_t = 0 $.
	Furthermore, the predictive gap of the numerically optimized general agents exhibits a discontinuous transition near  $r=0.71$ associated with a qualitative change in their memory mechanism
	(see 
	\Cref{supp:7}
	of the Supplemental Material for details).

	The numerically observed separation between the optimized general agents and the optimal predictive agents for $0<r<1$ can be understood from the two terms constituting the work rate (see \Cref{fig:9}c):
	Reducing the predictive gap requires retaining information about past actions that help predict future percepts. While this lowers the percept entropy $H(S_t | M_t)$ (increasing work rate), it also lowers the action entropy $H(A_t | M_t)$ (reducing work rate). Conversely, maximizing action entropy without retaining memory of past actions increases $H(A_t | M_t)$ but may hinder percept prediction or drive the environment into a less-predictable regime, thereby increasing $H(S_t | M_t)$.
	
	Thus, the two design principles for work-efficient agents in the tape setting can no longer be pursued independently in percept--action loops; a trade-off emerges between predictive memory and action forgetfulness.

	\section{Discussion and Future Directions} \label{sec:7}
	Predicting future observations is a central theme across various fields, including Bayesian and active inference \cite{parr2022active}, predictive analytics \cite{larose2015data}, computational mechanics \cite{crutchfield2012between}, and chaos theory \cite{boccaletti2000control}. It also plays a crucial role in modern machine learning, particularly in transformer models and large language models, which are designed to predict future states in a sequence \cite{lin2022survey}.
	
	However, as we show in this work by analyzing the fundamental limits of information processing in percept--action loops, the mere act of remembering the past to predict the future has thermodynamic consequences. 
	This phenomenon arises specifically in percept--action loops with genuine feedback. In such settings, when predicting percepts requires remembering past actions, a trade-off emerges and the goals of prediction and work-efficiency diverge: as we prove and demonstrate with examples, there exist environments in which any agent that maximizes work efficiency must necessarily forget certain aspects of its past actions and, therefore, cannot be maximally predictive.
	
	The present framework invites multiple lines of future research, including generalizations to goal-directed behavior, analyses of energetic dissipation, and extensions into the quantum domain:
	
	\begin{itemize}
		\item \textbf{Agents with goals}---In this work, we considered classes of agents with implicit objectives, such as maximizing work rate or predictive power. A natural next step is to investigate agents with specific goals within our framework. One approach is to fix a desired percept--action behavior, which corresponds to specifying an agent channel. Then, the energetic limits of the agent's behavior can be determined by optimizing over all models that implement this channel (see related ideas in the tape setting \cite{elliott2022quantum, thompson2025energetic}). Alternatively, one could emulate a reinforcement learning scenario by encoding rewards as predictable (i.e., low-entropy) percepts. In this case, an agent aiming to maximize its work rate could be guided toward desired behaviors through suitable reward design.
		\item \textbf{Dissipation in percept--action loops}---If one considers that both the agent and environment thermodynamically implement their respective channels, the agent’s positive work rate implies a corresponding work cost for the environment. In such a setting, the environment converts work into structured correlations, while the agent converts those correlations back into work. For memoryless channels, this conversion can happen without dissipation, with the energetic cost of implementing the environment channel, known in the quantum context as the thermodynamic capacity \cite{navascues2015nonthermal,faist2019thermodynamic}, equaling the work capacity. However, for channels with memory, it remains an open question whether for any percept--action loop $(\texttt{agt}, \texttt{env})$ the agent’s maximum work rate can match the environment’s minimum work cost. Any gap between these values would imply intrinsic entropy production in percept--action loops.
		\item \textbf{Quantum work capacity}---A natural extension of this work is to explore quantum generalizations of work capacity. Our framework admits a quantization by replacing classical channels with quantum combs \cite{chiribella2009theoretical}, enabling analysis of percept--action loops in the quantum domain. This allows for studying fundamental quantum limits on work extraction and the design of quantum-enhanced agents \cite{gu2012quantum,dunjko2016quantum, elliott2022quantum, huang2023engines, zambon2025quantum, thompson2025energetic}.
	\end{itemize}
	Beyond these directions, the findings of this work may also hold more speculative longer-term relevance:
	(i) As artificial intelligence systems begin to approach thermodynamic efficiency limits, memory architectures and action policies may need to be designed with energetic constraints in mind \cite{bennett1982thermodynamics, wolpert2024stochastic}. In such regimes, our results suggest that departing from predictive design principles can be essential for optimizing energy use.
	(ii) In biological systems \cite{ito2015maxwell,shimazaki2015neurons,ouldridge2017thermodynamics,brittain2019biochemical,tjalma2023trade}, behaviors that appear suboptimal from a predictive standpoint may, in fact, reflect underlying thermodynamic constraints. For example, deliberately forgetting aspects of past actions---at the cost of predictive accuracy---may represent strategies for maximizing energetic efficiency.
	
	More broadly,  our work opens the door to a search for new energetic design principles tailored to percept--action loops with feedback. Such considerations may inform novel organizational principles for biological and artificial agents \cite{rupe2024principles}, moving beyond the predictive paradigm \cite{friston2013life, barnett2015computational, boyd2022thermodynamic,briegel2025projective}. \nocite{fiderer2026workcapacitycode}
	
	\section{Acknowledgments}{This research was funded in whole or in part by the Austrian Science Fund (FWF) [SFB BeyondC F7102, DOI: 10.55776/F71; WIT9503323, DOI: 10.55776/WIT9503323]. For open access purposes, the author has applied a CC BY public copyright license to any author accepted manuscript version arising from this submission. We gratefully acknowledge support from the European Union (ERC Advanced Grant, QuantAI, No. 101055129). The views and opinions expressed in this article are however those of the author(s) only and do not necessarily reflect those of the European Union or the European Research Council - neither the European Union nor the granting authority can be held responsible for them. LJF acknowledges support by the Austrian Research Promotion Agency (FFG) and the European Union via NextGeneration EU under Contract Number FO999921407 (HDcode). LJF thanks Benjamin Morris and Andrew Garner for early discussions that helped shape the direction of this work. The authors are grateful to the anonymous reviewers for their insightful and constructive feedback, which helped to significantly improve the quality of the manuscript.}
	\section{Data Availability}{The code used to generate the numerical data underlying \Cref{fig:6} and the example in \Cref{sec:6} is openly available \cite{fiderer2026workcapacitycode}.}

%

	\clearpage
	\onecolumngrid

	\renewcommand{\thesection}{\Alph{section}}
	\renewcommand{\thesubsection}{\thesection.\arabic{subsection}}
	\renewcommand{\thesubsubsection}{\thesubsection.\arabic{subsubsection}}

	\makeatletter
	\def\p@subsection{}
	\def\p@subsubsection{}
	\@addtoreset{equation}{section}
	\@addtoreset{figure}{section}
	\@addtoreset{table}{section}
	\@addtoreset{theorem}{section}
	\@addtoreset{lemma}{section}
	\@addtoreset{definition}{section}
	\@addtoreset{corollary}{section}
	\makeatother
	
	\setcounter{figure}{0}
	\setcounter{table}{0}

	\setcounter{section}{0}
	\setcounter{subsection}{0}
	\setcounter{subsubsection}{0}
	\setcounter{equation}{0}
	\setcounter{figure}{0}
	\setcounter{table}{0}
	\setcounter{theorem}{0}
	\setcounter{lemma}{0}
	\setcounter{definition}{0}
	\setcounter{corollary}{0}
	
	\renewcommand{\thesection}{\Alph{section}}
	\renewcommand{\thesubsection}{\thesection.\arabic{subsection}}
	\renewcommand{\thesubsubsection}{\thesubsection.\arabic{subsubsection}}
	
	\renewcommand{\theequation}{\thesection\arabic{equation}}
	\renewcommand{\thefigure}{\thesection\arabic{figure}}
	\renewcommand{\thetable}{\thesection\arabic{table}}
	
	\renewcommand{\thetheorem}{\thesection\arabic{theorem}}
	\renewcommand{\thelemma}{\thesection\arabic{lemma}}
	\renewcommand{\thedefinition}{\thesection\arabic{definition}}
	\renewcommand{\thecorollary}{\thesection\arabic{corollary}}
	
	\providecommand{\theHsection}{}
	\providecommand{\theHsubsection}{}
	\providecommand{\theHequation}{}
	\providecommand{\theHfigure}{}
	\providecommand{\theHtable}{}
	\providecommand{\theHtheorem}{}
	\providecommand{\theHlemma}{}
	\providecommand{\theHdefinition}{}
	\providecommand{\theHcorollary}{}
	
	\renewcommand{\theHsection}{supp.\arabic{section}}
	\renewcommand{\theHsubsection}{supp.\arabic{section}.\arabic{subsection}}
	\renewcommand{\theHequation}{supp.\Alph{section}.\arabic{equation}}
	\renewcommand{\theHfigure}{supp.\Alph{section}.\arabic{figure}}
	\renewcommand{\theHtable}{supp.\Alph{section}.\arabic{table}}
	\renewcommand{\theHtheorem}{supp.\Alph{section}.\arabic{theorem}}
	\renewcommand{\theHlemma}{supp.\Alph{section}.\arabic{lemma}}
	\renewcommand{\theHdefinition}{supp.\Alph{section}.\arabic{definition}}
	\renewcommand{\theHcorollary}{supp.\Alph{section}.\arabic{corollary}}
	
	\section*{Supplemental Material for
		``Information Thermodynamics of Agents:
		The Work Capacity of Channels with Memory''}
	
	\begin{center}
		Lukas J. Fiderer, Paul C. Barth, Isaac D. Smith, and Hans J. Briegel
	\end{center}
	
	This Supplemental Material provides further mathematical
	background as well as the proofs of the theorems in the main text.
	
	\subsection*{Contents}
	\startcontents[supplement]
	\printcontents[supplement]{section}{1}{}
	
	\section{Some background on probability and information theory}
	\label{supp:1}
	
	\subsection{Notation for random variables and stochastic processes}

	In this section, we establish some of the notation relating to random variables and stochastic processes used throughout the sequel. Random variables will be denoted by capital letters in standard font, i.e. $X, Y, Z,$ etc. The set of values that each random variable can take, also called an {\em alphabet}, will be denoted by capital letters in calligraphic font, i.e. the alphabet for $X$ is $\mathcal{X}$. In this work, we consider finite alphabets and occasionally also countably infinite products of finite alphabets. The elements of these alphabets, at times referred to as {\em symbols}, will be denoted by lower-case letters, i.e. the random variable can take value $x \in \mathcal{X}$. Probability distributions associated to the random variable $X$ will be denoted by $p_{X}$ with $p_{X}(x)$ denoting the probability that $X$ takes value $x$. The subscript may be omitted when the variable to which the distribution refers is clear. 
	
	Given two alphabets $\mathcal{Y}$ and $\mathcal{Z}$ related to random variables $Y$ and $Z$ respectively, we can consider a new random variable $X$ that takes values in $\mathcal{X} := \mathcal{Y} \times \mathcal{Z}$. That is, $X$ takes values $x = (y,z)$ where $y \in \mathcal{Y}$ and $z \in \mathcal{Z}$. Composite random variables of this type can be constructed from any number of constituent variables, which will be useful for, e.g., the treatment of stochastic processes below.
	
	For the purposes of this work, a {\em discrete-time stochastic process} is given by a set of variables $\{X_{t} | t\in \mathbb{N}_{0}\}$ where, for each $t$, the variable $X_{t}$ takes values in the same (finite) alphabet $\mathcal{X}$. This assumption simplifies measure-theoretic treatments such as those in, e.g., \cite{gray2009probability}. 
	Pursuant to the paragraph above, we can associate a new random variable, denoted $\bm{X}$, to the stochastic process, which takes values in $\mathcal{X}^{\mathbb{N}_{0}} := \prod_{t \in \mathbb{N}_{0}} \mathcal{X}$. That is, $\bm{X}$ takes values that are sequences $(x_{0}, x_{1},...)$ with each $x_{t} \in \mathcal{X}$. It will also be convenient to consider random variables associated to {\em sub}sequences in the following way. We use the block notation
	\begin{align}
		X_{l:m} := (X_l, X_{l+1}, \dots, X_{m-1}), \qquad l \le m,
	\end{align}
	where $X_{l:m}$ takes values in $\mathcal{X}^{m-l}:= \prod_{i=l}^{m-1} \mathcal{X}$, that is, values that are tuples $x_{l:m} := (x_{l}, x_{l+1},...,x_{m-1})$ \footnote{Note that the labeling convention is such that the value left of the colon in the subscript is included in the sequence, while the value to the right is not.}.
	If $m=l$, we interpret $X_{l:m}$ as the empty tuple, i.e. a random variable taking a single deterministic value, and if $m = l+1$, then $X_{l:m}=X_l$.  The variable $\bm{X}$ can be considered as the limiting case where $l = 0$ and $m$ goes to infinity.
	
	The notation for distributions associated to the stochastic process $\bm{X}$ and to sequences of variables $X_{l:m}$ follows the same conventions: $p_{\bm{X}}$ denotes a distribution for $
	\bm{X}$ and $p_{\bm{X}}(\bm{x})$ denotes the probability that $\bm{X}$ takes as its value the sequence $\bm{x} := (x_{0},x_{1},...) \in \mathcal{X}^{\mathbb{N}_{0}}$; similarly for $p_{X_{l:m}}$ and $p_{X_{l:m}}(x_{l:m})$. In the cases where $l$ and $m$ are ``close'', we sometimes represent the tuples of variables explicitly. For example, in the case of $m = l+2$, instead of $p_{X_{l:m}}$, we write $p_{X_{l}, X_{l+1}}$. In particular, this allows us to consider the distribution over one part of the subsequence, when the other part takes some value. For example, if $p_{X_{l}, X_{l+1}}$ is known, we can consider the distribution over $X_{l}$ that results if $X_{l+1}$ takes the value $x_{l+1}$, which we denote by $p_{X_{l}, X_{l+1} = x_{l+1}}$.
	
	\subsection{Information theory}
	For an introductory treatment on information theory, see \cite{cover2005elements}; for a measure-theoretic treatment, see \cite{gray2023entropy}.
	\subsubsection{Basic definitions}
	Let $X$ be a random variable with distribution $p_X$. The \emph{Shannon entropy} quantifies the uncertainty associated with $p_X$ as
	\begin{align}
		H_p\left(X\right) \coloneqq \sum_{i\in\mathcal{X}}p_{X}(i)s_p(i)\label{eq:entropy}
	\end{align}
	where $s_p(i) \coloneqq -\log_2 p_{X}(i)$ is known as the \emph{surprise} of obtaining outcome $X=i$ \cite{shannon1948mathematical} (see \cite{hankerson2003introduction} for an elegant axiomatic derivation), and the sum runs over all $i\in\mathcal{X}$ such that $p_{X}(i)\neq 0$. If it is clear from context which distribution $p$ is used to compute entropy, we drop the index and simply write $H\left(X\right)$.
	
	Let $X$ and $Y$ be random variables with joint distribution $p_{XY}$. The \emph{conditional entropy} of $X$ given $Y$ is defined as
	\begin{align}
		H\left(X|Y\right) \coloneqq \sum_{i\in \mathcal{X},j\in \mathcal{Y}} p_{XY}\left(i,j\right) s_p(i|j), \label{eq:conditional_entropy}
	\end{align}
	where $s_p(i|j):=-\log_2 p_{X|Y}\left(i|j\right)$ denotes the conditional surprise of obtaining outcome $i$ given that $j$ has been observed.

	By the convention for block notation, $X_{l:m}$, we have
	\begin{align}
		H(X_{l:l}) = 0
	\end{align}
	and
	\begin{align}
		H(Y \mid X_{l:l}) = H(Y).
	\end{align}
	With this, the chain rule for entropy can be written as
	\begin{align}
		H\bigl(X_{0:n}\bigr)
		= \sum_{t=0}^{n-1} H\bigl(X_t \mid X_{0:t}\bigr).
		\label{eq:chain_rule_entropy}
	\end{align}
	The \textit{mutual information} $I\left[X; Y\right]$ of random variables $X$ and $Y$ is defined as 
	\begin{align}
		I\left[X; Y\right] \coloneqq H\left(X\right) - H\left(X|Y\right), \label{eq:mutual_information}
	\end{align}
	which, with the chain rule of entropy \cref{eq:chain_rule_entropy}, can be written in the symmetric form 
	\begin{align}
		I\left[X; Y\right] = H\left(X Y\right) - H\left(Y|X\right) - H\left(X|Y\right). \label{eq:mutual_information_symmetric}
	\end{align}
	The \textit{conditional mutual information} is then simply defined via the conditional entropy as
	\begin{align}
		I\left[X; Y |Z\right] \coloneqq H\left(X|Z\right) - H\left(X|Y Z\right),\label{eq:def_cond_mut_info}
	\end{align}
	or equivalently in a symmetric form:
	\begin{align}
		I\left[X; Y |Z\right] = H\left(X Y|Z\right) - H\left(Y|X Z\right) - H\left(X|Y Z\right).\label{eq:def_cond_mut_info_sym}
	\end{align}
	We say that $X$ and $Y$ are conditionally independent if $I\left[X;Y|Z\right] = 0$. In fact, $I\left[X;Y|Z\right] = 0$ iff
	\begin{align}
		p_{XY|Z} = p_{X|Z}p_{Y|Z}.
	\end{align}

	The conditional mutual information inherits a chain rule from the corresponding rule for the Shannon entropy above, which can be written as
	\begin{align}
		I\left[X_{0:n};Y|Z\right] = \sum_{t=0}^{n-1} I\left[X_{t};Y|Z X_{0:t}\right]. \label{eq:chain_rule_mutual_information}
	\end{align}
	The chain rule for mutual information is obtained by dropping $Z$ on both sides. Often, we will use the chain rule for a single step:
	\begin{align}
		I[W; XY|Z] = I[W; X|Z] + I[W; Y|XZ]. \label{eq:chain_rule_mutual_information_one_step}
	\end{align}
	The measures of information defined so far are all nonnegative and can be interpreted based on (conditional) surprise, respectively its averaged version, (conditional) entropy. For a consistent treatment of multiple random variables, it is convenient to extend the definition of (conditional) mutual information to more than two arguments. The so-called multivariate mutual information or \textit{interaction information}~\cite{yeung2002first,kolchinsky2022novel} can be defined inductively via
	\begin{align}
		I[X_i;\dots; X_j; X_k] \coloneqq I[X_i; \dots; X_j] - I[X_{i};\dots;X_j|X_k],
	\end{align}
	and similarly \textit{conditional interaction information} via
	\begin{align}
		I[X_i;\dots; X_j; X_k|X_l] \coloneqq I[X_i; \dots; X_j|X_l] - I[X_i;\dots;X_j|X_l X_k].
	\end{align}
	However, it should be noted that multivariate mutual information of three or more variables can assume negative values, which makes it difficult to interpret~\cite{kolchinsky2022novel}.

	\subsubsection{Information diagrams}
	The properties of Shannon's basic measures of information such as entropy and mutual information bear a resemblance to set theory. It has been shown that one can establish a one-to-one correspondence between these measures of information and a (signed) measure on sets \cite{yeung1991new,ting1962amount}. We write $X_1, \dots,X_n$ to denote random variables, and $\tilde{X_1}, \dots,\tilde{X_n}$ for the corresponding sets. The union of sets $\tilde{X_i} \cup \dots \cup \tilde{X_j}$ corresponds to the joint entropy $H\left(X_i, \dots, X_j\right)$, the intersection of sets $\tilde{X_i} \cap \dots \cap \tilde{X_j}$ corresponds to the multivariate mutual information $I\left[X_i; \dots ; X_j\right]$, and the set difference $\tilde{X_i} \setminus \tilde{X_j}$ corresponds to the conditional entropy $H\left(X_i|X_j\right)$. Conditional mutual information $I\left[X_i\dots X_j;S_k\dots S_l|C_n \dots C_m\right]$ corresponds then to $\left( (\tilde{X_i}\cup\dots\cup \tilde{X_{j}}) \cap  (\tilde{S_k}\cup\dots\cup \tilde{S_{l}}) \right) \setminus (\tilde{C_n}\cup\dots\cup \tilde{C_{m}}) $. 
	This correspondence allows us to represent the relations between measures of information in terms of Venn diagrams, whose primary sets correspond to the entropies of single random variables. One example of such an information diagram is given in \Cref{fig:InfoDiagram}.
	
	\begin{figure}[h!]
		\centering
		\begin{tikzpicture}[font=\small]
			\draw[draw = mred, fill=mred, fill opacity=0.3]   (-0.3,0) circle (2.2cm);
			\draw[draw = mblue, fill=mblue, fill opacity=0.3]  (1.8,0) circle (2.2cm);
			\draw[draw = mgreen, fill=mgreen, fill opacity=0.3] (0.75,1.6) circle (2.2cm);
			
			\node at (-1.4, -0.4)    (X) {$H(X|Y,Z)$};
			\node at (2.9, -0.4)   (Y) {$H(Y|X,Z)$};
			\node at (0.75, 2.5)  (Z) {$H(Z|X,Y)$};
			\node[font=\large, mred] at (-3.2,0) {$X$};
			\node[font=\large, mblue] at (4.6,0) {$Y$};
			\node[font=\large, mgreen] at (3.3,2.5) {$Z$};
			\node at (-0.7, 1.5)         {$I[X;Z|Y]$};
			\node at (2.2, 1.5)          {$I[Y;Z|X]$};
			\node at (0.75, -0.9)       {$I[X;Y|Z]$};
			\node at (0.75, 0.5)        {$I[X;Y;Z]$};
		\end{tikzpicture}
		\caption{\label{fig:InfoDiagram} An example of an information diagram. An information diagram illustrates the relations between (conditional) entropies and (conditional) mutual information. }
	\end{figure}

	\subsubsection{Stationary and cyclostationary processes}
	A stochastic process $\bm{X}$ is said to be  \emph{stationary} if its distribution $p_{\bm{X}}$ satisfies \cite[p.87]{gray2009probability}
	\begin{align}
		p_{X_{n:m}}=p_{X_{n+t:m+t}} \label{eq:def_stationary} 
	\end{align}
	for all $n,t\in\mathbb{N}_0$ and $m>n$. Here, $p_{X_{n:m}}$ is obtained from $p_{\bm{X}}$ through marginalization.
	
	A stochastic process is said to be  \emph{cyclostationary with period} $d\in\mathbb{N}$ if its distribution $p_{\bm{X}}$ satisfies \cref{eq:def_stationary}
	for all $t$ which are integer multiples of $d$, and for all $n\in\mathbb{N}_0$ and $m>n$. Stationary processes are cyclostationary processes with period $d=1$.
	\subsubsection{Entropy rate}
	Let $\bm{X}$ be a stochastic process with distribution $p_{\bm{X}}$. Then, the \textit{entropy rate}, a process's degree of intrinsic randomness, is defined as 
	\begin{align}
		h\left(\bm{X}\right) \coloneqq \lim\limits_{n\rightarrow \infty}\frac{H\left(X_{0:n}\right)}{n} \label{eq:entropy_rate}
	\end{align}
	if the limit exists. The entropy rate exists for a broad class of processes known as asymptotically mean stationary processes \cite[Theorem 3.1.1]{gray2023entropy} which contains stationary and cyclostationary processes and, as we will see, also processes which are generated by finite-state hidden Markov models.  
	
	A related quantity to the entropy rate is given by
	\begin{align}
		h'\left(\bm{X}\right) \coloneqq \lim\limits_{n\rightarrow \infty}H\left(X_{n}|X_{0:n}\right) \label{eq:entropy_rate_related}
	\end{align}
	if the limit exists. It is a well-known result that for stationary processes the limit in \cref{eq:entropy_rate_related} exists and equals \cref{eq:entropy_rate} \cite[Theorem 4.2.1]{cover2005elements}. Here, we present a minor generalization of this result for further reference:
	\begin{lemma} \label{lem:entropy_rate_cyclostationary}
		For any cyclostationary process $\bm{X}$ with period $d\in\mathbb{N}$, the limits
		\begin{align}
			\lim\limits_{n\rightarrow \infty}\frac{H\left(X_{0:n}\right)}{n} \label{eq:entropy_rate_proof}
		\end{align}
		and 
		\begin{align}
			h'_r(\bm{X})\coloneqq \lim\limits_{t\rightarrow \infty} H(X_{td+r}|X_{0:td+r}) \label{eq:entropy_rate_related_proof}
		\end{align}
		for all $r\in\{0,1,\dotsc, d-1\}$ exist and
		\begin{align}
			\lim\limits_{n\rightarrow \infty}\frac{H\left(X_{0:n}\right)}{n} =  \frac{1}{d}\sum_{r=0}^{d-1} h'_r(\bm{X}).
		\end{align}
	\end{lemma}
	\begin{proof}
		Since cyclostationary processes are asymptotic mean stationary, \cref{eq:entropy_rate_proof} exists \cite{gray2023entropy}. Further, \cref{eq:entropy_rate_proof} can be expanded as
		\begin{align}
			\lim\limits_{n\rightarrow \infty} \frac{H(X_{0:n})}{n} = \lim\limits_{n\rightarrow \infty} \frac{1}{n}\sum_{t=0}^{n-1}H(X_t|X_{0:t}) \label{eq:entropy_rate_proof_1}
		\end{align}    
		by the chain rule for entropy (see \cref{eq:chain_rule_entropy}). Since the limit exists, it equals the limit of a subsequence, $\lim\limits_{n\rightarrow \infty} \frac{1}{nd}\sum_{t=0}^{dn-1}H(X_t|X_{0:t})$. Further, reindexing this sum  yields
		\begin{align}
			\lim\limits_{n\rightarrow \infty} \frac{1}{d}\sum_{r=0}^{d-1}\frac{1}{n}\sum_{t=0}^{n-1}H(X_{td+r}|X_{0:td+r}).  \label{eq:entropy_rate_proof_2}
		\end{align}
		Since  conditioning reduces entropy, we have
		\begin{align}
			H(X_{td+r}|X_{0:td+r})&\leq H(X_{td+r}|X_{d:td+r})\\
			&= H(X_{(t-1)d+r}|X_{0:(t-1)d+r})
		\end{align}
		where the equality follows from the property of cyclostationarity. Thus, $H(X_{td+r}|X_{0:td+r})$ is non-increasing in $t$. Then, since conditional entropy is nonnegative, the limit
		\begin{align}
			h'_r(\bm{X})\coloneqq \lim\limits_{t\rightarrow \infty} H(X_{td+r}|X_{0:td+r})
		\end{align}
		exists by the monotone convergence theorem. Further, by \cite[Theorem 4.2.3]{cover2005elements}, if $a_n\rightarrow a$ and $b_n=\frac{1}{n}\sum_{t=0}^{n-1}a_t$, then $b_n\rightarrow a$. Thus, 
		\begin{align}
			\lim\limits_{n\rightarrow \infty}\frac{1}{n}\sum_{t=0}^{n-1}H(X_{td+r}|X_{0:td+r})= h'_r(\bm{X})
		\end{align}
		and \cref{eq:entropy_rate_proof_2} then equals
		\begin{align}
			\frac{1}{d}\sum_{r=0}^{d-1}h'_r(\bm{X}).
		\end{align}
	\end{proof}

	Before we proceed, we introduce the following notation for the \emph{Ces\`aro limit}, the limit of the arithmetic mean, which will be used throughout this work \cite{tijms2003first}:
	\begin{align}
		\cesaro{f(t)}_t \coloneqq \lim\limits_{n\rightarrow \infty}\frac{1}n\sum_{t=0}^{n-1}f(t). \label{eq:def_cesaro}
	\end{align}
	The Ces\`aro limit is linear in the sense that 
	\begin{align}
		\cesaro{af(t)+bg(t)}_t = a\cesaro{f(t)}_t+b\cesaro{g(t)}_t \label{eq:cesaro_linearity}
	\end{align}
	whenever both $\cesaro{f(t)}_t$ and $\cesaro{g(t)}_t$ exist, where $a,b\in\mathbb{R}$.

	Using the \emph{Ces\`aro limit}, we can state a chain rule for entropy rate as follows:
	\begin{align}
		\cesaro{H\left({X_t|X_{0:t}}\right)}_t = h\left(\bm{X}\right). \label{eq:entropy_rate_chain_rule}
	\end{align}
	This chain rule is a consequence of \cref{eq:chain_rule_entropy}.
	
	When working on the information theory of stochastic processes, expressions which contain an infinite number of random variables, such as $I[X_{n:\infty};Y|Z]$, are commonly encountered. It should be noted that such expressions are always defined via a limit, that is, $I[X_{n:\infty};Y|Z]=\lim\limits_{m\rightarrow \infty}I[X_{n:m};Y|Z]$. In particular, using the chain rule, one can show that
	\begin{lemma}
		For any $n\in\mathbb{N}_0$,  $I[X_{n:\infty};Y|Z]$ is finite, where $Y$, $Z$, and $X_t$ for all $t$ are finite random variables. 
	\end{lemma}
	
	\begin{proof}
		Using the chain rule of mutual information, \cref{eq:chain_rule_mutual_information}, we find
		\begin{align}
			I[X_{n:\infty};Y|Z] = \lim_{m\rightarrow\infty} \sum_{t=n}^{m-1} I[X_t;Y|ZX_{n:t}]  .
		\end{align}
		Note that, for finite random variables $A,D,S_{n:t}$ for $t \in \{n,...,m-1\}$, all partial sums $\sum_{t=n}^{m-1} I[A;S_t|DS_{n:t}] = I[A;S_{n:m}|D]$ are upper bounded by $H(A)$, and every summand is nonnegative, so the monotone convergence theorem ensures that the limit exists.
	\end{proof}

	\section{Finite-state Markov chains} \label{supp:2}
	
	This appendix reviews some results on finite-state Markov chains from the literature. For a more complete treatment of finite-state Markov chains, we refer the reader to \cite{gallager1996discrete, iosifescu2014finite}.

	Let $\bm{X}$ be a stochastic process with distribution $p$. A (first-order) Markov chain is a stochastic process $\bm{X}$ such that
	\begin{align}
		p_{X_t|X_{0:t}}\left(x_t|x_{0:t-1}, x_{t-1}\right) = p_{X_t|X_{0:t}}\left(x_t|x'_{0:t-1}, x_{t-1}\right)\label{eq:Markov_chain}
	\end{align}
	for any $t \geq 1$, for all $x_{t-1}, x_t\in\mathcal{X}$ and for any $x_{0:t-1},x'_{0:t-1}\in\mathcal{X}^{t-1}$.
	
	The Markov chain is said to be \emph{finite-state} if $\mathcal{X}$ is finite, and it is said to be \emph{homogeneous} if \cref{eq:Markov_chain} does not depend on time $t$, that is, $p_{X_t|X_{t-1}}( j| i) = p_{X_{t'}|X_{t'-1}}(j|i)$ for all $i,j \in \mathcal{X}$ and $t,t' \in \mathbb{N}$. We then write $\phi(j|i) \coloneqq p_{X_t|X_{t-1}}(  j| i)$ and call $\phi(j|i)$ the  \emph{transition probability}. For the remainder of this appendix, it is assumed that all finite-state Markov chains are homogeneous. 
	
	Finite-state Markov chains are thus conveniently characterized by their initial distribution $p_{X_0}$ and a $|\mathcal{X}|\times |\mathcal{X}|$ \emph{stochastic matrix} $\Phi=\left(\phi(j|i)\right)_{j,i\in\mathcal{X}}$. The matrix $\Phi$ is also called a \emph{transition matrix}. We use the convention that $\Phi$ is a left stochastic matrix, i.e., that each column of $\Phi$ must sum to one. This convention is more common in the physical literature, while the mathematical literature such as \cite{iosifescu2014finite} often uses the convention that $\Phi$ is right stochastic. Results can be translated from one convention to the other by a simple transposition.

	In what follows, we introduce some theory of finite-state Markov chains which will be needed to understand in what sense finite-state Markov chains are well-behaved in the asymptotic time limit $t\rightarrow \infty$.
	The probability to reach state $j$ after $n$ steps starting from state $i$ is given by	$\left(\Phi^n\right)_{j,i}$.  If $i$ is a return state, i.e., $\left(\Phi^n\right)_{i,i}>0$ for some $n\geq 1$, we define its \emph{period} $d_i$ as the greatest common divisor of all natural numbers $n$ such that  $\left(\Phi^n\right)_{i,i}>0$ \cite[p.\,81]{iosifescu2014finite}. Further, the \emph{first passage time} to state $j$ is defined as
	\begin{align}
		T^{\mathrm{first}}_{j} \coloneqq \min\{t\geq 1|X_t=j\},
	\end{align}
	where $T^{\mathrm{first}}_{j}$ takes values in $\mathbb{N}\cup \{\infty\}$. Note that the first passage time is a random variable. Define $f(n,i,j)$ as the probability $p\{T^{\mathrm{first}}_{j} = n|X_0=i\}$ that $T^{\mathrm{first}}_{j} = n$ given that the chain started in state $i$. Then, $(f(n,i,j)| n=1,2,\dotsc, \infty)$ is the probability distribution of $T^{\mathrm{first}}_{j}$ given that the Markov chain started in state $i$ \cite[p.86]{iosifescu2014finite}.
	
	Let
	\begin{align}
		f(i,j) \coloneqq  \sum_{n=1}^\infty f(n,i,j) =p\{T^{\mathrm{first}}_{j} <\infty|X_0=i\},
	\end{align}
	which is the probability that the chain reaches state $j$ at some finite time. If this probability is $1$ for some state $i$, i.e., $f(i,i)=1$, meaning that the chain is guaranteed to return to $i$ eventually, then $i$ is called \emph{recurrent}. Otherwise, state $i$ is called \emph{transient}, i.e., there is a nonzero probability that the chain will never return to $i$ \cite[p.88]{iosifescu2014finite}.
	
	For a recurrent state $i$, define its mean recurrence time by
	\begin{align}
		m(i,i) \coloneqq \sum_{n=1}^{\infty}n f(n,i,i).
	\end{align}
	Since the Markov chain is finite-state, $m(i,i)<\infty$ for every recurrent state. For a transient state $i$, we set $m(i,i)\coloneqq\infty$.

	Further, let $f_r(i,j)$ be the probability that $X_t=j$ occurs for at least one $t\equiv r\pmod{d_j}$ given that the chain started in $i$:
	\begin{align}
		f_r(i,j) \coloneqq \sum_{m\geq 0}f(md_j+r,i,j). \label{eq:f_r}
	\end{align}
	We are now in the position to state the following result which characterizes the asymptotic behavior of arbitrary homogeneous finite-state Markov chains.
	\begin{lemma}[Adapted from {\cite[p.~153]{iosifescu2014finite}}] \label{lem:Markov_chains_asymptotic}
		Let $\bm{X}$ be a homogeneous finite-state Markov chain over an alphabet $\mathcal{X}$ with transition matrix $\Phi$. Then, for any state $i\in\mathcal{X}$ and any transient state $j$ we have
		\begin{align}
			\lim\limits_{n\rightarrow\infty}\left(\Phi^n\right)_{j,i} =0. \label{eq:lem_transient_states}
		\end{align}
		Further, for any state $i\in\mathcal{X}$ and any recurrent state $j$ with period $d_j$ we have for all $r_j\in\{1,2,\dotsc, d_j\}$,
		\begin{align}
			\lim\limits_{n\rightarrow\infty}\left(\Phi^{nd_j+r_j}\right)_{j,i} =\frac{ f_{r_j}(i,j)d_j}{m(j,j)}.
		\end{align}
	\end{lemma}
	For a proof see \cite[p.153-154]{iosifescu2014finite}. Note that when compared to \cite[Thm 5.1, p.153]{iosifescu2014finite}, we treat the case of transient states (\cref{eq:lem_transient_states}) separately because for transient states $j$ it can happen that $d_j$ is not well defined if $(\Phi^n)_{j,j}=0$ for every $n\geq 1$.
	\Cref{lem:Markov_chains_asymptotic} states that for each starting state $i$, the probability for the chain to be in a transient state goes to zero as $n\rightarrow \infty$ while the probability for the chain to be in a recurrent state $j$ is periodic with some finite period $d_j$ as $n\rightarrow \infty$.

	The following corollary, which is adapted from \cite[p.154]{iosifescu2014finite}, summarizes some useful consequences of  \Cref{lem:Markov_chains_asymptotic}. We again make use of the notation $\cesaro{\bullet}_t = \lim\limits_{n\rightarrow \infty}\frac{1}n\sum_{t=0}^{n-1}\bullet$.
	\begin{corollary} \label{cor:Markov_chains_asymptotic}
		Let $\bm{X}$ with distribution $p_{\bm{X}}$ be a homogeneous finite-state Markov chain over an alphabet $\mathcal{X}$ with transition matrix $\Phi$. For any recurrent state $i\in\mathcal{X}$, let $d_i$ be its period, and let $d$ be the least common multiple of all $d_i$. Then, 
		\begin{enumerate}[label=(\roman*)]
			\item \emph{[$d$ convergent subsequences]} for all $r\in\{1,2, \dotsc, d\}$ the limit
			\begin{align}
				\Phi^{(r)}_{\infty}\coloneqq \lim\limits_{n\rightarrow\infty} \Phi^{nd+r} \label{eq:cor_Markov_asymptotics_1}
			\end{align}
			exists  and in particular $\Phi^{(r)}_{\infty}=\Phi^r\Phi^{(d)}_{\infty} $,
			\item \emph{[Ces\`aro limit]}  the matrix
			\begin{align}
				\Pi=\cesaro{\Phi^t}_t
			\end{align}
			exists, and its coefficients are given by $\pi_{j,i}=\frac{f(i,j)}{m(j,j)}$.
			\item \emph{[Ces\`aro limit of continuous functions, part 1]}  for any continuous function $g:\mathcal{T}\rightarrow \mathbb{R}$ where $\mathcal{T}$ denotes the set of $|\mathcal{X}|\times |\mathcal{X}|$ transition matrices, $\cesaro{g\left(\Phi^t\right)}_t$ exists and is given by 
			\begin{align}
				\cesaro{g\left(\Phi^t\right)}_t  =  \sum_{r=1}^d g\left(\Phi^{(r)}_\infty\right)/d.
			\end{align}
			\item \emph{ [Ces\`aro limit of continuous functions, part 2]} for any $n\in\mathbb{N}$ and any continuous function $g:\Delta^{\mathcal{X}^n}\rightarrow \mathbb{R}$  where $\Delta^{\mathcal{X}^n} \subset \mathbb{R}^{|\mathcal{X}|^{n}}$ denotes the set of probability distributions over $\mathcal{X}^n$,   $\cesaro{g\left(p_{X_{t:t+n}}\right)}_t$  exists and is given by 
			\begin{align}
				\cesaro{g\left(p_{X_{t:t+n}}\right)}_t  = \frac{1}{d}\sum_{r=1}^{d}g\left(p_{\overline{X}_{r:r+n}}\right) \label{eq:cor_cyclostationary}
			\end{align}
			where $\overline{\bm{X}}$ is a cyclostationary process with period $d$ and distribution $p_{\overline{\bm{X}}}(\bm{x}) = \sum_{i}\phi^{(d)}_{\infty}(x_{0}|i)p_{X_0}(i)\prod_{t=0}^\infty \phi(x_{t+1}|x_t)$, called the \emph{asymptotic process of $\bm{X}$}, where $\phi(j|i)$ (respectively $\phi^{(d)}_{\infty}(j|i)$) are the coefficients of $\Phi$ (respectively $\Phi^{(d)}_{\infty}$), i.e., $\Phi=\left(\phi(j|i)\right)_{j,i}$ and $\Phi^{(d)}_{\infty}=\left(\phi^{(d)}_{\infty}(j|i)\right)_{j,i}$.
		\end{enumerate}
	\end{corollary}
	\emph{Proof.}
	\begin{enumerate}[label=(\roman*)]
		\item Existence follows from \Cref{lem:Markov_chains_asymptotic}, and $\Phi^{(r)}_{\infty}=\Phi^r\Phi^{(d)}_{\infty}$ follows from $\lim\limits_{n\rightarrow\infty} \left(\Phi^{r}\Phi^{nd}\right)=\Phi^{r}\lim\limits_{n\rightarrow\infty} \left(\Phi^{nd}\right)$.
		\item Follows from the fact that by \Cref{cor:Markov_chains_asymptotic}(i) the sequence $(\Phi^t)_t$ has $d$ convergent subsequences each of which has a convergent Ces\`aro limit by the Ces\`aro limit theorem \cite[5.3.1]{edwards1979fourier}, and 
		\begin{align}
			\frac{1}{d_j}\sum_{r=1}^{d_j} \frac{f_r(i,j)d_j}{m(j,j)} = \frac{f(i,j)}{m(j,j)}.
		\end{align}
		\item A function $g$ is continuous if and only if for a convergent sequence $\Pi_n \rightarrow \Pi$ the sequence $g(\Pi_n)$ converges to $g(\Pi)$~\cite[{Thm. 21.3}]{munkres2013topology}. It then follows from \Cref{cor:Markov_chains_asymptotic}(i) that the sequence $g(\Phi^t)$ has $d$ convergent subsequences and therefore has Ces\`aro limit
		\begin{align}
			\cesaro{g\left(\Phi^t\right)}_t = \frac{1}{d} \sum_{r=1}^d g\left(\Phi^{(r)}_\infty\right).
		\end{align}
		\item Note that for any given $n$, $p_{X_{t:t+n}}$ can be written as a linear (and thus continuous) function of $\Phi^t$ for all $t$. With this, the cyclostationarity of $\overline{\bm{X}}$ follows from (i) and the statement of \cref{eq:cor_cyclostationary} follows from (iii).
	\end{enumerate}
	
	Finally, it should be noted that not only finite segments of distributions of Markov chains are asymptotically well behaved (as a consequence of \Cref{cor:Markov_chains_asymptotic}), but also the entropy rate as defined in \cref{eq:entropy_rate}. Entropy rate exists even for broader classes of processes such as deterministic functions of Markov chains, as we now explain.
	Let $\bm{X}$ be a finite-state Markov chain. We say that the process $\bm{Y}$ over a finite alphabet $\mathcal{Y}$  is a deterministic function of $\bm{X}$ if $Y_t=f(X_t)$ for all $t\in\mathbb{N}_0$ where $f:\mathcal{X}\rightarrow \mathcal{Y}$ is a deterministic function. (Note that the class of deterministic functions of finite-state Markov chains is equivalent to finite-state finite-alphabet hidden Markov chains \cite{ephraim2002hidden} in the sense that any deterministic function of a Markov chain can be
	described as a finite-alphabet hidden Markov chain, and any finite-alphabet
	hidden Markov chain can be described as a deterministic function
	of a Markov chain with an augmented state space \cite{ephraim2002hidden}.)
	
	\begin{lemma} \label{lem:entropy_rate_existence}
		Let $\mathcal{X}$ and $\mathcal{Y}$ be finite alphabets, $f:\mathcal{X}\rightarrow\mathcal{Y}$ a map, $\bm{X}$ a finite-state Markov chain on $\mathcal{X}$, and $\bm{Y}=(f(X_0),f(X_1),f(X_2), \dotsc)$.  Then, 
		\begin{align}
			\lim_{n\rightarrow\infty}\frac{H(Y_{0:n})}{n}
		\end{align}
		exists.
	\end{lemma}
	\emph{Proof.} 
	This follows from \cite[theorem 9]{kieffer1981markov} and the entropy ergodic theorem \cite[theorem 3.1.1]{gray2023entropy}. \hfill$\square$\\
	
	\section{Finite-alphabet finite-state hidden Markov channels} \label{supp:3}

	This appendix defines hidden Markov channels in general as well as some special classes of hidden Markov channels.	For a review on hidden Markov processes see \cite{ephraim2002hidden}.

	Let $ \mathcal{X} $ and $ \mathcal{Y} $ denote the finite input and output alphabets, respectively. A discrete-time, finite-alphabet channel is defined as a function
	from input sequences $\bm{x}\in\mathcal{X}^{\mathbb{N}_0}$ to distributions over the channel’s output process, $\bm{Y}$. This function can be represented as a conditional distribution, denoted $\nu_{\bm{Y}|\bm{X}}$.
	Thus, for a fixed input sequence $\bm{X} = \bm{x}$, a channel assigns probabilities $\nu_{\bm{Y}|\bm{X}}(\bm{y}|\bm{x})$ for all output sequences $\bm{y}\in\mathcal{Y}^{\mathbb{N}_0}$.
	
	In the simplest case, $\nu_{\bm{Y}|\bm{X}}$'s inputs are distributed as $p_{\bm{X}}$ such that the joint distribution becomes  $p_{\bm{X}\bm{Y}}  = \nu_{\bm{Y}|\bm{X}}p_{\bm{X}}$. However, note that in a percept-action loop, $\nu_{\bm{Y}|\bm{X}}$'s inputs may depend on (some of)  $\nu_{\bm{Y}|\bm{X}}$'s outputs. This leads to a crucial technical difference to the tape setting: In percept-action loops,  the joint probability distribution over the joint process of inputs and outputs, $\bm{X}\bm{Y}$, is  given by
	\begin{align}
		p_{\bm{X}\bm{Y}}  = \nu_{\bm{Y}|\bm{X}}\eta_{\bm{X}|\bm{Y}}
	\end{align}
	where $\eta_{\bm{X}|\bm{Y}}(\bm{x}|\bm{y})$ is another channel which specifies how the $\nu_{\bm{Y}|\bm{X}}$'s inputs are distributed, \emph{depending} on $\nu_{\bm{Y}|\bm{X}}$'s outputs. In such cases, the distribution
	\begin{align}
		p_{\bm{Y}|\bm{X}}(\bm{y}|\bm{x}) &=  \frac{p_{\bm{X}\bm{Y}}(\bm{x},\bm{y})}{p_{\bm{X}}(\bm{x})}\\
		&= \frac{\nu_{\bm{Y}|\bm{X}}(\bm{y}|\bm{x})\eta_{\bm{X}|\bm{Y}}(\bm{x}|\bm{y})}{\sum_{\bm{y}'\in\mathcal{Y}^{\mathbb{N}_0}}\nu_{\bm{Y}|\bm{X}}(\bm{y}'|\bm{x})\eta_{\bm{X}|\bm{Y}}(\bm{x}|\bm{y}')} \label{eq:pal_distribution_different_from_tape_setting}
	\end{align}
	can be different from $\nu_{\bm{Y}|\bm{X}}(\bm{y}|\bm{x})$. This difference is also the reason we denote the channel by $\nu$ and reserve the symbol $p$ for the joint distribution $p_{\bm{X}\bm{Y}}$ (and distributions which can be obtained from it by marginalizing or conditioning, such as $p_{\bm{Y}|\bm{X}}$ and $p_{\bm{X}}$).
	
	The conditional probability $\nu_{\bm{Y}|\bm{X}}$ thus characterizes the behavior intrinsic to the channel.
	The conditional probability $p_{\bm{Y}|\bm{X}}$ also takes into account how the channel's inputs are prepared
	given the channel's previous outputs. This difference was already pointed out in the context of stochastic thermodynamics with feedback control \cite{sagawa2012nonequilibrium}.
	
	In this work, we focus on a subclass of discrete-time finite-alphabet channels commonly known as finite-state \cite[p.97]{gallager1968information} or hidden Markov channels \cite{ephraim2002hidden}.
	\begin{definition}\label{def:hidden_Markov_channel}
		A (discrete-time finite-alphabet) channel $\nu_{\bm{Y}|\bm{
				X}}$ is a finite-state \emph
		{hidden Markov channel} if there exists a distribution $p_{Z_{0}}$ over a finite set of states $\mathcal{Z}$ and a
		transition matrix $\Phi = \left(\phi(j|i)\right)_{j,i}$ with $i\in\mathcal{X}\times \mathcal{Z}$ and $j\in\mathcal{Y}\times \mathcal{Z}$ such that
		\begin{align}
			\nu_{\bm{Y}|\bm{
					X}}(\bm{y}|\bm{
				x})=\sum_{\bm{z}}p_{Z_{0}}(z_{0})\prod_{t=0}^{\infty}\phi
			\left(y_t,z_{t+1}|x_t,z_t\right), \label{eq:hidden_Markov_channel}
		\end{align} 
		where the sum runs over all $\bm{z}\in\mathcal{Z}^{\mathbb{N}_0}$.
		Then, the tuple $ (\Phi, p_{Z_{0}})$ is called a \emph{hidden Markov model} of $\nu_{\bm{Y}|\bm{
				X}}$, and the elements $z\in\mathcal{Z}$ are called the \emph{hidden states} of the Markov model.
	\end{definition}
	In particular, since any such Markov model defines a hidden Markov channel and any hidden Markov channel by definition has a Markov model, \cref{eq:hidden_Markov_channel} defines a many-to-one correspondence between Markov models and channels.
	
	Further, hidden Markov channels are causal channels \cite[Definition 4]{barnett2015computational} in the sense that
	\begin{align}
		\nu_{Y_{0:n}|\bm{
				X}}(y_{0:n}|x_{0:n}x_{n:\infty}) = \nu_{Y_{0:n}|\bm{X}}(y_{0:n}|x_{0:n}x'_{n:\infty})
	\end{align}
	for all $n\in\mathbb{N}$ and for all future input sequences $x_{n:\infty},x'_{n:\infty}\in\mathcal{X}^{\mathbb{N}_0}$, where $\nu_{Y_{0:n}|\bm{
			X}}(y_{0:n}|\bm{
		x})=\sum_{y_{n:\infty}}\nu_{\bm{Y}|\bm{
			X}}(\bm{y}|\bm{
		x})$. This means that for a complete description of the channel's behavior for the first $n$ rounds (channel uses) it is sufficient to know its past input $x_{0:n}$. 
	In particular, hidden Markov channels can be understood as those causal channels which admit an implementation using only finite memory resources as represented by the finite set of hidden states $\mathcal{Z}$.

	Important classes of channels, which we consider in this work, are the following:
	\begin{definition} \label{def:noiseless_memoryless_invariant_and_source_channels}
		A channel $\nu_{\bm{Y}|\bm{
				X}}$ is said to be
		\begin{itemize}
			\item \emph{an identity channel} if	$\nu_{\bm{Y}|\bm{
					X}}(\bm{y}|\bm{x})=\delta_{\bm{x},\bm{y}}$ and $\mathcal{X}=\mathcal{Y}$ where $\delta_{\bm{x},\bm{y}}$ is a Kronecker delta.
			\item \emph{memoryless invariant} if there exists a  $|\mathcal{Y}|\times |\mathcal{X}|$
			stochastic matrix $\Phi$ such that $\nu_{\bm{Y}|\bm{
					X}}(\bm{y}|\bm{x})=\prod_{t=0}^{\infty}\phi\left(y_t|x_t\right)$.
			\item a \emph{product channel} if $\nu_{\bm{Y}|\bm{
					X}}(\bm{y}|\bm{x}) = \nu_{\bm{Y}|\bm{
					X}}(\bm{y}|\bm{x}')$ for all $\bm{x}, \bm{x}'\in\mathcal{X}^{\mathbb{N}_0}$.
		\end{itemize}
	\end{definition}	
	The output behavior of product channels is fully characterized without knowing their inputs. Thus, they can be understood as an information source which produces a (hidden Markov) process over outputs \cite{ephraim2002hidden}. Product channels are also called completely random channels in the literature \cite[Chapter 9.4.2]{gray2023entropy}.

	\section{The global Markov chain of percept-action loops} \label{supp:4}
	This appendix is concerned with percept-action loops and proves that the global process (involving agent and environment) is Markov. In the following, we refer to the hidden Markov channel of interest as the \emph{environment}, abbreviated as \texttt{env}:
	\begin{align}
		\texttt{env} \coloneqq \nu^{\mathrm{env}}_{\bm{S}|\bm{A}}.
	\end{align}
	The input random variables $A_t$ are called \emph{action} variables taking values $a\in\mathcal{A}$ and the output random variables $S_t$ are called \emph{percept} variables taking values $s\in\mathcal{S}$ ($S$ like \emph{state} or \emph{sensory input} is common nomenclature in reinforcement learning and related fields). For simplicity, we assume that the finite input and output alphabets of \texttt{env} are identical, $\mathcal{A}=\mathcal{S}$. In terms of expressivity of the model, this assumption is not restrictive, as any Markov channel with distinct input and output alphabets can be trivially extended to a channel with a common alphabet for inputs and outputs by embedding both into the larger of the two.
	
	A Markov model of the channel $\nu^{\mathrm{env}}_{\bm{S}|\bm{A}}$ (see \Cref{def:hidden_Markov_channel}), denoted as
	\begin{align}
		\texttt{env}_{\mathcal{Z}} \coloneqq  \left(\Phi^{\mathrm{env}}, p^{\mathrm{env}}_{Z_{0}}\right),
	\end{align}
	is called a (Markov) model of \texttt{env}. 
	
	Hidden Markov \emph{product} channels (see \Cref{def:noiseless_memoryless_invariant_and_source_channels}) represent a special class of environments which we will call \emph{product environment channels}.
	
	Protocols used to interact with environments are called \emph{agents}. 
	In full generality, agents, abbreviated as \texttt{agt}, can be represented as a channel $\eta^{\mathrm{agt}}_{\bm{A}|\bm{S}}$ from percepts to actions.
	Similarly to environments, we assume that agents respect a causal ordering and that they admit an implementation with finite memory. However, there is a small asymmetry between agent and environment: the agent must produce the very first action $A_0$ without being prompted by a percept (in contrast, the environment is prompted with an action before it produces the first percept). On a formal level, this is easily taken into account by defining agents as a hidden Markov channel from percepts $\bm{S}$ to actions $A_{1:\infty}$ where the initial distribution over hidden states is replaced by a suitable joint distribution over hidden states and action $A_0$. For the formal definition of environments, agents, and percept-action loops, we refer the reader to \cref{def:environment_channel,def:agent_channel,def:pal} of the main text.

	\begin{lemma}[Global Markov chain]\label{lem:global_Markov_chain}
		Let $\normalfont{(\texttt{agt}_{\mathcal{M}}, \texttt{env}_{\mathcal{Z}}})$ be a percept-action loop with global distribution as given in \cref{eq:global_distribution} of the main text.
		Then, the stochastic process $\bm{U}$, where
		\begin{align}
			U_t=(M_t,A_t,S_t,Z_t),
		\end{align}
		is a homogeneous finite-state Markov chain which will be called the \emph{global Markov chain} of the percept-action loop.
	\end{lemma}
	\emph{Proof.}
	
	First we check the Markov property, that is,
	\begin{align}
		p(u_n|u_{0:n-1}u_{n-1})&=p(u_n|u'_{0:n-1}u_{n-1}), \label{eq:global_Markov_in_proof}
	\end{align}
	for any $n\geq 1$ and $u_{0:n-1},u'_{0:n-1}$, where $u_n = (m_n, a_n, s_n,z_n)$. Note that for better readability, we drop $p$'s index.

	This is a direct consequence of the Markov property of the Markov models for agent and environment which can be seen as follows. For any $n \geq 1$ we have by the definition of conditional probability
	\begin{align}
		p(u_n|u_{0:n}) = 	\frac{p(u_{0:n+1})}{p(u_{0:n})}, \label{eq:cond_prob}
	\end{align}   
	where, by marginalizing the global distribution of a percept-action loop and writing $u_n$ as $(m_n, a_n, s_n,z_n)$ we get that 
	\begin{align}
		p({u}_{0:n+1}) &= p^{\mathrm{agt}}_{M_0 A_0}(m_0,a_0)p^{\mathrm{env}}_{Z_0}(z_0)\left[\sum_{z_{n+1}}\phi^{\mathrm{env}} (s_n, z_{n+1}|a_n, z_n)\right] \prod_{t=0}^{n-1} \theta^{\mathrm{agt}} (a_{t+1}, m_{t+1}|s_t, m_t) \phi^{\mathrm{env}} (s_t, z_{t+1}|a_t, z_t). \label{eq:spelled_out}
	\end{align}
	Due to the product structure of  \cref{eq:spelled_out}, most terms cancel out when we compute \cref{eq:cond_prob} and we are left with
	\begin{align}
		p(u_n|u_{n-1},u_{n-2},\dotsc) &=\frac{\left[\sum_{z_{n+1}}\phi^{\mathrm{env}}(s_n, z_{n+1}|a_n, z_n)\right]\theta^{\mathrm{agt}}(a_n,m_n|s_{n-1},m_{n-1})\phi^{\mathrm{env}}(s_{n-1}, z_n|a_{n-1}, z_{n-1})}{\sum_{
				z'_n}\phi^{\mathrm{env}}(s_{n-1}, z'_n|a_{n-1}, z_{n-1})}.\label{eq:lem_global_Markov2}
	\end{align}
	Since the right-hand side depends only on variables with time index $n$ and $n-1$, we have shown the Markov chain property, \cref{eq:global_Markov_in_proof}. Further, since the right-hand side is determined by the transition matrices of agent and environment, the Markov chain is homogeneous, and with this the lemma is proven. \hfill $\square$

	\section{Markov conditions for percept-action loops} \label{supp:5}
	Bayesian networks are graphical models that represent probabilistic relationships among random variables using directed acyclic graphs ~\cite{pearl1985bayesian,pearl1988probabilistic,verma1990causal}. They allow for efficient reasoning about conditional independence through d-separation. The notion of d-separation is a key concept in Bayesian networks that determines whether two sets of variables are independent given a third set, based on the structure of the graph. It provides a formal criterion for understanding how information flows through the network. This appendix introduces Bayesian networks in general and shows how to use them for percept-action loops.
	\subsection{Bayesian networks and d-separation}
	Let 
	$\{V_1,\dots,V_n\}$ be a set of $n$ random variables and let $G$ be a directed acyclic graph (DAG) such that for each random variable in $\{V_1,\dots,V_n\}$ there is precisely one node in $G$.
	Let $\operatorname{PA}_j$  be the set of  parents  of $V_j$ 
	and $\operatorname{ND}_j$ the set  of non-descendants of $V_j$ except itself. If $B,C,D$  are sets of  random  variables, $I[B ; C | D]$ is the conditional mutual information with respect to the joint random variables constituting the sets, and 
	$I[B ; C | D] = 0$ means that $B$ is statistically independent of  $C$, given $D$. 
	
	In the following, a \emph{path} is defined as a sequence of nodes connected by edges, regardless of the direction of the edges. The following definition is adapted from \cite{janzing2010causal}.
	\begin{definition}[d-separation] \label{def:d-separation} 
		A path  $\mathfrak{p}$ in  a DAG is said to be blocked by a set of nodes $D$ if
		at least one of the following conditions holds:
		\begin{enumerate}[label=(\roman*)]
			\item $\mathfrak{p}$ contains a chain $X\rightarrow Y \rightarrow Z$ or fork $X\leftarrow Y  \rightarrow Z$ such that the middle node $Y$ is in $D$, or
			\item  $\mathfrak{p}$  contains an inverted fork (or collider) $X\rightarrow Y \leftarrow Z$ such that the middle node $Y$ is not in $D$ and such that
			no descendant of $Y$ is in  $D$.
		\end{enumerate}
		A set $D$ is said to d-separate $B$ from $C$ if and only if  $D$ blocks every
		path from a node in $B$ to a node in $C$.
	\end{definition}
	
	\begin{lemma}[{Equivalent Markov conditions,  \cite[Theorem~3.27]{lauritzen1996graphical}, see also \cite[Lemma 1]{janzing2010causal}}] \label{lem:equivalent_markov_coniditions}
		Let 
		$p(V_1,\dots,V_n)$ be the joint distribution of random variables $V_1,\dots,V_n$ (as always, in this work, with respect to a product measure). 
		Then the following three statements are equivalent:
		
		\begin{enumerate}[label=(\roman*)]
			\item \emph{Recursive form:} $p(V_1,\dots,V_n)$ admits the factorization
			\begin{align}
				p(V_1,\dots,V_n)=\prod_{j=1}^n p(V_j|\operatorname{PA}_j), \label{eq:factorization_condition}
			\end{align}
			where the notation $p(V_j|\operatorname{PA}_j)$ is understood as $p(V_{j})$ if $\operatorname{PA}_j$ is empty.
			\item \emph{Local (or parental) Markov condition:}
			for every node  $V_j$  we have
			\begin{align}
				I[
				V_j ; \operatorname{ND}_j| \operatorname{PA}_j
				] = 0,
			\end{align}
			i.e.,
			it is conditionally independent of its non-descendants (except itself), given its parents.

			\item \emph{Global Markov condition:}
			\begin{align}
				I[B ; C| D]=0
			\end{align}
			for all three sets $B,C,D$ of nodes
			for which $B$ and $C$ are d-separated by $D$.
		\end{enumerate} 
	\end{lemma}
	In the following, we will make extensive use of the notion of compatibility of a distribution with a Bayesian network, which we define as follows.
	\begin{definition}
		Let $p$ be a distribution over a set of variables $W$, and let $G$ be a Bayesian network with nodes $V$ such that $W\subseteq V$. Then, the  distribution $p$ is said to be \emph{Markov compatible} with $G$ if
		\begin{align}
			I[B ; C| D]=0 \label{eq:compatible}
		\end{align}
		for all three sets $B,C,D \subseteq W$ of nodes
		for which $B$ and $C$ are d-separated by $D$.
	\end{definition}
	Since `Markov' is used in several other contexts in this work, we will refer to the above notion simply as ``compatible'' to avoid confusion.
	Note that the conditions given by \cref{eq:compatible} are those global Markov conditions with respect to $G$ which only involve variables of $p$. Compatibility of $p$ with $G$ thus means that the Markov conditions implied by $G$ for the variables of $p$ hold.
	\subsection{d-separation conditions for percept-action loops}
	\begin{figure}[h!]
		\begin{tikzpicture}[>=stealth',shorten >=1pt,node distance=2cm,on grid,auto,state/.style={circle, draw, minimum size=.8cm, inner sep=1pt},font=\small]
			
			\definecolor{morange}{rgb}{0.88,0.61,0.14}
			
			\def \S {Z}
			\def \A {X}
			\def \M {W}
			\def \Z {Y}
			
			\def \H {V}

			\node (Xn) at (0.,0. ) {$\M$};
			\node (An) at (0.,-1.5) {$\A$};
			\node (An1) at (3,-1.5) {$\S$};
			\node (Xn1) at (3,0) {$\Z$};
			\path[-] (An) edge node {} (An1);
			\path[-] (Xn) edge node {} (Xn1);
			\draw[fill=morange!50] (1,0.3) rectangle (2,-1.8);
			\node[centered, black] at (1.5, -0.75 ) {$\Phi$};
			
			\node  at (-0.3,0.8) {\textbf{(a)}};

			\node[state] (Xn) at (6.,0.) {$\M$};
			\node[state] (An) at (6.,-1.5) {$\A$};
			\node[state] (An1) at (9,-1.5) {$\S$};
			\node[state] (Xn1) at (9,0) {$\Z$};
			
			\path[->] (An) edge node {} (An1);
			\path[->] (An) edge node {} (Xn1);
			\path[->] (Xn) edge node {} (An1);
			\path[->] (Xn) edge node {} (Xn1);
			
			\node  at (5.5,0.8) {\textbf{(b)}};

			\node[state] (Xn) at (12.,0.) {$\M$};
			\node[state] (An) at (12.,-1.5) {$\A$};
			\node[state] (An1) at (15,-1.5) {$\S$};
			\node[state] (Xn1) at (15,0) {$\Z$};
			\node[state, minimum size = .6] (Hn) at (13.5,-0.75) {$\H$};

			\path[->] (An) edge node {} (Hn);
			\path[->] (Hn) edge node {} (Xn1);
			\path[->] (Xn) edge node {} (Hn);
			\path[->] (Hn) edge node {} (An1);
			
			\node  at (11.5,0.8) {\textbf{(c)}};
			
		\end{tikzpicture}
		\caption{On finding a compatible Bayesian network of a Markov channel with two inputs and two outputs. (a): Circuit diagram of a memoryless channel with input $(W,X)$ and output $(Y,Z)$ described by transition matrix $\Phi$, (b): a Bayesian network which is \emph{not} compatible with arbitrary $\Phi$, and (c): a Bayesian network with an auxiliary variable $V=YZ$ which is compatible with arbitrary $\Phi$. } \label{fig:Bayesian_net_for_a_single_channel}
	\end{figure}
	One may initially be tempted to think  that the Bayesian network depicted in \Cref{fig:Bayesian_net_for_a_single_channel}b is compatible with any $\Phi = \left(\phi(j|i)\right)_{j,i}$ with $i\in\mathcal{W}\times \mathcal{X}$ and $j\in\mathcal{Y}\times \mathcal{Z}$  (see \Cref{fig:Bayesian_net_for_a_single_channel}a)  in the sense that all local Markov conditions implied by \Cref{fig:Bayesian_net_for_a_single_channel}b hold for any distribution $p_{WXYZ}(w,x,y,z)=\phi(y,z|w,x)p_{WX}(w,x)$. 
	\Cref{fig:Bayesian_net_for_a_single_channel}b implies (by d-separation) conditional independence of $Y$ and $Z$ given their parents, that is,
	\begin{align}
		I[Y;Z|WX] = 0. \label{eq:cond_independence_violation}
	\end{align}
	This condition, however, is easily shown to be violated by a channel which produces correlation  \emph{independent} of the values of $W$ and $X$. For instance, let $\phi(\texttt{0},\texttt{0}|w,x)=\phi(\texttt{1},\texttt{1}|w,x)=1/2$ for all $w\in\mathcal{W}$ and $x\in\mathcal{X}$, where $\texttt{0},\texttt{1}\in\mathcal{Y}$ and $\texttt{0},\texttt{1}\in\mathcal{Z}$. Then, \cref{eq:cond_independence_violation} is clearly not fulfilled.
	
	The problem related to \Cref{fig:Bayesian_net_for_a_single_channel}b can be solved with a little sleight of hand: We introduce an additional variable $V=YZ$, defined as the joint channel output, as depicted in \Cref{fig:Bayesian_net_for_a_single_channel}c  \footnote{In fact, $V$ can be any sufficient statistic of $WX$ about $YZ$ but for our purposes setting $V=YZ$ is the simplest. This solution was also pointed out in~\cite[figure 13a]{klyubin2007representations}.}. This is by no means the only way to address the problem with \Cref{fig:Bayesian_net_for_a_single_channel}b (see, for instance,~\cite[{Lemma 1}]{evans2016graphs}), but it is a particularly simple solution which, as we will see, allows one to use d-separation for percept-action loops.
	With this choice of $V$, we can write
	\begin{align}
		p_{YZV|WX}=p_{Y|V}\,p_{Z|V}\,p_{V|WX}\label{eq:Bayesian_channel_decomposition}
	\end{align}
	where, since $V=YZ$, the conditional distribution $p_{V|WX}$ is given by the transition matrix $\Phi$, $p_{V|WX}(y,z|w,x) = \phi(v|w,x)$ with $v=(y,z)$, and $p_{Y|V}$, $p_{Z|V}$ are delta distributions since
	\begin{align}
		p_{Y|V}(y|v)=p_{Y|YZ}(y|y',z')=\delta_{y,y'},
	\end{align} 
	and similarly for $p_{Z|V}$.
	
	We recover the original channel from \cref{eq:Bayesian_channel_decomposition} through marginalization:
	\begin{align}
		\sum_{v\in\mathcal{V}}p_{YZV|WX}(y,z,v|w,x)&=\sum_{v\in\mathcal{V}}
		p_{Y|V}(y|v)p_{Z|V}(z|v)p_{V|WX}(v|w,x)\\
		&=\sum_{y'\in\mathcal{Y},z'\in\mathcal{Z}}
		p_{Y|YZ}(y|y', z')p_{Z|YZ}(z|y',z')\phi(y',z'|w,x)\\
		&=\sum_{y'\in\mathcal{Y},z'\in\mathcal{Z}}
		\delta_{y,y'}\delta_{z,z'}\phi(y',z'|w,x)\\ 
		&=\phi(y,z|w,x). \label{eq:recover_channel}
	\end{align}
	
	\begin{figure}[h!]
		\begin{tikzpicture}[>=stealth',shorten >=1pt,node distance=2cm,on grid,auto,state/.style={circle, draw, minimum size=.6cm, inner sep=1pt},font=\tiny]
			\def \S {S}
			\def \A {A}
			\def \M {M}
			\def \Z {Z}
			\def \V {V}
			\def \W {W}

			\node at (-0.4,0) {};
			\def \n {4}
			\foreach \s in {0,...,\n}
			{
				\pgfmathtruncatemacro{\x}{3*\s}
				\ifnum \s = 0
				\node[state] (Xn-\s) at (0.75,0.75) {$\M_{0}$};
				\node[state] (An-\s) at (0,-1.5) {$\A_{0}$};
				\node[state] (Bn-\s) at (1.5,-1.5) {$\S_{0}$};
				\node[state] (Zn-\s) at (2.25,-3.75) {$\Z_{1}$};
				\node[state] (Zn0) at (-0.75,-3.75) {$\Z_{0}$};
				\node[state,black!50,minimum size=.5cm] (Tn-0) at (0.75,-2.25) {$W_{0}$};
				\node[state,black!50,minimum size=.5cm] (Fn-0) at (-0.75,-0.75) {$V_{0}$};
				
				\else
				\pgfmathtruncatemacro{\prevstate}{\s-1}
				\pgfmathtruncatemacro{\ms}{\s+1}
				\pgfmathtruncatemacro{\aftestate}{-\s+1}
				\node[state] (Xn-\s) at (3*\s+0.75,0.75) {$\M_{\s}$};
				\node[state] (An-\s) [right=3cm of An-\prevstate] {$\A_{{\s}}$};
				\node[state] (Bn-\s) [right=3cm of Bn-\prevstate] {$\S_{{\s}}$};
				\node[state] (Zn-\s) [right=3cm of Zn-\prevstate] {$\Z_{\ms}$};
				\node[state,black!50,minimum size=.5cm] (Tn-\s) at (0.75+3*\s,-2.25) {$\W_{\s}$};
				\node[state,black!50,minimum size=.5cm] (Fn-\s) at (-0.75+3*\s,-0.75) {$\V_{\s}$};
				\path[->] (Tn-\prevstate) edge node {} (Zn-\prevstate);
				\path[->] (Tn-\prevstate) edge node {} (Bn-\prevstate);
				\path[->] (Xn-\prevstate) edge node {} (Fn-\s);
				\path[->] (Bn-\prevstate) edge node {} (Fn-\s);
				\path[->] (Fn-\prevstate) edge node {} (Xn-\prevstate);
				\path[->] (Fn-\prevstate) edge node {} (An-\prevstate);
				\path[->] (An-\s) edge node {} (Tn-\s);
				\path[->] (Zn-\prevstate) edge node {} (Tn-\s);
				\fi
			}
			\path[->] (Tn-\n) edge node {} (Zn-\n);
			\path[->] (Tn-\n) edge node {} (Bn-\n);
			\path[->] (Fn-\n) edge node {} (Xn-\n);
			\path[->] (Fn-\n) edge node {} (An-\n);
			\foreach \s in {1,...,\n}
			{
				\pgfmathtruncatemacro{\prevstate}{\s-1}
				\pgfmathsetmacro{\ms}{int(\s-1)}
			}
			
			\path[->] (Zn0) edge node {} (Tn-0);
			\path[->] (An-0) edge node {} (Tn-0);
			\path[dotted, very thick, black!50] (Xn-\n) edge node {} (3*\n+1.25,0.25);
			\path[dotted, very thick, black!50] (Bn-\n) edge node {} (3*\n+2.0,-1.0);
			\path[dotted, very thick, black!50] (Zn-\n) edge node {} (3*\n+2.75,-3.25);
		\end{tikzpicture}
		\caption{Bayesian network of a general percept-action loop.}\label{fig:Bayes_net_for_pal}
	\end{figure}
	
	Applying the Bayesian network representation \Cref{fig:Bayesian_net_for_a_single_channel}c to the Markov channels of agent and environment in $(\texttt{agt}_{\mathcal{M}}, \texttt{env}_{\mathcal{Z}})$ leads to the following construction.

	Given a distribution $q_{V}$ over a set of variables $V$ compatible with a Bayesian network $B$, for any subset $W\subseteq V$, let  $\mathcal{G}(q,W)$ denote the set of Markov conditions
	\begin{align}
		I[S ; T| R]=0
	\end{align}
	with respect to $q$
	for all three sets $S,T,R\subset W$ of nodes
	for which $S$ and $T$ are d-separated by $R$.
	\begin{lemma} \label{lem:Bayes_net_compatibility}
		The global distribution of any percept-action loop $\normalfont(\texttt{agt}_{\mathcal{M}}, \texttt{env}_{\mathcal{Z}})$, as given in \cref{eq:global_distribution} of the main text, is compatible with the Bayesian network in \Cref{fig:Bayes_net_for_pal}.
	\end{lemma}
	
	The proof proceeds by constructing a distribution over all variables in \Cref{fig:Bayes_net_for_pal} for which 
	\begin{enumerate}[label=(\roman*)]
		\item  the distribution admits a recursive form in the sense of \Cref{lem:equivalent_markov_coniditions}(i) and thus, by \Cref{lem:equivalent_markov_coniditions}, the global Markov conditions must hold, and 
		\item  the distribution of $(\texttt{agt}_{\mathcal{M}}, \texttt{env}_{\mathcal{Z}})$ is recovered through marginalization.
	\end{enumerate}
	\emph{Proof.}
	By \Cref{lem:equivalent_markov_coniditions}, a distribution over the variables in the Bayesian network shown in \Cref{fig:Bayes_net_for_pal} fulfills the global Markov conditions if and only if it factorizes as
	\begin{align}
		p_{\bm{M}\bm{A}\bm{S}\bm{Z}\bm{V}\bm{W}} = p_{V_0}\,p_{Z_0}\prod_{t=0}^{\infty} p_{M_t|V_t}\,p_{A_t|V_t}\,p_{V_{t+1}| S_t M_t}\,p_{Z_{t+1}|W_t}\,p_{S_t|W_t}\,p_{W_t|A_tZ_t}. \label{eq:global_distribution_pal_extended}
	\end{align}
	Let $p_{\bm{M}\bm{A}\bm{S}\bm{Z}\bm{V}\bm{W}}$ be of the form in \cref{eq:global_distribution_pal_extended}. Then, in particular those  d-separation relations which involve only variables $M_t$, $A_t$, $S_t$, and $Z_t$, $t\in\mathbb{N}_0$, must hold. Thus, all that is left to show is that the distribution of any $(\texttt{agt}_{\mathcal{M}}, \texttt{env}_{\mathcal{Z}})$, as given in \cref{eq:global_distribution} of the main text, can be recovered through marginalization from a distribution of the form in \cref{eq:global_distribution_pal_extended}.
	
	For all $t\in\mathbb{N}_0$, we set $V_t=A_tM_t$ and $W_t=S_tZ_{t+1}$. As a result, the conditional distributions $ p_{M_t|V_t}$, $p_{A_t|V_t}$, $p_{Z_{t+1}|W_t}$, and $p_{S_t|W_t}$ evaluate to delta distributions. Further, let
	\begin{align}
		(\texttt{agt}_{\mathcal{M}}, \texttt{env}_{\mathcal{Z}}) = \left(\Theta^{\mathrm{agt}}, p^{\mathrm{agt}}_{M_{0}A_0}, \Phi^{\mathrm{env}}, p^{\mathrm{env}}_{Z_{0}}\right)
	\end{align}
	be any percept-action loop. Then, let $p_{Z_0} = p^{\mathrm{env}}_{Z_{0}}$, and for all $t\in\mathbb{N}_0$ define
	\begin{align}
		p_{V_{t+1}| S_t M_t}(v_{t+1}|s_t,m_t) &= \theta^{\mathrm{agt}} (a_{t+1},m_{t+1}|s_t, m_t)\text{ for all }v_{t+1} = (a_{t+1},m_{t+1})\text{, and} \label{eq:proof_Bayes_net_compatibility_1}\\ 
		p_{W_{t}| A_t Z_t}(w_t|a_t,z_t) &= \phi^{\mathrm{env}} (s_t,z_{t+1}|a_t, z_t)~~\quad \text{ for all }w_t = (s_t,z_{t+1}). \label{eq:proof_Bayes_net_compatibility_2}
	\end{align}
	For each $t\in\mathbb{N}$, we consider all terms on the right-hand side of \cref{eq:global_distribution_pal_extended} which contain $V_t$ and marginalize:
	\begin{align}
		\sum_{v_t\in\mathcal{V}}p_{M_t|V_t}(m_t|v_t)\,p_{A_t|V_t}(a_t|v_t)\,p_{V_t| S_{t-1} M_{t-1}}(v_t|s_{t-1}m_{t-1}) = \theta^{\mathrm{agt}} (a_{t},m_{t}|s_{t-1}, m_{t-1})
	\end{align}
	which follows from $V_t=A_tM_t$ and \cref{eq:proof_Bayes_net_compatibility_1}.  For each $t\in\mathbb{N}_0$, a similar calculation for all terms on the right-hand side of \cref{eq:global_distribution_pal_extended} containing $W_t$ yields $\phi^{\mathrm{env}} (s_t, z_{t+1}|a_t, z_t)$. Next,  we consider all terms on the right-hand side of \cref{eq:global_distribution_pal_extended} which contain $V_0$ and marginalize:
	\begin{align}
		\sum_{v_0\in\mathcal{V}}p_{V_0}(v_0)p_{M_0|V_0}(m_0|v_0)p_{A_0|V_0}(a_0|v_0) = p_{M_0A_0}(m_0,a_0), \label{eq:proof_Bayes_net_compatibility_last_step}
	\end{align}
	which follows from $V_0=A_0M_0$. Finally, let $p_{V_0}$ be such that $p_{M_0A_0} = p^{\mathrm{agt}}_{M_0A_0}$.
	
	We thus constructed a distribution $p_{\bm{M}\bm{A}\bm{S}\bm{Z}\bm{V}\bm{W}}$ such that marginalizing out $\bm{V}$ and $\bm{W}$ yields the global distribution of $(\texttt{agt}_{\mathcal{M}}, \texttt{env}_{\mathcal{Z}})$ as given in \cref{eq:global_distribution} of the main text.  \hfill $\square$
	
	The following corollary shows that a simplified Bayesian network can be used when the environment is memoryless. Recall that for a memoryless environment $\normalfont\texttt{env}_{\mathcal{Z}}^{\mathrm{memless}}$, there exists a  $|\mathcal{S}|\times |\mathcal{A}|$
	stochastic matrix $\Phi^{\mathrm{env}}$ such that $\nu_{\bm{S}|\bm{
			A}}(\bm{s}|\bm{a})=\prod_{t=0}^{\infty}\phi^{\mathrm{env}}\left(s_t|a_t\right)$ and, thus, the global distribution of any $\normalfont\left(\texttt{agt}_{\mathcal{M}}, \texttt{env}^{\mathrm{memless}}\right)$ is of the form
	\begin{align}
		p_{\bm{M}\bm{A}\bm{S}}(\bm{m},\bm{a},\bm{s}) = p^{\mathrm{agt}}_{M_0 A_0}(m_0,a_0)\prod_{t=0}^{\infty} \theta^{\mathrm{agt}} (a_{t+1},m_{t+1}|s_t, m_t) \phi^{\mathrm{env}} (s_t|a_t). \label{eq:global_distribution_pal_memoryless_env}
	\end{align}
	
	\begin{figure}[h!]
		\centering
		\begin{tikzpicture}[>=stealth',shorten >=1pt,node distance=2cm,on grid,auto,state/.style={circle, draw, minimum size=.6cm, inner sep=1pt},font=\tiny]
			\def \S {S}
			\def \A {A}
			\def \M {M}
			\def \Z {Z}
			\def \V {V}

			\node at (-0.4,0) {};
			\def \n {4}
			\foreach \s in {0,...,\n}
			{
				\pgfmathtruncatemacro{\x}{3*\s}
				\ifnum \s = 0
				\node[state] (Xn-\s) at (0.75,0.75) {$\M_{0}$};
				\node[state] (An-\s) at (0,-1.5) {$\A_{0}$};
				\node[state] (Bn-\s) at (1.5,-1.5) {$\S_{0}$};
				\node[state,black!50,minimum size=.5cm] (Fn-0) at (-0.75,-0.75) {$V_{0}$};
				
				\else
				\pgfmathtruncatemacro{\prevstate}{\s-1}
				\pgfmathtruncatemacro{\ms}{\s+1}
				\pgfmathtruncatemacro{\aftestate}{-\s+1}
				\node[state] (Xn-\s) at (3*\s+0.75,0.75) {$\M_{\s}$};
				\node[state] (An-\s) [right=3cm of An-\prevstate] {$\A_{{\s}}$};
				\node[state] (Bn-\s) [right=3cm of Bn-\prevstate] {$\S_{{\s}}$};
				\node[state,black!50,minimum size=.5cm] (Fn-\s) at (-0.75+3*\s,-0.75) {$\V_{\s}$};
				\path[->] (Xn-\prevstate) edge node {} (Fn-\s);
				\path[->] (Bn-\prevstate) edge node {} (Fn-\s);
				\path[->] (Fn-\prevstate) edge node {} (Xn-\prevstate);
				\path[->] (Fn-\prevstate) edge node {} (An-\prevstate);
				\path[->] (An-\s) edge node {} (Bn-\s);
				\fi
			}
			\path[->] (An-\n) edge node {} (Bn-\n);
			\path[->] (Fn-\n) edge node {} (Xn-\n);
			\path[->] (Fn-\n) edge node {} (An-\n);
			\foreach \s in {1,...,\n}
			{
				\pgfmathtruncatemacro{\prevstate}{\s-1}
				\pgfmathsetmacro{\ms}{int(\s-1)}
			}
			\path[->] (An-0) edge node {} (Bn-0);
			\path[dotted, very thick, black!50] (Xn-\n) edge node {} (3*\n+1.25,0.25);
			\path[dotted, very thick, black!50] (Bn-\n) edge node {} (3*\n+2.0,-1.0);
		\end{tikzpicture}
		\caption{Bayesian network of an agent interacting with a memoryless environment channel.}\label{fig:Bayes_net_for_agent_and_memoryless env}
	\end{figure}
	\begin{corollary} \label{cor:Bayes_net_compatibility_memoryless}
		For any $\normalfont\texttt{env}^{\mathrm{memless}}$, the global distribution given in \cref{eq:global_distribution_pal_memoryless_env} is compatible with the Bayesian network in \Cref{fig:Bayes_net_for_agent_and_memoryless env}.
	\end{corollary}

	\begin{proof}
		The corollary is a special case of \Cref{lem:Bayes_net_compatibility} where the environment is taken care of by setting $p_{S_t|A_t}(s_t|a_t) = \phi^{\mathrm{env}} (s_t|a_t)$ for all $t\in\mathbb{N}_0$.
	\end{proof}
	\begin{figure}[h!]
		\begin{tikzpicture}[>=stealth',shorten >=1pt,node distance=2cm,on grid,auto,state/.style={circle, draw, minimum size=.6cm, inner sep=1pt},font=\tiny]
			\def \S {S}
			\def \A {A}
			\def \M {M}
			\def \Z {Z}
			\def \V {V}
			\def \W {W}

			\node at (-0.4,0) {};
			\def \n {4}
			\foreach \s in {0,...,\n}
			{
				\pgfmathtruncatemacro{\x}{3*\s}
				\ifnum \s = 0
				\node[state] (Xn-\s) at (0.75,0.75) {$\M_{0}$};
				\node[state] (An-\s) at (0,-1.5) {$\A_{0}$};
				\node[state] (Bn-\s) at (1.5,-1.5) {$\S_{0}$};
				\node[state] (Zn-\s) at (2.25,-3.75) {$\Z_{1}$};
				\node[state] (Zn0) at (-0.75,-3.75) {$\Z_{0}$};
				\node[state,black!50,minimum size=.5cm] (Tn-0) at (0.75,-2.25) {$W_{0}$};
				\node[state,black!50,minimum size=.5cm] (Fn-0) at (-0.75,-0.75) {$V_{0}$};
				
				\else
				\pgfmathtruncatemacro{\prevstate}{\s-1}
				\pgfmathtruncatemacro{\ms}{\s+1}
				\pgfmathtruncatemacro{\aftestate}{-\s+1}
				\node[state] (Xn-\s) at (3*\s+0.75,0.75) {$\M_{\s}$};
				\node[state] (An-\s) [right=3cm of An-\prevstate] {$\A_{{\s}}$};
				\node[state] (Bn-\s) [right=3cm of Bn-\prevstate] {$\S_{{\s}}$};
				\node[state] (Zn-\s) [right=3cm of Zn-\prevstate] {$\Z_{\ms}$};
				\node[state,black!50,minimum size=.5cm] (Tn-\s) at (0.75+3*\s,-2.25) {$\W_{\s}$};
				\node[state,black!50,minimum size=.5cm] (Fn-\s) at (-0.75+3*\s,-0.75) {$\V_{\s}$};
				\path[->] (Tn-\prevstate) edge node {} (Zn-\prevstate);
				\path[->] (Tn-\prevstate) edge node {} (Bn-\prevstate);
				\path[->] (Xn-\prevstate) edge node {} (Fn-\s);
				\path[->] (Bn-\prevstate) edge node {} (Fn-\s);
				\path[->] (Fn-\prevstate) edge node {} (Xn-\prevstate);
				\path[->] (Fn-\prevstate) edge node {} (An-\prevstate);
				\path[->] (Zn-\prevstate) edge node {} (Tn-\s);
				\fi
			}
			\path[->] (Tn-\n) edge node {} (Zn-\n);
			\path[->] (Tn-\n) edge node {} (Bn-\n);
			\path[->] (Fn-\n) edge node {} (Xn-\n);
			\path[->] (Fn-\n) edge node {} (An-\n);
			\foreach \s in {1,...,\n}
			{
				\pgfmathtruncatemacro{\prevstate}{\s-1}
				\pgfmathsetmacro{\ms}{int(\s-1)}
			}
			
			\path[->] (Zn0) edge node {} (Tn-0);
			\path[dotted, very thick, black!50] (Xn-\n) edge node {} (3*\n+1.25,0.25);
			\path[dotted, very thick, black!50] (Bn-\n) edge node {} (3*\n+2.0,-1.0);
			\path[dotted, very thick, black!50] (Zn-\n) edge node {} (3*\n+2.75,-3.25);
		\end{tikzpicture}
		\caption{Bayesian network of an agent receiving percepts from a source. This is an edge case of a percept-action loop where the environment is modeled as a product environment channel.}\label{fig:Bayes_net_for_agent_and_source}
	\end{figure}

	Let $\normalfont\texttt{env}$ be a product environment channel. Then, the distribution of any $(\texttt{agt}_{\mathcal{M}}, \texttt{env})  = \left(\Theta^{\mathrm{agt}}, p^{\mathrm{agt}}_{M_0 A_0}, \nu^{\mathrm{env}}_{\bm{S}|\bm{A}}\right)$ takes the form
	\begin{align}
		p_{\bm{M}\bm{A}\bm{S}}(\bm{m},\bm{a},\bm{s}) = \nu^{\mathrm{env}}_{\bm{S}|\bm{A}}(\bm{s}|\bm{a}) p^{\mathrm{agt}}_{M_0 A_0}(m_0,a_0)\prod_{t=0}^{\infty} \theta^{\mathrm{agt}} (a_{t+1},m_{t+1}|s_t, m_t), \label{eq:global_distribution_agent_and_source}
	\end{align}
	and we have the following 
	\begin{lemma} \label{lem:Bayes_net_compatibility_source}
		Let \normalfont{\texttt{env}} be a product environment channel. Then, for any percept--action loop $\normalfont(\texttt{agt}_{\mathcal{M}}, \texttt{env})$ whose distribution $p_{\bm{M}\bm{A}\bm{S}}$ has the form in \cref{eq:global_distribution_agent_and_source}, the Bayesian network in \Cref{fig:Bayes_net_for_agent_and_source} is compatible with $p_{\bm{M}\bm{A}\bm{S}}$.
	\end{lemma}
	The proof is similar to the proof of \Cref{lem:Bayes_net_compatibility}.
	
	\emph{Proof.}
	By \Cref{lem:equivalent_markov_coniditions}, a distribution over the variables in the Bayesian network shown in \Cref{fig:Bayes_net_for_agent_and_source} fulfills the global Markov conditions if and only if it factorizes as
	\begin{align}
		p_{\bm{M}\bm{A}\bm{S}\bm{Z}\bm{V}\bm{W}} = p_{V_0}\,p_{Z_0}\prod_{t=0}^{\infty} p_{M_{t}|V_{t}}\,p_{A_{t}|V_{t}}\,p_{V_{t+1}| S_t M_t}\,p_{Z_{t+1}|W_t}\,p_{S_t|W_t}\,p_{W_t|Z_t}. \label{eq:global_distribution_agent_and_source_extended}
	\end{align}
	Let $p_{\bm{M}\bm{A}\bm{S}\bm{Z}\bm{V}\bm{W}}$ be of the form in \cref{eq:global_distribution_agent_and_source_extended}. Then, in particular those global Markov conditions that involve only variables $M_t$, $A_t$, and $S_t$ , $t\in\mathbb{N}_0$, must hold.

	Further, 
	since product environment channels are hidden Markov channels, by \Cref{def:environment_channel} of the main text
	for any product environment channel there must exist a Markov model $(\Phi^{\mathrm{env}}, p^{\mathrm{env}}_{Z_{0}})$ such that
	\begin{align}
		\nu^{\mathrm{env}}_{\bm{S}|\bm{
				A}}(\bm{s}|\bm{
			a})=\sum_{\bm{z}}p^{\mathrm{env}}_{Z_{0}}(z_{0})\prod_{t=0}^{\infty}\phi^{\mathrm{env}}
		\left(s_t,z_{t+1}|a_t,z_t\right). \label{eq:proof_hidden_source}
	\end{align} 
	Further,  by the definition of product environment channels (\Cref{def:noiseless_memoryless_invariant_and_source_channels}) we have $\nu^{\mathrm{env}}_{\bm{S}|\bm{
			A}}(\bm{s}|\bm{a}) =\nu^{\mathrm{env}}_{\bm{S}|\bm{
			A}}(\bm{s}|\bm{a}')$ for all $\bm{a}, \bm{a}'\in\mathcal{A}^{\mathbb{N}_0}$. Thus, for product environment channels, \cref{eq:proof_hidden_source} must still hold if one sets all actions on the right-hand side in \cref{eq:proof_hidden_source}  to some $a\in\mathcal{A}$. In this case, we obtain
	\begin{align}
		\nu^{\mathrm{env}}_{\bm{S}|\bm{
				A}}(\bm{s}|\bm{
			a})=\sum_{\bm{z}}p^{\mathrm{env}}_{Z_{0}}(z_{0})\prod_{t=0}^{\infty}\tilde\phi^{\mathrm{env}}
		\left(s_t,z_{t+1}|z_t\right), \label{eq:proof_hidden_source_2}
	\end{align} 
	where we defined a new $\abs{\mathcal{S}\times \mathcal{Z}}\times \abs{\mathcal{Z}}$ transition matrix $\tilde\Phi^{\mathrm{env}}$ with coefficients
	\begin{align}
		\tilde\phi^{\mathrm{env}}
		\left(s_t,z_{t+1}|z_t\right) =  \phi^{\mathrm{env}}\left(s_t,z_{t+1}|a,z_t\right).
	\end{align}
	Plugging \cref{eq:proof_hidden_source_2} into  \cref{eq:global_distribution_agent_and_source} yields
	\begin{align}
		p_{\bm{M}\bm{A}\bm{S}}(\bm{m},\bm{a},\bm{s}) =\sum_{\bm{z}}p^{\mathrm{env}}_{Z_{0}}(z_{0})p^{\mathrm{agt}}_{M_0 A_0}(m_0,a_0)\prod_{t=0}^{\infty}\tilde\phi^{\mathrm{env}}
		\left(s_t,z_{t+1}|z_t\right) \theta^{\mathrm{agt}} (a_{t+1},m_{t+1}|s_t, m_t), \label{eq:global_distribution_agent_and_source_rewritten}
	\end{align}
	for the global distribution.
	
	All that is left to show is that the distribution in \cref{eq:global_distribution_agent_and_source_rewritten}
	can be recovered through marginalization from a distribution of the form in \cref{eq:global_distribution_agent_and_source_extended}.
	
	For all $t\in\mathbb{N}_0$, we set $V_t=A_tM_t$ and $W_t=S_tZ_{t+1}$ 
	and       
	\begin{align}
		p_{V_{t+1}| S_t M_t}(v_{t+1}|s_t,m_t) &= \theta^{\mathrm{agt}} (a_{t+1},m_{t+1}|s_t, m_t)\text{ for all }v_{t+1} = (a_{t+1},m_{t+1})\text{, and} \label{eq:proof_Bayes_net_compatibility_3}\\
		p_{W_t| Z_t}(w_t|z_t) &= \tilde\phi^{\mathrm{env}} (s_t,z_{t+1}|z_t)\text{ for all }w_t = (s_t,z_{t+1}). \label{eq:proof_Bayes_net_compatibility_4}
	\end{align}
	For each $t\in\mathbb{N}$, we consider all terms on the right-hand side of \cref{eq:global_distribution_agent_and_source_extended} which contain $V_t$ and marginalize:
	\begin{align}
		\sum_{v_t\in\mathcal{V}}p_{M_t|V_t}(m_t|v_t)\,p_{A_t|V_t}(a_t|v_t)\,p_{V_t| S_{t-1} M_{t-1}}(v_t|s_{t-1}m_{t-1}) = \theta^{\mathrm{agt}} (a_t,m_t|s_{t-1}, m_{t-1})
	\end{align}
	which follows from $V_t=A_tM_t$ and \cref{eq:proof_Bayes_net_compatibility_3}, and thus $p_{M_t|V_t}$ and $p_{A_t|V_t}$ are delta distributions. Similarly,  for each $t\in\mathbb{N}_0$, we consider all terms on the right-hand side of \cref{eq:global_distribution_agent_and_source_extended} which contain $W_t$ and marginalize:
	\begin{align}
		\sum_{w_t\in\mathcal{W}}p_{Z_{t+1}|W_t}(z_{t+1}|w_t)p_{S_t|W_t}(s_t|w_t)p_{W_t|Z_t}(w_t|z_t) = \tilde\phi^{\mathrm{env}} (s_t,z_{t+1}|z_t)
	\end{align}
	which follows from $W_t=S_tZ_{t+1}$ and \cref{eq:proof_Bayes_net_compatibility_4}.
	
	Finally,  we consider all terms on the right-hand side of \cref{eq:global_distribution_agent_and_source_extended} which contain $V_0$ and marginalize:
	\begin{align}
		\sum_{v_0\in\mathcal{V}}p_{V_0}(v_0)p_{M_0|V_0}(m_0|v_0)p_{A_0|V_0}(a_0|v_0) = p_{M_0A_0}(m_0,a_0),
	\end{align}
	which follows from $V_0=A_0M_0$. Finally, let $p_{V_0}$ be such that $p_{M_0A_0} = p^{\mathrm{agt}}_{M_0A_0}$.
	
	We thus constructed a distribution $p_{\bm{M}\bm{A}\bm{S}\bm{Z}\bm{V}\bm{W}}$ such that marginalizing out $\bm{V}$, $\bm{W}$, and $\bm{Z}$ yields \cref{eq:global_distribution_agent_and_source_rewritten}.  \hfill $\square$
	
	In Bayesian networks of percept-action loops, there can in general be infinitely many paths between two nodes $X$ and $Y$, as the total process $\bm{M}\bm{A}\bm{S}\bm{Z}$ extends to the infinite future. However, note that paths that go through nodes that lie in the future of both $X$ and $Y$ must necessarily contain a collider. Those paths are therefore d-separated if the collider and all of its children are not part of the separating set.

	\subsection{Existing approaches to the information theory of percept-action loops}
	In the previous section, we introduced a Bayesian network (\Cref{fig:Bayes_net_for_pal}) for a general class of percept-action loops. Existing information-theoretic treatments of percept-action loops such as \cite{klyubin2007representations,tishby2010information,salge2014empowerment,ay2014causal} also provide Bayesian networks, see for example \cite[Figure 1]{klyubin2007representations}, \cite[Equation 11]{tishby2010information}, \cite[Figure 4.1b]{salge2014empowerment}, and \cite[Figure 4]{ay2014causal}. These Bayesian networks mainly deviate from our network in how the agent dynamics is modeled.
	
	The difference between our network and the ones from the literature can be understood as follows. Since we model the environment  (respectively the agent) with a Markov channel on an input-output and a hidden-state register  (see \Cref{fig:Bayesian_net_for_a_single_channel}) we focus on incoming  and outgoing random variables of this channel while being agnostic to its inner workings. In comparison, from the perspective of our framework, existing approaches model variables \emph{inside} the channel (such as $V$ in \Cref{fig:Bayesian_net_for_a_single_channel}c). For example, we recover the Bayesian network in \cite[Equation 11]{tishby2010information} from \Cref{fig:Bayes_net_for_pal} by considering variables $W_t$ and $V_t$ as the agent's memory while ignoring variables $M_t$ and $Z_t$. While our approach requires the introduction of auxiliary hidden variables $V_t$ and $W_t$ to obtain a compatible Bayesian network, we only need a \emph{single} transition matrix to model the agent (in  \cite{klyubin2007representations,tishby2010information,salge2014empowerment,ay2014causal} two transition matrices are necessary). Accordingly, our model is suitable in those contexts where one wishes to model the environment (respectively the agent) with a single Markov channel on an input-output and a memory register.

	\section{The extractable work in percept-action loops} \label{supp:6}
	
	In this framework, memory is represented by a physical system coupled to a thermal reservoir at temperature $T$. The system possesses a few degrees of freedom, the information-bearing degrees of freedom, which are assumed to be metastable, i.e., their equilibration time $\tau_{\mathrm{info}}$ is much larger than that of the system's other degrees of freedom, $\tau_{\mathrm{others}}$. Information processing on the information-bearing degrees of freedom is carried out through an isothermal protocol, i.e., a protocol executed at constant temperature $T$, with a time scale such that $\tau_{\mathrm{others}} \ll \tau_{\mathrm{protocol}} \ll \tau_{\mathrm{info}}$. The protocol has access to a work reservoir for storing (or retrieving) work.
	
	\subsection{Derivation of work capacity}\label{supp:6.1}
	Let $\mathcal{X}$ be a finite set of information-bearing degrees of freedom of an information reservoir \cite{deffner2013information}, $p_{X_{\mathrm{in}}}$ an arbitrary initial distribution over $\mathcal{X}$, and $\Phi = \left(\phi(j|i)\right)_{j,i}$ an arbitrary transition matrix where $i,j\in\mathcal{X}$. Then, given that the available knowledge about the information-bearing degrees of freedom $\mathcal{X}$ is $p_{X_{\mathrm{in}}}$, the work $W$ which one can expect (with respect to $p_{X_{\mathrm{in}}}$) to extract by implementing an isothermal process realizing $\Phi$ on $\mathcal{X}$
	is upper-bounded by the second law of thermodynamics as \cite{seifert2012stochastic,parrondo2015thermodynamics}
	\begin{align}
		W\leq H\left(X_{\mathrm{out}}\right) - H\left(X_{\mathrm{in}}\right) \label{eq:thermodynamic_assumtion}
	\end{align}
	where $W$ is in units of $\kb T \ln 2$, and $X_{\mathrm{out}}$ is distributed as
	\begin{align}
		p_{X_{\mathrm{out}}}(x_{\mathrm{out}}) = \sum_{x_{\mathrm{in}}\in\mathcal{X}}\phi(x_{\mathrm{out}}|x_{\mathrm{in}})p_{X_{\mathrm{in}}}(x_{\mathrm{in}}),
	\end{align}
	called the output distribution. Note that the upper bound in \cref{eq:thermodynamic_assumtion} can be positive, negative, or zero. In particular, if the expected extracted work  $W$ is negative, realizing the isothermal process requires work; if it is positive, work can be gained.

	If an agent implements an isothermal process such that the expected extracted work equals the upper bound in \cref{eq:thermodynamic_assumtion}, we say that the agent is \emph{Landauer efficient}, in reference to Landauer's bound on the erasure of one bit, which is a special case of \cref{eq:thermodynamic_assumtion}.
	
	Based on the assumption that \cref{eq:thermodynamic_assumtion} holds, we will derive an upper bound on the work an agent $\texttt{agt}_{\mathcal{M}}$ can expect to extract by undergoing a percept-action loop with an environment \texttt{env}.
	
	Let $\left(\texttt{agt}_{\mathcal{M}}, \texttt{env}\right)  = \left(\Theta^{\mathrm{agt}}, p^{\mathrm{agt}}_{M_0 A_0}, \nu^{\mathrm{env}}_{\bm{S}|\bm{A}}\right)$ be a percept-action loop with identical action and percept alphabets $\mathcal{A}=\mathcal{S}$ and memory alphabet $\mathcal{M}$ of the agent. Then, based on \cref{eq:thermodynamic_assumtion}, the work an agent can expect to extract by implementing $\Theta^{\mathrm{agt}}$ in between rounds (channel uses) $t$ and $t+1$ is upper bounded by
	\begin{align}
		W_{t\rightarrow t+1} \leq  H\left(A_{t+1}, M_{t+1}\right)-H\left(S_t, M_t\right).
	\end{align}
	Taking the Ces\`aro limit (for a definition, see \cref{eq:def_cesaro}), we find an upper bound on the expected extracted work per round:
	\begin{align}
		\cesaro{W_{t\rightarrow t+1}}_t \leq  \cesaro{H\left(A_{t+1}, M_{t+1}\right)-H\left(S_t, M_t\right)}_t.
	\end{align}
	It is convenient to regroup terms in the Ces\`aro sum on the right-hand side of this expression:
	\begin{align}
		\cesaro{H\left(A_{t+1}, M_{t+1}\right)-H\left(S_t, M_t\right)}_t &=\lim\limits_{n\rightarrow \infty}\frac{1}{n}\sum_{t=0}^{n-1}\left[H\left(A_{t+1}, M_{t+1}\right)-H\left(S_t, M_t\right)\right]\\
		&=\lim\limits_{n\rightarrow \infty}\frac{1}{n}\left(H\left(A_0, M_0\right)-H\left(A_n, M_n\right)+\sum_{t=0}^{n-1}\left[H\left(A_{t+1}, M_{t+1}\right)-H\left(S_t, M_t\right)\right]\right) \label{eq:derivation_add_term}\\
		&=\lim\limits_{n\rightarrow \infty}\frac{1}{n}\left(\sum_{t=0}^{n-1}\left[H\left(A_t, M_t\right)-H\left(S_t, M_t\right)\right]\right)\\    
		&= \cesaro{H\left(A_t, M_t\right)-H\left(S_t, M_t\right)}_t
	\end{align}
	where in \cref{eq:derivation_add_term} we added the boundary terms $H\left(A_0, M_0\right)$ and $-H\left(A_n, M_n\right)$ inside the Ces\`aro limit. They do not change the result because the alphabets are finite, so both terms divided by $n$ vanish as $n\rightarrow\infty$.
	
	Then, we can rewrite the argument of the Ces\`aro limit using twice the definition of conditional entropy: 
	\begin{align}
		H\left(A_t, M_t\right)-H\left(S_t, M_t\right) &= H\left(A_t| M_t\right)+H\left(M_t\right)-H\left(S_t| M_t\right)-H\left(M_t\right) \\
		&= H\left(A_t| M_t\right)-H\left(S_t| M_t\right). \label{eq:switch_between_work_rate_forms}
	\end{align}
	We define
	\begin{align}
		W(\texttt{agt}_{\mathcal{M}}, \texttt{env}) \coloneqq  \cesaro{H\left(A_t| M_t\right)-H\left(S_t|M_t\right)}_t  \label{eq:extractable_work}
	\end{align}
	as the \emph{work rate}, the a.m.\,extractable work (both in units of $\kb T \ln 2$).
	\subsection{Existence of Landauer-efficient agents} \label{supp:6.2}
	The bound on expected extracted work for a single isothermal implementation of a transition matrix, \cref{eq:thermodynamic_assumtion}, can be reached using efficient protocols. These protocols typically have idealized requirements such as arbitrary energy functions or infinite timescales; see for example \cite{boyd2022thermodynamic} for a protocol based on over-damped Brownian motion in a controllable energy landscape. 
	
	In the following we will outline, for any percept-action loop $\left(\texttt{agt}_{\mathcal{M}}, \texttt{env}\right)$ and provided that such idealized protocols are available, how an implementation of $\texttt{agt}_{\mathcal{M}}$ can be found that extracts all the a.m.\,extractable work in \cref{eq:extractable_work}, using only finite memory. Such agents will be called \emph{Landauer efficient}.

	To this end, recall that any $\left(\texttt{agt}_{\mathcal{M}}, \texttt{env}\right)$ can be represented through a finite-state global Markov process of some $\left(\texttt{agt}_{\mathcal{M}}, \texttt{env}_{\mathcal{Z}}\right)$ which models $\left(\texttt{agt}_{\mathcal{M}}, \texttt{env}\right)$, see \Cref{lem:global_Markov_chain}. By \Cref{cor:Markov_chains_asymptotic}, this process is asymptotically cyclostationary with a finite period in the sense of \Cref{cor:Markov_chains_asymptotic}.  Let $d$ be this period. Then, asymptotically, the transition matrix $\Theta^{\mathrm{agt}}$ is applied only to $d$ distinct joint distributions over memory–percept pairs, namely $\lim_{n \rightarrow \infty} p_{M_{dn+c} S_{dn+c}}$ for $c \in \{1,2,\dotsc,d\}$, and these distributions recur with period $d$. We will now exploit this to construct a Landauer-efficient agent.
	
	Let us extend the agent $\texttt{agt}_{\mathcal{M}}$ by a separate deterministic counter $c$ which starts at 1 and, with every round, if $c<d$ counts up or if $c=d$ resets to 1. This additional counter memory is fully deterministic and thus has zero entropy for all times. It therefore does not contribute to the extractable work.
	
	Now, consider a protocol implementing the agent which, conditioned on the counter $c$, implements one of $d$ efficient protocols optimized for the asymptotically expected distribution $\lim\limits_{n\rightarrow \infty} p_{M_{dn+c}S_{dn+c}}$. Thereby, we have constructed a protocol which implements $\texttt{agt}_{\mathcal{M}}$ in a Landauer-efficient way using only finite memory.
	
	\subsection{Properties of work capacity} \label{supp:6.3}
	In the following we will prove various properties of work capacity. Work capacity is defined in \Cref{def:work_capacity} of the main text.
	
	\begin{theorem} \label{th:properties_of_work_capacity}
		For any environment channel $\normalfont \texttt{env}=\nu^{\mathrm{env}}_{\bm{S}|\bm{A}}$, work capacity $\normalfont C^{\mathrm{work}}(\texttt{env})$ has the following properties:
		\begin{enumerate}[label=(\roman*)]
			\item \emph{(Existence)} $\normalfont C^{\mathrm{work}}(\texttt{env})$ exists,
			\item \emph{(Bounds)} $\normalfont 0\leq C^{\mathrm{work}}(\texttt{env})\leq   \log_2 |\mathcal{S}|$,
			\item \emph{(Subadditivity under channel cascade)} Let $\normalfont\texttt{env}_1 = \nu^{\mathrm{env}_1}_{\bm{X}|\bm{A}}$ and $\normalfont\texttt{env}_2 =\nu^{\mathrm{env}_2}_{\bm{S}|\bm{X}}$ be two hidden Markov channels with matching intermediate alphabet $\mathcal{X}$. Define the cascade $\normalfont \texttt{env}_2\circ \texttt{env}_1 = \nu^{\texttt{env}_2\circ \texttt{env}_1}_{\bm{S}|\bm{A}}$ of $\normalfont\texttt{env}_1$ and $\normalfont\texttt{env}_2$ as
			\begin{align}
				\normalfont	\nu^{\texttt{env}_2\circ \texttt{env}_1}_{\bm{S}|\bm{A}} (\bm{s}|\bm{a})= \sum_{\bm{x}\in\mathcal{X}^{\mathbb{N}_0}}\nu^{\mathrm{env}_2}_{\bm{S}|\bm{X}}(\bm{s}|\bm{x})\nu^{\mathrm{env}_1}_{\bm{X}|\bm{A}} (\bm{x}|\bm{a}),
			\end{align}
			see also \Cref{fig:5} of the main text. Then,
			\begin{align}
				\normalfont C^{\mathrm{work}}(\texttt{env}_2\circ \texttt{env}_1)\leq  C^{\mathrm{work}}(\texttt{env}_1) + C^{\mathrm{work}}(\texttt{env}_2). \label{eq:cascade_inequality}
			\end{align}
		\end{enumerate} 
	\end{theorem}
	
	\emph{Proof of (i)}: \\
	By \Cref{lem:global_Markov_chain}, the global process $\bm{U}= (U_t)_{t=0}^\infty = (M_t,A_t,S_t,Z_t)_{t=0}^\infty$ is a homogeneous Markov chain. Let $\Lambda$ be its transition matrix. Then, with \cref{eq:switch_between_work_rate_forms} work capacity can be rewritten as 
	\begin{align}
		C^{\mathrm{work}}(\texttt{env}) =  \sup_{\texttt{agt}_{\mathcal{M}}\in\mathbb{A}^{\lrstacksmall \texttt{env}}} \cesaro{g\left(p_{U_t}\right)}_t, \label{eq:work_capacity_proof_appendix}
	\end{align}
	where $g$ is a function from the set of all probability distributions over $\mathcal{U}$  to the real numbers given by 
	\begin{align}
		g\left(p_{U_t}\right) = H\left({M_t,A_t}\right)-H\left({M_t,S_t}\right),
	\end{align}
	with
	$p_{M_t A_t}$ and $p_{M_t S_t}$  obtained from $p_{U_t}$ through marginalization. Since $g$ is continuous, $\cesaro{g\left(p_{U_t}\right)}_t$ exists by \Cref{cor:Markov_chains_asymptotic}(iv). The achievable work rates form a nonempty bounded set, so the supremum in \cref{eq:work_capacity_proof_appendix} exists as well. \hfill $\square$

	\emph{Proof of (ii)}: \\
	
	We first prove the upper bound.\\
	
	For all $t\in\mathbb{N}_0$, the summands $H\left({A_t|M_t}\right)-H\left({S_t|M_t}\right)$
	in the expression for work capacity (see \Cref{def:work_capacity} of the main text) are bounded from above as
	\begin{align}
		H\left({A_t|M_t}\right)-H\left({S_t|M_t}\right)\leq \log_2\left|\mathcal{A}\right|-0, \label{eq:upper_bound_capacity_term}
	\end{align}
	where the upper bound of conditional entropy, $H\left({A_t|M_t}\right) \leq H\left({A_t}\right)\leq \log_2 (|\mathcal{A}|)$, was used to obtain an upper bound for the first term, and the nonnegativity of conditional entropy, $0 \leq H\left({S_t|M_t}\right)$, was used to obtain an upper bound for the second term. The upper bound in \cref{eq:upper_bound_capacity_term} depends only on the size of the action-percept alphabet and thus is independent of the choice of the agent Markov model. Applying \cref{eq:upper_bound_capacity_term} to each summand in the work-rate expression \cref{eq:work_rate} of the main text and then taking the supremum over agent models yields the upper bound on work capacity. \\
	
	What is left is the proof for the lower bound.\\
	
	The proof proceeds by showing that for any $\texttt{env}$ there exists an agent model $\texttt{agt}_{\mathcal{M}}$ which has zero extracted work in each step. Consider an agent which implements the identity map from percept $S_t$ to action $A_{t+1}$, that is $p_{A_{t+1}|S_t}(a_{t+1}|s_t)=\delta_{a_{t+1},s_t}$ for all $t\in\mathbb{N}_0$. This implies
	\begin{align}
		H\left(S_t\right) =H\left(A_{t+1}\right) \label{eq:proof_lower_bound_1}
	\end{align}
	for all $t\in \mathbb{N}_0$. Further, note that since \texttt{agt} only employs the identity map there exists a memoryless $\texttt{agt}_{\mathcal{M}}$ (i.e., with $|\mathcal{M}| = 1$) which models it. We thus have 
	\begin{align}
		H\left(S_t|M_t\right) &=H\left(S_t\right) \label{eq:proof_lower_bound_2}\\ 
		H\left(A_{t+1}|M_t\right) &=H\left(A_{t+1}\right).\label{eq:proof_lower_bound_3}
	\end{align}
	Plugging \cref{eq:proof_lower_bound_1,eq:proof_lower_bound_2,eq:proof_lower_bound_3} into the expression for a.m.~extracted work (\cref{eq:extractable_work}) yields zero.
	
	We have thus shown that for any environment there exists an agent with zero extracted work. Since the definition of work capacity involves a supremum with respect to agents, this proves nonnegativity of work capacity for all environments. \hfill $\square$

	\emph{Proof of (iii)}: \\
	Let $\texttt{env}_{12}$ be the channel which is obtained by  alternating between $\texttt{env}_1$ and $\texttt{env}_2$ every round, see \Cref{fig:alternating_environment_proof}. That is, the action and percept processes of $\texttt{env}_{12}$ are
	\begin{align}
		\bm{A}^{12} &= A^{ 1}_0A^{ 2}_0 A^{ 1}_1A^{ 2}_1\dotsm \label{eq:new_input_process}\\
		\bm{S}^{12} &=  S^{1}_0S^{2}_0 S^{1}_1S^{2}_1\dotsm \label{eq:new_otput_process}
	\end{align}
	where $A^{ k}_t$ (respectively $S^{k}_t$) are the inputs (respectively outputs) of \texttt{env}$_k$ where $k=1,2$.
	\begin{figure}
		\centering
		\includegraphics[width=\textwidth]{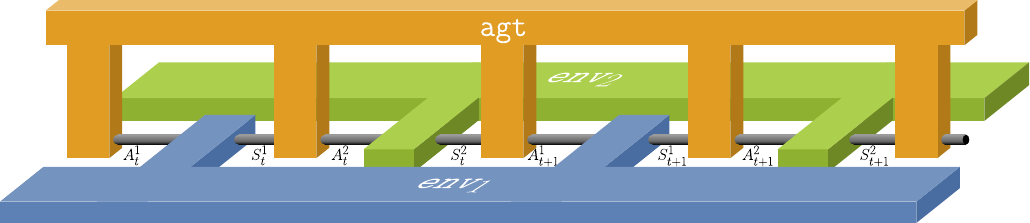}
		\caption{An agent which alternates between using channels $\texttt{env}_1$ and $\texttt{env}_2$.}
		\label{fig:alternating_environment_proof}
	\end{figure}

	Note that we can write the nested suprema in the definition of work capacity as a single supremum $\sup_{\texttt{agt}_{\mathcal{M}}\in\mathbb{A}^{\lrstacksmall \texttt{env}_{12}}}$, where $\mathbb{A}^{\lrstacksmall \texttt{env}_{12}}$ contains agent models of arbitrary memory size. Then, the work capacity of $\texttt{env}_{12}$ is given by
	\begin{align}
		C^{\mathrm{work}}(\texttt{env}_{12}) &=  \sup_{\texttt{agt}_{\mathcal{M}}\in\mathbb{A}^{\lrstacksmall \texttt{env}_{12}}}	\lim\limits_{N\rightarrow\infty}\frac{\sum_{t=0}^{N-1}	\left(H\left({A^{ 1}_t|M^{1}_t}\right)-H\left({S^{ 1}_t|M^{1}_t}\right)+H\left({A^{ 2}_t|M^{2}_t}\right)-H\left({S^{ 2}_t|M^{2}_t}\right)\right)}{2N},
	\end{align}
	where the notation $\bm{M}=M^{1}_0M^{2}_0M^{1}_1M^{2}_1\dotsm$ was used for the agent's memory process in order to match the indexing of  \cref{eq:new_input_process,eq:new_otput_process}. Note that we can distribute the Ces\`aro limit over individual terms (\cref{eq:cesaro_linearity}). Then, by replacing the supremum over a sum of terms with a supremum over individual terms, we obtain an upper bound:
	\begin{align}
		C^{\mathrm{work}}(\texttt{env}_{12}) &\leq  \frac12\sup_{\texttt{agt}_{\mathcal{M}}\in\mathbb{A}^{\lrstacksmall \texttt{env}_{12}}}	\lim\limits_{N\rightarrow\infty}\frac{\sum_{t=0}^{N-1}	\left(H\left({A^{ 1}_t|M^{1}_t}\right)-H\left({S^{ 1}_t|M^{1}_t}\right)\right)}{N}+\\
		&\quad +  \frac12\sup_{\texttt{agt}_{\mathcal{M}}\in\mathbb{A}^{\lrstacksmall \texttt{env}_{12}}}	\lim\limits_{N\rightarrow\infty}\frac{\sum_{t=0}^{N-1}	\left(H\left({A^{ 2}_t|M^{2}_t}\right)-H\left({S^{ 2}_t|M^{2}_t}\right)\right)}{N}\\
		&\leq 	\frac12\left[C^{\mathrm{work}}(\texttt{env}_{1})  + 	C^{\mathrm{work}}(\texttt{env}_{2})\right]  \label{eq:casc_ineq_1}.
	\end{align}
	Further, restrict attention to agents that implement the identity channel from outputs of $\texttt{env}_{1}$ to inputs of $\texttt{env}_{2}$. One two-channel cycle then implements $\texttt{env}_{2}\circ\texttt{env}_{1}$ and contains two ordinary percept--action rounds, so the supremum of the per-round work rates over this restricted class is $\frac12 C^{\mathrm{work}}(\texttt{env}_{2}\circ\texttt{env}_{1})$. Since restricting the set of agents can only decrease the supremum, we have
	\begin{align}
		C^{\mathrm{work}}(\texttt{env}_{12}) &\geq
		\frac12 C^{\mathrm{work}}(\texttt{env}_{2}\circ \texttt{env}_{1}). \label{eq:casc_ineq_2}
	\end{align}
	Then, \cref{eq:cascade_inequality} follows by combining \cref{eq:casc_ineq_1} and \cref{eq:casc_ineq_2}.\hfill $\square$\\

	The following lemma provides simplified expressions for the work capacity (in units of $\kb T \ln 2$) for the restricted classes of environment channels as defined in \Cref{def:noiseless_memoryless_invariant_and_source_channels}.
	In particular, we also consider identity channels with Markovian ROM which are defined as follows. A memoryless invariant environment channel $\nu^{\mathrm{env}}$ is an identity channel with Markovian ROM if there exists a transition matrix $\Phi^{\mathrm{ROM}}=\left(\phi^{\mathrm{ROM}}(j|i)\right)_{j,i}$ with $j\in\mathcal{S}^{\mathrm{ROM}}$ and $i\in\mathcal{A}$ such that the transition matrix $\Phi^{\mathrm{env}} = \left(\phi^{\mathrm{env}}(j,j^{\mathrm{ROM}}|i,i^{\mathrm{ROM}})\right)$ with $i\in\mathcal{A}$, $j\in\mathcal{S}$, and $i^{\mathrm{ROM}},j^{\mathrm{ROM}}\in\mathcal{S}^{\mathrm{ROM}}$, of $\nu^{\mathrm{env}}$ admits $\phi^{\mathrm{env}}(s_t,s^{\mathrm{ROM}}_t|a_t,a^{\mathrm{ROM}}_t) = \delta_{s_t,a_t}\phi^{\mathrm{ROM}}(s^{\mathrm{ROM}}_t|a_t)$.
	\begin{lemma}\label{lem:work_capacity_expressions}
		\begin{align} 
			\normalfont\displaystyle
			C^{\mathrm{work}}(\texttt{env}) = \begin{cases}
				\max\limits_{p_{A_0}}\left[H\left({A_0}\right)-H\left({S_0}\right)\right] &\text{	if \texttt{env} is memoryless invariant,}\\
				0&\text{	if  \texttt{env} is the identity,}\\
				\max_{p_{S_0}} I[S_0^{\mathrm{ROM}};S_0]&\text{	if  \texttt{env} is the identity with Markovian ROM,}\\
				\log_2|\mathcal{A}|-h(\bm{S})&\text{	if \texttt{env} is a  product channel.}\\
			\end{cases} \label{eq:channel_types}
		\end{align}
	\end{lemma}
	\emph{Proof.}\\
	
	(i) Let \texttt{env} be a memoryless invariant channel.\\
	The equality $C^{\mathrm{work}}(\texttt{env}) =                 \max\limits_{p_{A_0}}\left(H\left({A_0}\right)-H\left({S_0}\right)\right)$ will be proven by showing the respective inequalities
	\begin{align}
		C^{\mathrm{work}}(\texttt{env}) &\leq \max\limits_{p_{A_0}}\left[H\left({A_0}\right)-H\left({S_0}\right)\right] \label{eq:proof_mem_less_upper}\\
		C^{\mathrm{work}}(\texttt{env}) &\geq  \max\limits_{p_{A_0}}\left[H\left({A_0}\right)-H\left({S_0}\right)\right] \label{eq:proof_mem_less_lower}.
	\end{align}
	\begin{itemize}
		\item[$\leq$:] The work capacity is given by
		\begin{align}
			C^{\mathrm{work}} &= \sup_{\texttt{agt}_{\mathcal{M}}\in\mathbb{A}^{\lrstacksmall \texttt{env}}}\cesaro{H\left({A_t,M_t}\right) - H\left({S_t,M_t}\right)}_t.
		\end{align}
		We can rewrite the argument of the Ces\`aro limit as follows:
		\begin{align}
			H\left({A_t,M_t}\right) - H\left({S_t,M_t}\right) &= H\left(A_t\right) + H\left(M_t\right) -I\left[A_t;M_t\right] - H\left(S_t\right)-H\left(M_t\right)+I\left[S_t;M_t\right] \label{eq:proof_mem_less_work_capacity_1}
		\end{align}
		where we used $H\left(X,Y\right) = H(X)+H(Y)-I[X;Y]$ which is easily checked with an information diagram (see \Cref{supp:1}). Using $I\left[X_1;Y\right]-I\left[X_2;Y\right] = I\left[X_1;Y|X_2\right]-I\left[X_2;Y|X_1\right]$, \cref{eq:proof_mem_less_work_capacity_1} becomes
		\begin{align}
			H\left({A_t,M_t}\right) - H\left({S_t,M_t}\right) &= H\left(A_t\right) -I\left[A_t;M_t|S_t\right] - H\left(S_t\right)+I\left[S_t;M_t|A_t\right]. \label{eq:proof_mem_less_work_capacity_2}
		\end{align}
		\begin{figure}[h!]
			\centering
			\begin{tikzpicture}[>=stealth',shorten >=1pt,node distance=2.5cm, scale= 0.83,on grid,auto, state/.style={circle, draw, minimum size=0.8cm, inner sep=1pt},font=\tiny]
				
				\def \n {4}
				
				\def \S {S}
				\def \A {A}
				\def \M {M}
				
				\def\XCol{{0,0,0,0,0,0,1,0,0,0}}
				\def\SCol{{0,0,0,0,0,0,3,0,0,0}}
				\def\ACol{{0,0,0,0,0,0,2,0,0,0}}
				\def \t {0}
				
				\foreach \step in {0,...,\n} {
					\pgfmathtruncatemacro{\ms}{-\step}
					\pgfmathtruncatemacro{\aftestate}{-\step+1}
					\pgfmathtruncatemacro{\prevstate}{\step-1}
					
					\ifnum \step= 0
					\pgfmathsetmacro{\xc}{\XCol[\n]}
					\pgfmathsetmacro{\sc}{\SCol[\n]}
					\pgfmathsetmacro{\ac}{\ACol[\n]}
					
					\node[state, fill={\ifnum\xc=1 mred!50\else\ifnum\xc=2 mblue!50\else\ifnum\xc=3 mgreen!50\else white\fi\fi\fi}] (Xn-\step) at (0.75,0.75) {$\M_\t$};
					\node[state, fill={\ifnum\ac=1 mred!50\else\ifnum\ac=2 mblue!50\else\ifnum\ac=3 mgreen!50\else white\fi\fi\fi}] (An-\step) at (0,-1.5) {$\A_\t$};
					\node[state, fill={\ifnum\sc=1 mred!50\else\ifnum\sc=2 mblue!50\else\ifnum\sc=3 mgreen!50\else white\fi\fi\fi}] (Bn-\step) at (1.5,-1.5) {$\S_\t$};
					\node[circle, draw=black, fill=black!40, inner sep=2pt] (Fn-0) at (-0.75,-0.75) {};
					\path[->, thick] (Fn-\ms) edge (Xn-\ms);
					\path[->, thick] (Fn-\ms) edge (An-\ms);
					\else
					
					\pgfmathsetmacro{\xc}{\XCol[\n+\step]}
					\pgfmathsetmacro{\sc}{\SCol[\n+\step]}
					\pgfmathsetmacro{\ac}{\ACol[\n+\step]}
					
					\node[state, fill={\ifnum\xc=1 mred!50\else\ifnum\xc=2 mblue!50\else\ifnum\xc=3 mgreen!50\else white\fi\fi\fi}] (Xn-\step) [right of = Xn-\prevstate] {$\M_{\step}$};
					\node[state, fill={\ifnum\ac=1 mred!50\else\ifnum\ac=2 mblue!50\else\ifnum\ac=3 mgreen!50\else white\fi\fi\fi}] (An-\step) [right of = An-\prevstate] {$\A_{\step}$};
					\node[circle, draw=black, fill=black!40, inner sep=2pt] (Fn-\step) [right of = Fn-\prevstate] {};
					
					\ifnum \step< \n
					\node[state, fill={\ifnum\sc=1 mred!50\else\ifnum\sc=2 mblue!50\else\ifnum\sc=3 mgreen!50\else white\fi\fi\fi}] (Bn-\step) [right of = Bn-\prevstate] {$\S_{\step}$};
					\fi
					
					\path[->, thick] (Fn-\step) edge (Xn-\step);
					\path[->, thick] (Fn-\step) edge (An-\step);
					
					\fi

				}
				\foreach \step in {1,...,\n} {
					\pgfmathtruncatemacro{\prevstate}{\step-1}
					\path[->, thick] (Xn-\prevstate) edge (Fn-\step);
					\path[->, thick] (Bn-\prevstate) edge (Fn-\step);
					\path[->, thick] (An-\prevstate) edge (Bn-\prevstate);
				}
				
				\path[dotted, very thick, black!50] (Xn-\n) edge ++(0.7,-0.7);
				\path[dotted, very thick, black!50] (An-\n) edge ++(1.4,0);
				
			\end{tikzpicture}
			\caption{Bayesian network for a memoryless environment channel (\Cref{cor:Bayes_net_compatibility_memoryless}) with colorized d-separation (blue d-separates red and green) used in the proof of \Cref{lem:work_capacity_expressions}.}\label{fig:Bayesian_net_memoryless_env_proof}
		\end{figure}
		Note that $I\left[S_t;M_t|A_t\right]=0$ due to d-separation (see \Cref{fig:Bayesian_net_memoryless_env_proof}). Then, by the non-negativity of the conditional mutual information, we find that
		\begin{align}
			H\left({A_t,M_t}\right) - H\left({S_t,M_t}\right) &\leq H\left(A_t\right) - H\left(S_t\right) \label{eq:proof_mem_less_work_capacity_3}
		\end{align} 
		and thus
		\begin{align}
			C^{\mathrm{work}} &\leq \sup_{\texttt{agt}_{\mathcal{M}}\in\mathbb{A}^{\lrstacksmall \texttt{env}}}\cesaro{H\left(A_t\right) - H\left(S_t\right)}_t. \label{eq:proof_memoryless_used_in_iii}
		\end{align}
		We obtain a further upper bound by taking the supremum of each summand:
		\begin{align}
			C^{\mathrm{work}} &\leq \cesaro{\sup_{\texttt{agt}_{\mathcal{M}}\in\mathbb{A}^{\lrstacksmall \texttt{env}}}\left(H\left(A_t\right) - H\left(S_t\right)\right)}_t. \label{eq:proof_memoryless_1}
		\end{align}
		For a memoryless invariant environment,
		\begin{align}
			p_{M_tA_tS_t}(m,a,s) =\phi^{\mathrm{env}}(s\mid a) p_{M_tA_t}(m,a),
		\end{align}
		so $H(A_t)-H(S_t)$ depends only on $p_{A_t}$, with $p_{S_t}=\Phi^{\mathrm{env}} p_{A_t}$. Furthermore, for any probability distribution $q$ on $\mathcal A$, the memoryless agent defined by $p_{A_0}=q$ and	$\theta^{\mathrm{agt}}(a'\mid s)=q(a')$ satisfies $p_{A_t}=q$ for all	$t$. Thus,
		\begin{align}
			\sup_{\texttt{agt}_{\mathcal M}\in\mathbb A^{\lrstacksmall\texttt{env}}} \left[H(A_t)-H(S_t)\right] =\max_{p_{A_0}} \left[H(A_0)-H(S_0)\right].
		\end{align}
		The maximum exists because the probability simplex is compact and the
		objective is continuous. Substituting this time-independent value into
		\cref{eq:proof_memoryless_1} yields $C^{\mathrm{work}}(\texttt{env}) \leq \max_{p_{A_0}}	\left[H(A_0)-H(S_0)\right]$ which is the desired upper bound.
		\item[$\geq$:] Choose $p^*_{A_0}\in	\operatorname*{argmax}_{p_{A_0}}
		\left[H(A_0)-H(S_0)\right]$. Consider a memoryless agent model that draws its action according to $p^*_{A_0}$ in every round. Its extracted work is $\max\limits_{p_{A_0}}\left[H(A_0)-H(S_0)\right]$. Since any agent's extracted work is a lower bound on the work capacity, this proves the lower bound.
	\end{itemize}

	\Cref{eq:proof_mem_less_upper,eq:proof_mem_less_lower}  imply equality.\\
	
	(ii) Let \texttt{env} be an identity channel.\\
	By \Cref{def:noiseless_memoryless_invariant_and_source_channels}, $\bm{S}=\bm{A}$ for an identity environment channel. Setting $S_t=A_t$ in the expression for memoryless invariant channels (see \cref{eq:channel_types}) yields  $C^{\mathrm{work}}\left(\texttt{env}\right) = 0$. \\
	
	(iii) As we have seen in the proof of (i), work capacity for memoryless invariant channels can be written as in \cref{eq:proof_memoryless_used_in_iii}:
	\begin{align}
		C^{\mathrm{work}} &= \sup_{\texttt{agt}_{\mathcal{M}}\in\mathbb{A}^{\lrstacksmall \texttt{env}}}\cesaro{H\left(S_tS^{\mathrm{ROM}}_{t-1}\right) - H\left(S_tS^{\mathrm{ROM}}_{t}\right)}_t
	\end{align}
	where, since here we are considering identity channels with ROM, we replaced actions and percepts with  $S_tS^{\mathrm{ROM}}_{t-1}$ and $S_tS^{\mathrm{ROM}}_{t}$, respectively. Using $H(X,Y)=H(X)+H(Y)-I[X;Y]$ as in the proof of (i) and the linearity of the Ces\`aro limit (\cref{eq:cesaro_linearity}) we find
	\begin{align}
		C^{\mathrm{work}} &= \sup_{\texttt{agt}_{\mathcal{M}}\in\mathbb{A}^{\lrstacksmall \texttt{env}}}\left(-\cesaro{I\left[S^{\mathrm{ROM}}_{t-1};S_t\right]}_t + \cesaro{I\left[S^{\mathrm{ROM}}_{t};S_t\right]}_t+\cesaro{H\left(S^{\mathrm{ROM}}_{t-1}\right)}_t - \cesaro{H\left(S^{\mathrm{ROM}}_{t}\right)}_t\right),
	\end{align}
	where the third and fourth Ces\`aro limits cancel out. We obtain an upper bound by dropping the negative term. Following the proof of (i), a further upper bound is obtained by applying the supremum to every summand in the Ces\`aro limit of the second term:  
	\begin{align}
		C^{\mathrm{work}} &\leq  \cesaro{\sup_{\texttt{agt}_{\mathcal{M}}\in\mathbb{A}^{\lrstacksmall \texttt{env}}}I\left[S^{\mathrm{ROM}}_{t};S_t\right]}_t.
	\end{align}
	This is the same optimization problem for all $t\in\mathbb{N}_0$ since the environment is memoryless invariant, reducing to a maximum with respect to $p_{A_t}=p_{S_t}$. We then find
	\begin{align}
		\cesaro{\sup_{\texttt{agt}_{\mathcal{M}}\in\mathbb{A}^{\lrstacksmall \texttt{env}}}I\left[S^{\mathrm{ROM}}_{t};S_t\right]}_t =  \max_{p_{S_0}}I\left[S^{\mathrm{ROM}}_0;S_0\right]. \label{eq:proof_identity_with_ROM}
	\end{align}
	Choose $p^*_{S_0}
	\in
	\operatorname*{argmax}_{p_{S_0}}
	I\left[S_0^{\mathrm{ROM}};S_0\right]$ which exists by continuity of the objective on the compact probability simplex. Let a memoryless agent use the corresponding action distribution in every round, independently of $A^{\mathrm{ROM}}_{t}$, so that $I\left[S_t;A^{\mathrm{ROM}}_{t}\right] = I\left[S_t;S^{\mathrm{ROM}}_{t-1}\right] = 0$. Thus the upper bound in \cref{eq:proof_identity_with_ROM} equals the work capacity.
	
	(iv) We start by deriving an expression for the a.m.\,work production under the assumption that the environment is modeled by  a product environment channel. The work rate in units of $\kb T \ln 2$ is given by \cref{eq:extractable_work},
	\begin{align}
		W(\texttt{agt}_{\mathcal{M}}, \texttt{env}) = \cesaro{H\left({A_t|M_t}\right) - H\left({S_t|M_t}\right) }_t. 
	\end{align}
	Rewriting the second term in the Ces\`aro limit using twice the definition of conditional mutual information (\cref{eq:def_cond_mut_info}) we find
	\begin{align}    
		H\left({S_t|M_t}\right) &= H\left({S_t|M_t A_{0:t+1} S_{0:t}}\right) + I\left[A_{0:t+1} S_{0:t}; S_t| M_t\right] \\
		&= H\left({S_t|S_{0:t}}\right) - I\left[S_t; M_tA_{0:t+1} | S_{0:t} \right] + I\left[A_{0:t+1} S_{0:t}; S_t| M_t\right]. \label{eq:source_channel_work_1}
	\end{align}
	Further, the last term in \cref{eq:source_channel_work_1} can be rewritten using the chain rule of mutual information:
	\begin{align}
		I\left[A_{0:t+1} S_{0:t}; S_t| M_t\right] = I\left[S_{0:t}; S_t| M_t\right] + I\left[A_{0:t+1}; S_t| M_tS_{0:t} \right]. \label{eq:source_channel_work_2}
	\end{align}
	The second term on the right-hand side of \cref{eq:source_channel_work_2} and the second term on the right-hand side of \cref{eq:source_channel_work_1} can be rewritten again using the chain rule of mutual information:
	\begin{align}
		I\left[A_{0:t+1}; S_t| M_tS_{0:t} \right]-I\left[S_t; M_tA_{0:t+1} | S_{0:t} \right] =  -I\left[S_t; M_t | S_{0:t} \right]. \label{eq:source_channel_work_3}
	\end{align}
	The term $ I\left[S_t; M_t | S_{0:t} \right]$ vanishes because of the d-separations shown in \Cref{fig:DAG1}.
	
	Taking \cref{eq:source_channel_work_1,eq:source_channel_work_2} and using linearity of the Ces\`aro limit (\cref{eq:cesaro_linearity}) and the chain rule of entropy rate (\cref{eq:entropy_rate_chain_rule}), we find for the work rate:
	\begin{align}
		W(\texttt{agt}_{\mathcal{M}}, \texttt{env}) =\cesaro{H\left({A_t|M_t}\right)}_t -  h\left(\bm{S}\right) - \cesaro{ I\left[S_{{0:t}};S_t|M_t\right]}_t. \label{eq:work_capacity_derivation_proof_1}
	\end{align}
	\begin{figure}
		\centering
		\begin{tikzpicture}[>=stealth',shorten >=1pt,node distance=2.5cm, scale= 0.83,on grid,auto, state/.style={circle, draw, minimum size=0.8cm, inner sep=1pt},font=\tiny]
			
			\def \n {3}
			
			\def \S {S}
			\def \A {A}
			\def \M {M}
			\def \Z {Z}
			\def \t {t}
			
			\def\XCol{{0,0,0,1,0,0,0}}
			\def\SCol{{2,2,2,3,0,0,0}}
			\def\ACol{{0,0,0,0,0,0,0}}
			\def\ZCol{{0,0,0,0,0,0,0}}
			
			\foreach \step in {0,...,\n} {
				\pgfmathtruncatemacro{\ms}{-\step}
				\pgfmathtruncatemacro{\aftestate}{-\step+1}
				\pgfmathtruncatemacro{\prevstate}{\step-1}
				
				\ifnum \step= 0
				\pgfmathsetmacro{\xc}{\XCol[\n]}
				\pgfmathsetmacro{\sc}{\SCol[\n]}
				\pgfmathsetmacro{\ac}{\ACol[\n]}
				\pgfmathsetmacro{\zc}{\ZCol[\n]}
				
				\node[state, fill={\ifnum\xc=1 mred!50\else\ifnum\xc=2 mblue!50\else\ifnum\xc=3 mgreen!50\else white\fi\fi\fi}] (Xn-\step) at (0.75,0.75) {$\M_\t$};
				\node[state, fill={\ifnum\ac=1 mred!50\else\ifnum\ac=2 mblue!50\else\ifnum\ac=3 mgreen!50\else white\fi\fi\fi}] (An-\step) at (0,-1.5) {$\A_\t$};
				\node[state, fill={\ifnum\sc=1 mred!50\else\ifnum\sc=2 mblue!50\else\ifnum\sc=3 mgreen!50\else white\fi\fi\fi}] (Bn-\step) at (1.5,-1.5) {$\S_\t$};
				\node[state, fill={\ifnum\zc=1 mred!50\else\ifnum\zc=2 mblue!50\else\ifnum\zc=3 mgreen!50\else white\fi\fi\fi}] (Zn-\step) at (-0.75,-3.75) {$\Z_\t$};
				\node[circle, draw=black, fill=black!40, inner sep=2pt] (Tn-0) at (0.75,-2.25) {};
				\node[circle, draw=black, fill=black!40, inner sep=2pt] (Fn-0) at (-0.75,-0.75) {};
				\else
				
				\pgfmathsetmacro{\xc}{\XCol[\n+\step]}
				\pgfmathsetmacro{\sc}{\SCol[\n+\step]}
				\pgfmathsetmacro{\ac}{\ACol[\n+\step]}
				\pgfmathsetmacro{\zc}{\ZCol[\n+\step]}
				
				\node[state, fill={\ifnum\xc=1 mred!50\else\ifnum\xc=2 mblue!50\else\ifnum\xc=3 mgreen!50\else white\fi\fi\fi}] (Xn-\step) [right of = Xn-\prevstate] {$\M_{\t+\step}$};
				\node[state, fill={\ifnum\ac=1 mred!50\else\ifnum\ac=2 mblue!50\else\ifnum\ac=3 mgreen!50\else white\fi\fi\fi}] (An-\step) [right of = An-\prevstate] {$\A_{\t+{\step}}$};
				\node[state, fill={\ifnum\zc=1 mred!50\else\ifnum\zc=2 mblue!50\else\ifnum\zc=3 mgreen!50\else white\fi\fi\fi}] (Zn-\step) [right of = Zn-\prevstate] {$\Z_{\t+\step}$};
				\node[circle, draw=black, fill=black!40, inner sep=2pt] (Fn-\step) [right of = Fn-\prevstate] {};
				
				\ifnum \step< \n
				\node[state, fill={\ifnum\sc=1 mred!50\else\ifnum\sc=2 mblue!50\else\ifnum\sc=3 mgreen!50\else white\fi\fi\fi}] (Bn-\step) [right of = Bn-\prevstate] {$\S_{\t+{\step}}$};
				\node[circle, draw=black, fill=black!40, inner sep=2pt] (Tn-\step) [right of = Tn-\prevstate] {};
				\fi
				
				\pgfmathsetmacro{\xc}{\XCol[\n-\step]}
				\pgfmathsetmacro{\sc}{\SCol[\n-\step]}
				\pgfmathsetmacro{\ac}{\ACol[\n-\step]}
				\pgfmathsetmacro{\zc}{\ZCol[\n-\step]}
				\node[state, fill={\ifnum\xc=1 mred!50\else\ifnum\xc=2 mblue!50\else\ifnum\xc=3 mgreen!50\else white\fi\fi\fi}] (Xn-\ms) [left of = Xn-\aftestate] {$\M_{\t-\step}$};
				\node[state, fill={\ifnum\ac=1 mred!50\else\ifnum\ac=2 mblue!50\else\ifnum\ac=3 mgreen!50\else white\fi\fi\fi}] (An-\ms) [left of = An-\aftestate] {$\A_{\t-{\step}}$};
				\node[state, fill={\ifnum\sc=1 mred!50\else\ifnum\sc=2 mblue!50\else\ifnum\sc=3 mgreen!50\else white\fi\fi\fi}] (Bn-\ms) [left of = Bn-\aftestate] {$\S_{\t-{\step}}$};
				\node[state, fill={\ifnum\zc=1 mred!50\else\ifnum\zc=2 mblue!50\else\ifnum\zc=3 mgreen!50\else white\fi\fi\fi}] (Zn-\ms) [left of = Zn-\aftestate] {$\Z_{\t-\step}$};
				\node[circle, draw=black, fill=black!40, inner sep=2pt] (Tn-\ms) [left of = Tn-\aftestate] {};
				\node[circle, draw=black, fill=black!40, inner sep=2pt] (Fn-\ms) [left of = Fn-\aftestate] {};
				
				\path[->, thick] (Tn-\prevstate) edge (Zn-\step);
				\path[->, thick] (Tn-\prevstate) edge (Bn-\prevstate);
				\path[->, thick] (Fn-\step) edge (Xn-\step);
				\path[->, thick] (Fn-\step) edge (An-\step);
				\path[->, thick] (Tn-\ms) edge (Zn-\aftestate);
				\path[->, thick] (Tn-\ms) edge (Bn-\ms);
				\path[->, thick] (Zn-\ms) edge (Tn-\ms);
				\path[->, thick] (Xn-\ms) edge (Fn-\aftestate);
				\path[->, thick] (Bn-\ms) edge (Fn-\aftestate);
				\fi
				
				\path[->, thick] (Fn-\ms) edge (Xn-\ms);
				\path[->, thick] (Fn-\ms) edge (An-\ms);
			}
			
			\foreach \step in {1,...,\n} {
				\pgfmathtruncatemacro{\prevstate}{\step-1}
				\path[->, thick] (Xn-\prevstate) edge (Fn-\step);
				\path[->, thick] (Zn-\prevstate) edge (Tn-\prevstate);
				\path[->, thick] (Bn-\prevstate) edge (Fn-\step);
			}
			
			\path[dotted, very thick, black!50] (Xn-\n) edge ++(0.7,-0.7);
			\path[dotted, very thick, black!50] (Zn-\n) edge ++(1,1);
			\path[dotted, very thick, black!50] (Fn--\n) edge ++(-0.7,-0.7);
			\path[dotted, very thick, black!50] (Fn--\n) edge ++(-0.7,0.7);
			\path[dotted, very thick, black!50] (Zn--\n) edge ++(-1,1);
			
		\end{tikzpicture}
		\caption{Bayesian network for a product environment channel (\Cref{lem:Bayes_net_compatibility_source}) with colorized d-separation (blue d-separates red and green) used in the proof of \Cref{lem:work_capacity_expressions}.}
		\label{fig:DAG1}
	\end{figure}
	By the definition of work capacity, we then have
	\begin{align}
		C^{\mathrm{work}}(\texttt{env}) =  \sup_{\texttt{agt}_{\mathcal{M}}\in\mathbb{A}^{\lrstacksmall \texttt{env}}} W(\texttt{agt}_{\mathcal{M}}, \texttt{env}) &=\sup_{\texttt{agt}_{\mathcal{M}}\in\mathbb{A}^{\lrstacksmall \texttt{env}}}\left(\cesaro{H\left({A_t|M_t}\right)}_t- \cesaro{I\left[S_{{0:t}};S_t|M_t\right]}_t\right)  -  h\left(\bm{S}\right) 
	\end{align}
	where we used that the entropy rate $h\left(\bm{S}\right)$ does not depend on the choice of agent model $\texttt{agt}_{\mathcal{M}}$ since $\texttt{env}$ is a product channel.

	An upper bound for $\sup_{\texttt{agt}_{\mathcal{M}}\in\mathbb{A}^{\lrstacksmall \texttt{env}}}\left(\cesaro{H\left({A_t|M_t}\right)}_t- \cesaro{ I\left[S_{{0:t}};S_t|M_t\right]}_t\right)$ is obtained  by setting the first term to its upper bound ($\log_2|\mathcal{A}|$) and the last term to its upper bound (zero) which yields the desired expression for work capacity. What is left to show is that the supremum indeed takes this value.

	To this end, consider an agent model $\texttt{agt}_{\mathcal{M}}$ which always stores the previous $L$ percepts in its memory: $M_t=S_{t-L:t}$. (For $t<L$ we have $M_t=S_{0:t}$ and $L-t$ ``blank registers'' in the agent's memory.) Then, we find for the last term in \cref{eq:work_capacity_derivation_proof_1} 
	\begin{align}
		\cesaro{I[S_{0:t};S_t|S_{\max(0,t-L):t}]}_t &=  \cesaro{H(S_t|S_{\max(0,t-L):t})-H(S_t|S_{0:t})}_t \label{eq:transfer_predictivity_to_asymptotic_process_1}\\
		&= \cesaro{H(S_t|S_{\max(0,t-L):t})}_t-\cesaro{H(S_t|S_{0:t})}_t\label{eq:transfer_predictivity_to_asymptotic_process_2}\\
		&= \cesaro{H(S_{t+L}|S_{t:t+L})}_t-\cesaro{H(S_t|S_{0:t})}_t\label{eq:transfer_predictivity_to_asymptotic_process_21}\\
		&= \frac{1}{d}\sum_{r=0}^{d-1} H(\overline{S}_{r+L}|\overline{S}_{r:r+L})-h(\bm{S}) \label{eq:transfer_predictivity_to_asymptotic_process_3}
	\end{align}
	where \cref{eq:transfer_predictivity_to_asymptotic_process_1} follows from the definition of conditional mutual information, \cref{eq:transfer_predictivity_to_asymptotic_process_2} follows by the observation that the Ces\`aro limit exists for the individual terms, and the first term in \cref{eq:transfer_predictivity_to_asymptotic_process_21} follows since the Ces\`aro limit is invariant with respect to finite time-index shifts. To obtain the first term of  \cref{eq:transfer_predictivity_to_asymptotic_process_3}, we used \Cref{cor:Markov_chains_asymptotic}(iv) where $\overline{\bm{S}}$ is a cyclostationary process with some finite period $d$ which can be obtained by marginalizing the asymptotic cyclostationary process of the global Markov chain (see \Cref{cor:Markov_chains_asymptotic}(iv)). The second term of  \cref{eq:transfer_predictivity_to_asymptotic_process_3} follows from the chain rule (see \cref{eq:entropy_rate_chain_rule}).
	
	Note that only the first term in \cref{eq:transfer_predictivity_to_asymptotic_process_3} depends on $L$. Taking the limit $L\rightarrow \infty$ in the first term in \cref{eq:transfer_predictivity_to_asymptotic_process_3} yields
	\begin{align}
		\lim\limits_{L\rightarrow \infty} \frac{1}{d}\sum_{r=0}^{d-1} H(\overline{S}_{r+L}|\overline{S}_{r:r+L}) &= \frac{1}{d}\sum_{r=0}^{d-1} \lim\limits_{L\rightarrow \infty}H(\overline{S}_{r+L}|\overline{S}_{r:r+L})\\
		&= h(\bm{S})
	\end{align}
	where the last step follows from \Cref{lem:entropy_rate_cyclostationary}.
	
	We have thus seen that $\cesaro{I[S_{0:t};S_t|S_{\max(0,t-L):t}]}_t$ goes to zero as $L$ goes to infinity. Note that since the agent model is required to store the $L$ previous percepts, as $L$ goes to infinity,  $\abs{\mathcal{M}}$ must go to infinity, too. Further note that $I[S_{0:t};S_t|S_{\max(0,t-L):t}]$ is nonnegative. It follows then by the property of the limit, that for any $\epsilon>0$, there exists an agent model with sufficiently large but finite memory $\abs{\mathcal{M}}\in\mathbb{N}$ such that $\cesaro{I[S_{0:t};S_t|S_{\max(0,t-L):t}]}_t<\epsilon$. In particular $\inf_{\texttt{agt}_{\mathcal{M}}}\cesaro{I[S_{0:t};S_t|S_{\max(0,t-L):t}]}_t = 0$.
	
	Finally, note that we can always randomize the agent model's actions such that $\cesaro{H(A_t|M_t)}_t=\log_2|\mathcal{A}|$. This is possible since for product-channel environments, remembering previous percepts does not contribute to $\cesaro{H(A_t|M_t)}_t$. \hfill $\square$
	
	\subsection{Efficient agent models} \label{supp:6.4}
	\Cref{th:work_efficient_predictive_main} of the main text,
	which we restate here for convenience before providing its proof, is based on \Cref{def:agent_subsets} of the main text.
	
	\begin{theorem}
		\label{th:work_efficient_predictive}
		For any product environment channel  $\normalfont\texttt{env}$ and any agent memory size $d\in\mathbb{N}$,
		\begin{align}
			\normalfont\mathbb{A}^{\lrstacksmall \texttt{env}}_{\mathrm{eff}}  (d)=\normalfont\mathbb{A}^{\lrstacksmall \texttt{env}}_{\mathrm{mea}} (d) \cap \normalfont\mathbb{A}^{\lrstacksmall \texttt{env}}_{\mathrm{pred}} (d).
		\end{align}
	\end{theorem}
	\begin{proof}
		Recall \cref{eq:work_capacity_derivation_proof_1}, the work rate for a product environment channel:
		\begin{align}
			W(\texttt{agt}_{\mathcal{M}}, \texttt{env}) =\cesaro{H\left({A_t|M_t}\right)}_t -  h\left(\bm{S}\right) - \cesaro{ I\left[S_{{0:t}};S_t|M_t\right]}_t. \label{eq:extractable_work_source_environment_copy}
		\end{align}
		An upper bound can be obtained by maximizing the terms appearing with a positive coefficient and minimizing those with negative coefficients in the preceding expression:
		\begin{align}
			W(\texttt{agt}_{\mathcal{M}}, \texttt{env})&\leq W^{(d)}(\texttt{env})\\
			&\leq  \sup_{\texttt{agt}_{\mathcal{M}}\in\mathbb{A}^{\lrstacksmall \texttt{env}}(d)}\cesaro{H\left({A_t|M_t}\right)}_t - h\left(\bm{S}\right) - \inf_{\texttt{agt}_{\mathcal{M}}\in\mathbb{A}^{\lrstacksmall \texttt{env}}(d)}\cesaro{ I\left[S_{{0:t}};S_t|M_t\right]}_t \label{eq:max_splitting}
		\end{align}
		Here we used that, for product channels, the percept process $\bm{S}$ and hence $h(\bm{S})$ do not depend on the choice of agent model. We next show that the upper bound in \cref{eq:max_splitting} equals $W^{(d)}(\texttt{env})$ at the level of suprema, without assuming that either extremum is attained. First note that the predictive gap can be rewritten as
		\begin{align}
			I\left[A_{{0:t+1}}S_{{0:t}};S_t|M_t\right] = I\left[S_{{0:t}};S_t|M_t\right] +  I\left[A_{{0:t+1}};S_t|M_tS_{{0:t}}\right]
		\end{align}
		where, for product environment channels, the second term vanishes by
		d-separation (see \Cref{fig:Bayes_net_for_product_channel_capacity}).
		\begin{figure}[h!]
			\centering
			\begin{tikzpicture}[>=stealth',shorten >=1pt,node distance=2.5cm, scale= 0.83,on grid,auto, state/.style={circle, draw, minimum size=0.8cm, inner sep=1pt},font=\tiny]
				
				\def \n {3}
				
				\def \S {S}
				\def \A {A}
				\def \M {M}
				\def \Z {Z}
				\def \t {t}
				
				\def\XCol{{0,0,0,2,0,0,0}}
				\def\SCol{{2,2,2,3,0,0,0}}
				\def\ACol{{1,1,1,1,0,0,0}}
				\def\ZCol{{0,0,0,0,0,0,0}}
				
				\foreach \step in {0,...,\n} {
					\pgfmathtruncatemacro{\ms}{-\step}
					\pgfmathtruncatemacro{\aftestate}{-\step+1}
					\pgfmathtruncatemacro{\prevstate}{\step-1}
					
					\ifnum \step= 0
					\pgfmathsetmacro{\xc}{\XCol[\n]}
					\pgfmathsetmacro{\sc}{\SCol[\n]}
					\pgfmathsetmacro{\ac}{\ACol[\n]}
					\pgfmathsetmacro{\zc}{\ZCol[\n]}
					
					\node[state, fill={\ifnum\xc=1 mred!50\else\ifnum\xc=2 mblue!50\else\ifnum\xc=3 mgreen!50\else white\fi\fi\fi}] (Xn-\step) at (0.75,0.75) {$\M_\t$};
					\node[state, fill={\ifnum\ac=1 mred!50\else\ifnum\ac=2 mblue!50\else\ifnum\ac=3 mgreen!50\else white\fi\fi\fi}] (An-\step) at (0,-1.5) {$\A_\t$};
					\node[state, fill={\ifnum\sc=1 mred!50\else\ifnum\sc=2 mblue!50\else\ifnum\sc=3 mgreen!50\else white\fi\fi\fi}] (Bn-\step) at (1.5,-1.5) {$\S_\t$};
					\node[state, fill={\ifnum\zc=1 mred!50\else\ifnum\zc=2 mblue!50\else\ifnum\zc=3 mgreen!50\else white\fi\fi\fi}] (Zn-\step) at (-0.75,-3.75) {$\Z_\t$};
					\node[circle, draw=black, fill=black!40, inner sep=2pt] (Tn-0) at (0.75,-2.25) {};
					\node[circle, draw=black, fill=black!40, inner sep=2pt] (Fn-0) at (-0.75,-0.75) {};
					\else
					
					\pgfmathsetmacro{\xc}{\XCol[\n+\step]}
					\pgfmathsetmacro{\sc}{\SCol[\n+\step]}
					\pgfmathsetmacro{\ac}{\ACol[\n+\step]}
					\pgfmathsetmacro{\zc}{\ZCol[\n+\step]}
					
					\node[state, fill={\ifnum\xc=1 mred!50\else\ifnum\xc=2 mblue!50\else\ifnum\xc=3 mgreen!50\else white\fi\fi\fi}] (Xn-\step) [right of = Xn-\prevstate] {$\M_{\t+\step}$};
					\node[state, fill={\ifnum\ac=1 mred!50\else\ifnum\ac=2 mblue!50\else\ifnum\ac=3 mgreen!50\else white\fi\fi\fi}] (An-\step) [right of = An-\prevstate] {$\A_{\t+{\step}}$};
					\node[state, fill={\ifnum\zc=1 mred!50\else\ifnum\zc=2 mblue!50\else\ifnum\zc=3 mgreen!50\else white\fi\fi\fi}] (Zn-\step) [right of = Zn-\prevstate] {$\Z_{\t+\step}$};
					\node[circle, draw=black, fill=black!40, inner sep=2pt] (Fn-\step) [right of = Fn-\prevstate] {};
					
					\ifnum \step< \n
					\node[state, fill={\ifnum\sc=1 mred!50\else\ifnum\sc=2 mblue!50\else\ifnum\sc=3 mgreen!50\else white\fi\fi\fi}] (Bn-\step) [right of = Bn-\prevstate] {$\S_{\t+{\step}}$};
					\node[circle, draw=black, fill=black!40, inner sep=2pt] (Tn-\step) [right of = Tn-\prevstate] {};
					\fi
					
					\pgfmathsetmacro{\xc}{\XCol[\n-\step]}
					\pgfmathsetmacro{\sc}{\SCol[\n-\step]}
					\pgfmathsetmacro{\ac}{\ACol[\n-\step]}
					\pgfmathsetmacro{\zc}{\ZCol[\n-\step]}
					\node[state, fill={\ifnum\xc=1 mred!50\else\ifnum\xc=2 mblue!50\else\ifnum\xc=3 mgreen!50\else white\fi\fi\fi}] (Xn-\ms) [left of = Xn-\aftestate] {$\M_{\t-\step}$};
					\node[state, fill={\ifnum\ac=1 mred!50\else\ifnum\ac=2 mblue!50\else\ifnum\ac=3 mgreen!50\else white\fi\fi\fi}] (An-\ms) [left of = An-\aftestate] {$\A_{\t-{\step}}$};
					\node[state, fill={\ifnum\sc=1 mred!50\else\ifnum\sc=2 mblue!50\else\ifnum\sc=3 mgreen!50\else white\fi\fi\fi}] (Bn-\ms) [left of = Bn-\aftestate] {$\S_{\t-{\step}}$};
					\node[state, fill={\ifnum\zc=1 mred!50\else\ifnum\zc=2 mblue!50\else\ifnum\zc=3 mgreen!50\else white\fi\fi\fi}] (Zn-\ms) [left of = Zn-\aftestate] {$\Z_{\t-\step}$};
					\node[circle, draw=black, fill=black!40, inner sep=2pt] (Tn-\ms) [left of = Tn-\aftestate] {};
					\node[circle, draw=black, fill=black!40, inner sep=2pt] (Fn-\ms) [left of = Fn-\aftestate] {};
					
					\path[->, thick] (Tn-\prevstate) edge (Zn-\step);
					\path[->, thick] (Tn-\prevstate) edge (Bn-\prevstate);
					\path[->, thick] (Fn-\step) edge (Xn-\step);
					\path[->, thick] (Fn-\step) edge (An-\step);
					\path[->, thick] (Tn-\ms) edge (Zn-\aftestate);
					\path[->, thick] (Tn-\ms) edge (Bn-\ms);
					\path[->, thick] (Zn-\ms) edge (Tn-\ms);
					\path[->, thick] (Xn-\ms) edge (Fn-\aftestate);
					\path[->, thick] (Bn-\ms) edge (Fn-\aftestate);
					\fi
					
					\path[->, thick] (Fn-\ms) edge (Xn-\ms);
					\path[->, thick] (Fn-\ms) edge (An-\ms);
				}
				
				\foreach \step in {1,...,\n} {
					\pgfmathtruncatemacro{\prevstate}{\step-1}
					\path[->, thick] (Xn-\prevstate) edge (Fn-\step);
					\path[->, thick] (Zn-\prevstate) edge (Tn-\prevstate);
					\path[->, thick] (Bn-\prevstate) edge (Fn-\step);
				}
				
				\path[dotted, very thick, black!50] (Xn-\n) edge ++(0.7,-0.7);
				\path[dotted, very thick, black!50] (Zn-\n) edge ++(1,1);
				\path[dotted, very thick, black!50] (Fn--\n) edge ++(-0.7,-0.7);
				\path[dotted, very thick, black!50] (Fn--\n) edge ++(-0.7,0.7);
				\path[dotted, very thick, black!50] (Zn--\n) edge ++(-1,1);
				
			\end{tikzpicture}
			\caption{Bayesian network for a product channel with colorized d-separation (blue d-separates red and green) used in the proof of \Cref{th:work_efficient_predictive}.}\label{fig:Bayes_net_for_product_channel_capacity}
		\end{figure}
		For any $\texttt{agt}_{\mathcal M}\in\mathbb A^{\lrstacksmall\texttt{env}}(d)$, define a companion agent model $\texttt{agt}^{\mathrm{rnd}}_{\mathcal M}$ by
		\begin{align}
			p^{\mathrm{rnd}}_{M_0A_0}(m,a) &\coloneqq \frac{p^{\mathrm{agt}}_{M_0}(m)}{\abs{\mathcal A}},\\
			\theta^{\mathrm{rnd}}(a',m'|s,m) &\coloneqq \frac{1}{\abs{\mathcal A}} \sum_{\tilde a\in\mathcal A} \theta^{\mathrm{agt}}(\tilde a,m'|s,m).
		\end{align}
		
		Summing $\theta^{\mathrm{rnd}}$ over $a'$ gives
		\begin{align}
			\sum_{a'\in\mathcal A}	\theta^{\mathrm{rnd}}(a',m'\mid s,m)
			=\sum_{\tilde a\in\mathcal A} \theta^{\mathrm{agt}}(\tilde a,m'\mid s,m),
		\end{align}
		so the construction preserves both $p_{M_0}$ and the marginal
		memory update. Since actions do not affect the percept process in a product environment, it follows that the joint process	$(\bm M,\bm S)$, and hence the predictive gap by the d-separation
		relation above, is unchanged. Therefore, every achievable predictive gap can be paired with $\cesaro{H(A_t|M_t)}_t=\log_2\abs{\mathcal A}$ by considering the the companion agent model. Applying this construction to a sequence of agents whose predictive gaps approach the infimum yields work rates approaching the right-hand side of \cref{eq:max_splitting}. Thus, this upper bound equals $W^{(d)}(\texttt{env})$.
		
		Therefore, an agent attains $W^{(d)}(\texttt{env})$ if and only if it attains both the action-entropy supremum and the predictive-gap infimum. By \Cref{def:agent_subsets} of the main text, this is precisely the claimed set identity. 
	\end{proof}
	For general environment channels, we first prove the following lemma which shows  properties for a particular memoryless environment channel. 
	\begin{lemma} \label{lem:env_example}
		Let the environment {\normalfont\texttt{env}} be a memoryless environment channel such that $A_t$ and $S_t$ take values in the alphabet $\normalfont\mathcal{A}=\mathcal{S}=\{\texttt{0},\texttt{1}\}$. Let the environment's transition matrix $\normalfont\Phi^{\mathrm{env}}=(\phi^{\mathrm{env}}(j|i))_{j,i}$ with $j,i\in\mathcal{A}$ be such that $\normalfont\phi^{\mathrm{env}}(j|\texttt{0})=\delta_{\texttt{0},j}$ and $\phi^{\mathrm{env}}(j|\texttt{1})=1/2$ for $j=\texttt{0},\texttt{1}$. Then, for any percept-action loop $\normalfont(\texttt{agt}_{\mathcal{M}},\texttt{env})$ with $\normalfont\texttt{agt}_{\mathcal{M}}\in\mathbb{A}^{\lrstacksmall \texttt{env}}$, we have
		\begin{align}
			\cesaro{I[A_t; S_t|M_t]}_t = 0 \Leftrightarrow \cesaro{H(A_t|M_t)}_t=0. \label{eq:lemma_env_to_show_1}
		\end{align}
	\end{lemma}
	\emph{Proof.}

	First note that if an agent model $\texttt{agt}_{\mathcal{M}}$ is such that $\cesaro{H\left(A_t|M_t\right)}_t=0$, then
	\begin{align}
		\cesaro{I\left[A_t; S_t|M_t\right]}_t = \cesaro{H\left(A_t|M_t\right)}_t- \cesaro{H\left(A_t|M_tS_t\right)}_t=0,
	\end{align}
	where we used the definition of mutual information (\cref{eq:def_cond_mut_info}), and  the conclusion follows from the nonnegativity of conditional mutual information and conditional entropy, proving one direction of \cref{eq:lemma_env_to_show_1}.

	For the other direction, for the environment $\texttt{env}$ under consideration, by \Cref{def:hidden_Markov_channel} there exists a Markov model $\texttt{env}_{\mathcal{Z}}$ on some finite state space $\mathcal{Z}$ and thus, by \Cref{lem:global_Markov_chain}, there also exists a global Markov chain $\bm{U}= (U_t)_{t=0}^\infty = (M_t,A_t,S_t,Z_t)_{t=0}^\infty$. Let $\overline{\bm{U}}= (\overline{U}_t)_{t=0}^\infty = (\overline{M}_t,\overline{A}_t,\overline{S}_t,\overline{Z}_t)_{t=0}^\infty$ be its asymptotic process as defined in \Cref{cor:Markov_chains_asymptotic}(iv), and let $\ell$ denote the finite period of this cyclostationary process.
	By \Cref{cor:Markov_chains_asymptotic}(iv) we then have
	\begin{align}
		\cesaro{I\left[A_t; S_t|M_t\right]}_t  = \frac{1}{\ell}\sum_{r=1}^{\ell}I\left[\overline{A}_r
		; \overline{S}_r|\overline{M}_r\right], \label{eq:pred_asymptotic_period_representation}
	\end{align}
	and similarly
	\begin{align}
		\cesaro{H\left(A_t|M_t\right)}_t  = \frac{1}{\ell}\sum_{r=1}^{\ell}H\left(\overline{A}_r
		|\overline{M}_r\right). \label{eq:pred_asymptotic_period_representation_2}
	\end{align}

	Using the definition of mutual information, we find for each summand in \cref{eq:pred_asymptotic_period_representation}
	\begin{align}
		I\left[\overline{A}_r; \overline{S}_r|\overline{M}_r\right]  = H\left(\overline{A}_r|\overline{M}_r\right)- H\left(\overline{A}_r|\overline{M}_r\overline{S}_r\right).
	\end{align}
	We now want to show that, for any $r\in\{1,2,\dotsc, \ell\}$, $I\left[\overline{A}_r; \overline{S}_r|\overline{M}_r\right]=0$ implies $H\left(\overline{A}_r|\overline{M}_r\right)=0$.
	
	The proof proceeds by contradiction. Assume that $I\left[\overline{A}_r; \overline{S}_r|\overline{M}_r\right]=0$ but $H\left(\overline{A}_r|\overline{M}_r\right)>0$.
	
	First, using basic properties of conditional entropies, we have $H\left(\overline{A}_r|\overline{M}_r\right) = \sum_{m\in\mathcal{M}}p_{\overline{M}_r} (m)H\left(\overline{A}_r|\overline{M}_r=m\right)$ where $H\left(\overline{A}_r|\overline{M}_r=m\right)=0 $ iff $p_{\overline{A}_r|\overline{M}_r=m}$ is a delta distribution.
	
	Then, due to $H\left(\overline{A}_r|\overline{M}_r\right)>0$, there exists a memory state $m'_r\in\mathcal{M}$ with $p_{\overline{M}_r}(m'_r)>0$ such that $p_{\overline{A}_r|\overline{M}_r}(\texttt{0}|m'_r)>0$ and $p_{\overline{A}_r|\overline{M}_r}(\texttt{1}|m'_r)>0$. We have
	\begin{align}
		I\left[\overline{A}_r;\overline{S}_r|\overline{M}_r\right]= \sum_{m_r\in\mathcal{M}}p_{\overline{M}_r}(m_r)I\left[\overline{A}_r;\overline{S}_r|\overline{M}_r=m_r\right] \label{eq:proof_efficent_but_not_predictive_2}
	\end{align}
	where $I\left[\overline{A}_r;\overline{S}_r|\overline{M}_r=m_r\right]$ is the mutual information $I\left[\overline{A}_r;\overline{S}_r\right]$ with $\overline{A}_r$, $\overline{S}_r$ distributed as $p_{\overline{A}_r \overline{S}_r|\overline{M}_r=m_r}$. The expansion in \cref{eq:proof_efficent_but_not_predictive_2} can be obtained by writing out mutual information, \cref{eq:def_cond_mut_info}, in terms of probabilities.
	
	Now, by the nonnegativity of mutual information, for the left-hand side of \cref{eq:proof_efficent_but_not_predictive_2} to vanish, each summand on the right-hand side of \cref{eq:proof_efficent_but_not_predictive_2} must vanish individually. In particular, for the summand corresponding to $\overline{M}_r=m'_r$ to vanish, $I\left[\overline{A}_r;\overline{S}_r|\overline{M}_r=m'_r\right]$ must be zero. Further, using basic properties of mutual information, $I\left[\overline{A}_r;\overline{S}_r|\overline{M}_r=m'_r\right]=0$ iff $p_{\overline{A}_r \overline{S}_r|\overline{M}_r=m'_r}$ is a product distribution. However, note that for percept-action loops with memoryless environment channel we have
	\begin{align}
		p_{\overline{A}_r \overline{S}_r|\overline{M}_r=m'_r} = p_{\overline{S}_r|\overline{A}_r}p_{\overline{A}_r|\overline{M}_r=m'_r}
	\end{align}
	where $p_{\overline{S}_r|\overline{A}_r}(s|a) = \phi^{\mathrm{env}}(s|a)$ is given by the memoryless environment which is chosen such that $\phi^{\mathrm{env}}(s|\texttt{0})\neq \phi^{\mathrm{env}}(s|\texttt{1})$ for all $s\in\mathcal{S}$ and, thus, $p_{\overline{A}_r \overline{S}_r|\overline{M}_r=m'_r}$ is not a product distribution. By this contradiction,
	we have shown, for any $r\in\{1,2,\dotsc, \ell\}$, that $I\left[\overline{A}_r; \overline{S}_r|\overline{M}_r\right]=0$ implies $H\left(\overline{A}_r|\overline{M}_r\right)=0$. By \cref{eq:pred_asymptotic_period_representation,eq:pred_asymptotic_period_representation_2} it then follows that $\cesaro{I\left[A_t; S_t|M_t\right]}_t=0$ implies $\cesaro{H\left(A_t|M_t\right)}_t=0$. \hfill $\square$\\
	
	We restate \Cref{th:efficient_but_not_predictive_main} of the main text
	here for convenience before providing its proof.
	\begin{theorem} \label{th:efficient_but_not_predictive}
		There exists an environment channel $\normalfont\texttt{env}$ such that for any $d\in\mathbb{N}$ the sets $\normalfont\mathbb{A}^{\lrstacksmall \texttt{env}}_{\mathrm{pred}}(d)$, $\normalfont\mathbb{A}^{\lrstacksmall \texttt{env}}_{\mathrm{mea}}(d)$, and $\normalfont\mathbb{A}^{\lrstacksmall \texttt{env}}_{\mathrm{eff}}(d)$ are mutually exclusive.
	\end{theorem}
	
	\begin{proof}
		\begin{figure}[h!]
			\centering
			\includegraphics[width=0.25\linewidth]{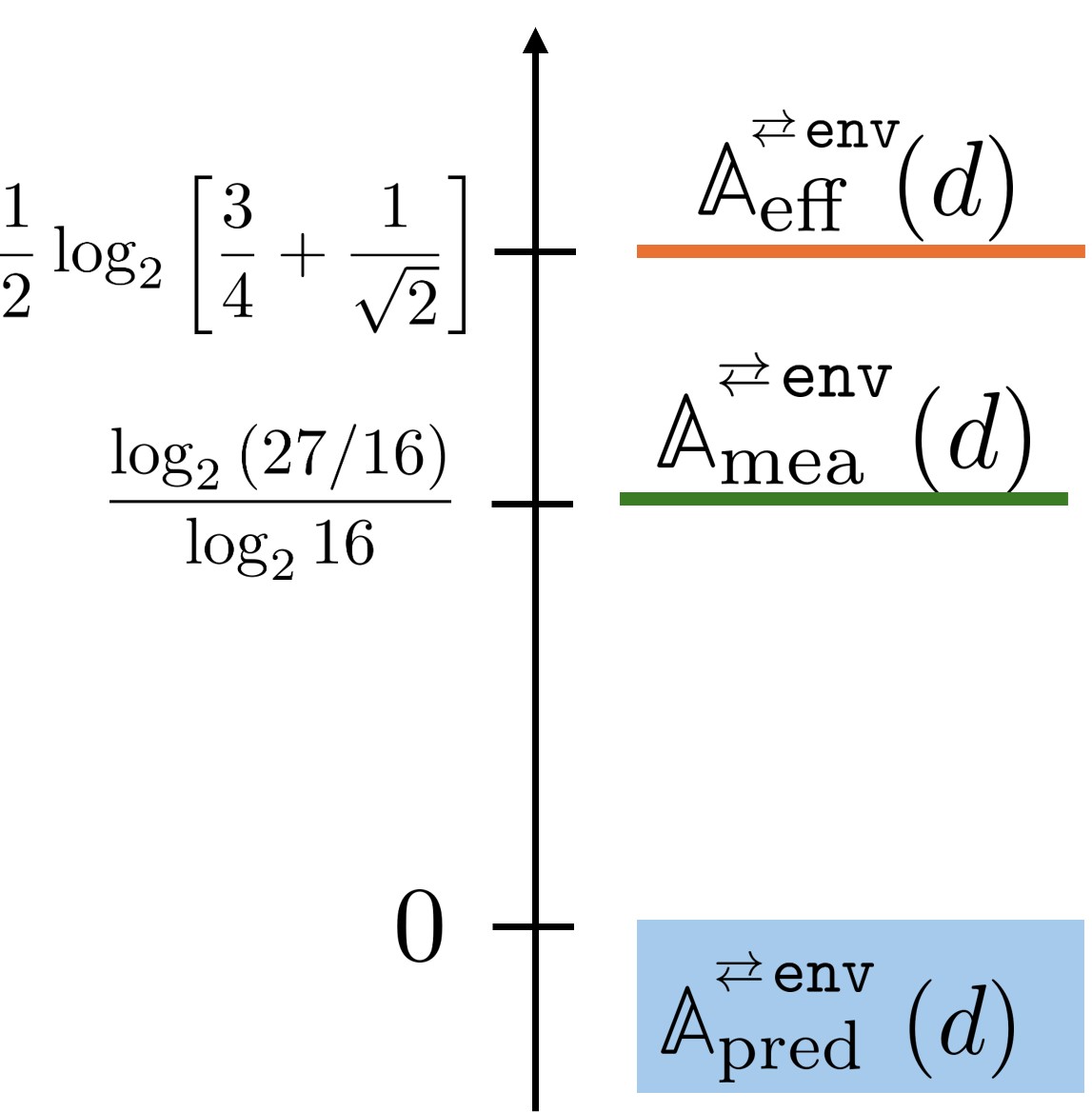}
			\caption{Schematic illustration of the idea underlying the proof of \Cref{th:efficient_but_not_predictive}. The proof consists in showing that the work rates (as shown on the axis) differ between agent models in different sets and thus the sets of agent models must be mutually exclusive.}
			\label{fig:supp_1}
		\end{figure}
		The idea of the proof is sketched in \Cref{fig:supp_1}.
		
		Let the environment {\normalfont\texttt{env}} be defined as in \Cref{lem:env_example}.
		Given actions distributed as $p_{A_t}$, the definition of \texttt{env} gives rise to a joint distribution
		\begin{align}
			\begin{pmatrix}
				p_{A_tS_t}(\texttt{0},\texttt{0})\\
				p_{A_tS_t}(\texttt{0},\texttt{1})\\
				p_{A_tS_t}(\texttt{1},\texttt{0})\\
				p_{A_tS_t}(\texttt{1},\texttt{1})
			\end{pmatrix} = \begin{pmatrix}
				p_{A_t}(\texttt{0})\\
				0\\
				\frac{p_{A_t}(\texttt{1})}{2}\\
				\frac{p_{A_t}(\texttt{1})}{2}
			\end{pmatrix} \label{eq:dist_for_proof}
		\end{align}
		which will be used in the following calculations.
		
		First, we will derive a lower bound on the work rate of any agent model in $\bigcup_{d}\mathbb{A}^{\lrstacksmall \texttt{env}}_{\mathrm{eff}}(d)$ where $\bigcup_{d}$ denotes the union of all sets with $d\in\mathbb{N}$.
		To this end, consider maximizing the work extractable from \texttt{env} in round $t$  with respect to agent models from $\mathbb{A}^{\lrstacksmall \texttt{env}}(d=1)$, i.e., memoryless agent models:
		\begin{align}
			\max_{p_{A_t}}\big[ H(A_t) - H(S_t) \big]. \label{eq:optimize_work_in_proof}
		\end{align}
		As a result of a straightforward calculation, the solution to \cref{eq:optimize_work_in_proof} is found to be
		\begin{align}
			\frac{1}{2}\log_2\left[\frac{3}{4}+\frac{1}{\sqrt{2}}\right] \label{eq:work_rate_value}
		\end{align}
		which is reached by memoryless agent models which distribute actions such that $p_{A_t}(\texttt{0})=1/\sqrt{2}$. Thus, if a memoryless agent model employs this action distribution \emph{for all} $t\in\mathbb{N}_0$, it achieves a work rate which equals this value. Since, by trivially extending the memory, any memoryless agent model is also contained in  $\mathbb{A}^{\lrstacksmall \texttt{env}}(d)$ for any $d>1$, we have shown that \cref{eq:work_rate_value} is a lower bound for any agent model in $\bigcup_d\mathbb{A}^{\lrstacksmall \texttt{env}}_{\mathrm{eff}}(d)$. In fact, computing the work capacity of \texttt{env} using \Cref{lem:work_capacity_expressions} yields the same value which shows that the work rate of all agent models in $\bigcup_d\mathbb{A}^{\lrstacksmall \texttt{env}}_{\mathrm{eff}}(d)$ \emph{equals} \cref{eq:work_rate_value}.
		
		Next we will prove an upper bound on the work rate of any agent model in $\bigcup_d\mathbb{A}^{\lrstacksmall \texttt{env}}_{\mathrm{pred}}(d)$. To this end, we first show that any agent model in  $\bigcup_d\mathbb{A}^{\lrstacksmall \texttt{env}}_{\mathrm{pred}}(d)$ must achieve vanishing predictive gap, i.e., $\cesaro{I\left[A_{{0:t+1}}S_{{0:t}};S_t|M_t\right]}_t = 0$. To this end, 
		consider memoryless agent models with $A_t=\texttt{0}$ for all $t\in\mathbb{N}_0$. Then, $S_t=\texttt{0}$ for all $t\in\mathbb{N}_0$ and thus the predictive gap  vanishes, i.e., $\cesaro{I\left[A_{0:t+1}S_{0:t};S_t|M_t\right]}_t=0$.  Since, by trivially extending the memory, any memoryless agent model is also contained in $\mathbb{A}^{\lrstacksmall \texttt{env}}(d)$ for any $d>1$, any agent model in  $\mathbb{A}^{\lrstacksmall \texttt{env}}_{\mathrm{pred}}(d)$ for any $d\in\mathbb{N}$ must have a vanishing predictive gap.
		
		With this result, we obtain an upper bound on the work rate as follows.
		As we have seen, any agent model in $\bigcup_d\mathbb{A}^{\lrstacksmall \texttt{env}}_{\mathrm{pred}}(d)$ fulfills
		\begin{align}
			0&=\cesaro{I\left[A_{{0:t+1}}S_{{0:t}};S_t|M_t\right]}_t\\
			&=\cesaro{I\left[A_t;S_t|M_t\right]}_t+\cesaro{I\left[A_{{0:t}}S_{{0:t}};S_t|M_tA_t\right]}_t \label{eq:proof_efficient_but_not_predictive_1}
		\end{align}
		where the second line follows from the chain rule of mutual information (\cref{eq:chain_rule_mutual_information_one_step}). The second term vanishes by the d-separation relation $I\left[A_{{0:t}}S_{{0:t}};S_t|M_tA_t\right]=0$ shown in \Cref{fig:Bayesian_net_memoryless_env_proof_2}.
		
		\begin{figure}[h!]
			\centering
			\begin{tikzpicture}[>=stealth',shorten >=1pt,node distance=2.5cm, scale= 0.83,on grid,auto, state/.style={circle, draw, minimum size=0.8cm, inner sep=1pt},font=\tiny]
				
				\def \n {3}
				
				\def \S {S}
				\def \A {A}
				\def \M {M}
				\def \t {t}
				
				\def\XCol{{0,0,0,2,0,0,0}}
				\def\SCol{{1,1,1,3,0,0,0}}
				\def\ACol{{1,1,1,2,0,0,0}}
				
				\foreach \step in {0,...,\n} {
					\pgfmathtruncatemacro{\ms}{-\step}
					\pgfmathtruncatemacro{\aftestate}{-\step+1}
					\pgfmathtruncatemacro{\prevstate}{\step-1}
					
					\ifnum \step= 0
					\pgfmathsetmacro{\xc}{\XCol[\n]}
					\pgfmathsetmacro{\sc}{\SCol[\n]}
					\pgfmathsetmacro{\ac}{\ACol[\n]}
					
					\node[state, fill={\ifnum\xc=1 mred!50\else\ifnum\xc=2 mblue!50\else\ifnum\xc=3 mgreen!50\else white\fi\fi\fi}] (Xn-\step) at (0.75,0.75) {$\M_\t$};
					\node[state, fill={\ifnum\ac=1 mred!50\else\ifnum\ac=2 mblue!50\else\ifnum\ac=3 mgreen!50\else white\fi\fi\fi}] (An-\step) at (0,-1.5) {$\A_\t$};
					\node[state, fill={\ifnum\sc=1 mred!50\else\ifnum\sc=2 mblue!50\else\ifnum\sc=3 mgreen!50\else white\fi\fi\fi}] (Bn-\step) at (1.5,-1.5) {$\S_\t$};
					\node[circle, draw=black, fill=black!40, inner sep=2pt] (Fn-0) at (-0.75,-0.75) {};
					\else
					
					\pgfmathsetmacro{\xc}{\XCol[\n+\step]}
					\pgfmathsetmacro{\sc}{\SCol[\n+\step]}
					\pgfmathsetmacro{\ac}{\ACol[\n+\step]}
					
					\node[state, fill={\ifnum\xc=1 mred!50\else\ifnum\xc=2 mblue!50\else\ifnum\xc=3 mgreen!50\else white\fi\fi\fi}] (Xn-\step) [right of = Xn-\prevstate] {$\M_{\t+\step}$};
					\node[state, fill={\ifnum\ac=1 mred!50\else\ifnum\ac=2 mblue!50\else\ifnum\ac=3 mgreen!50\else white\fi\fi\fi}] (An-\step) [right of = An-\prevstate] {$\A_{\t+{\step}}$};
					\node[circle, draw=black, fill=black!40, inner sep=2pt] (Fn-\step) [right of = Fn-\prevstate] {};
					
					\ifnum \step< \n
					\node[state, fill={\ifnum\sc=1 mred!50\else\ifnum\sc=2 mblue!50\else\ifnum\sc=3 mgreen!50\else white\fi\fi\fi}] (Bn-\step) [right of = Bn-\prevstate] {$\S_{\t+{\step}}$};
					\fi
					
					\pgfmathsetmacro{\xc}{\XCol[\n-\step]}
					\pgfmathsetmacro{\sc}{\SCol[\n-\step]}
					\pgfmathsetmacro{\ac}{\ACol[\n-\step]}
					\node[state, fill={\ifnum\xc=1 mred!50\else\ifnum\xc=2 mblue!50\else\ifnum\xc=3 mgreen!50\else white\fi\fi\fi}] (Xn-\ms) [left of = Xn-\aftestate] {$\M_{\t-\step}$};
					\node[state, fill={\ifnum\ac=1 mred!50\else\ifnum\ac=2 mblue!50\else\ifnum\ac=3 mgreen!50\else white\fi\fi\fi}] (An-\ms) [left of = An-\aftestate] {$\A_{\t-{\step}}$};
					\node[state, fill={\ifnum\sc=1 mred!50\else\ifnum\sc=2 mblue!50\else\ifnum\sc=3 mgreen!50\else white\fi\fi\fi}] (Bn-\ms) [left of = Bn-\aftestate] {$\S_{\t-{\step}}$};
					\node[circle, draw=black, fill=black!40, inner sep=2pt] (Fn-\ms) [left of = Fn-\aftestate] {};
					
					\path[->, thick] (Fn-\step) edge (Xn-\step);
					\path[->, thick] (Fn-\step) edge (An-\step);
					\path[->, thick] (An-\ms) edge (Bn-\ms);
					\path[->, thick] (Xn-\ms) edge (Fn-\aftestate);
					\path[->, thick] (Bn-\ms) edge (Fn-\aftestate);
					\fi
					
					\path[->, thick] (Fn-\ms) edge (Xn-\ms);
					\path[->, thick] (Fn-\ms) edge (An-\ms);
				}
				
				\foreach \step in {1,...,\n} {
					\pgfmathtruncatemacro{\prevstate}{\step-1}
					\path[->, thick] (Xn-\prevstate) edge (Fn-\step);
					\path[->, thick] (Bn-\prevstate) edge (Fn-\step);
					\path[->, thick] (An-\prevstate) edge (Bn-\prevstate);
				}
				
				\path[dotted, very thick, black!50] (Xn-\n) edge ++(0.7,-0.7);
				\path[dotted, very thick, black!50] (An-\n) edge ++(1.4,0);
				\path[dotted, very thick, black!50] (Fn--\n) edge ++(-0.7,-0.7);
				\path[dotted, very thick, black!50] (Fn--\n) edge ++(-0.7,0.7);
				
			\end{tikzpicture}
			\caption{Bayesian network for a memoryless environment channel (\Cref{cor:Bayes_net_compatibility_memoryless}) with colorized d-separation (blue d-separates red and green) used in the proof of \Cref{th:efficient_but_not_predictive}.}\label{fig:Bayesian_net_memoryless_env_proof_2}
		\end{figure}
		Further,
		\begin{align}
			\cesaro{I\left[A_t; S_t|M_t\right]}_t = 0 \Rightarrow \cesaro{H\left(A_t|M_t\right)}_t=0 \label{eq:lemma_1}
		\end{align}
		by \Cref{lem:env_example}.
		We have just seen that any agent in $\bigcup_d\mathbb{A}^{\lrstacksmall \texttt{env}}_{\mathrm{pred}}(d)$ fulfills the left-hand side of \cref{eq:lemma_1}. Plugging the right-hand side into the expression for work rate, \cref{eq:extractable_work}, shows that work rate is upper bounded by zero.
		
		Finally, we will determine the work rate of agent models in $\bigcup_d\mathbb{A}^{\lrstacksmall \texttt{env}}_{\mathrm{mea}}(d)$. To this end, note that a memoryless agent model which randomizes actions such that $H(A_t)=1$ for all $t\in\mathbb{N}_0$ achieves the maximum $\cesaro{H(A_t|M_t)}_t=1$ and thus must lie in $\mathbb{A}^{\lrstacksmall \texttt{env}}_{\mathrm{mea}}(d=1)$. By trivially extending memory, this agent model is also contained in $\mathbb{A}^{\lrstacksmall \texttt{env}}_{\mathrm{mea}}(d)$ for all $d>1$, and thus all agents in $\bigcup_d\mathbb{A}^{\lrstacksmall \texttt{env}}_{\mathrm{mea}}(d)$ must fulfill $\cesaro{H(A_t|M_t)}_t=1$. Due to the canonical bounds $H(A_t|M_t)\leq H(A_t)\leq 1$, this implies $\cesaro{H(A_t)}_t=1$. We thus have
		\begin{align}
			0&=\cesaro{H(A_t)}_t-\cesaro{H(A_t|M_t)}_t\\
			&=\cesaro{I[A_t;M_t]}_t.
		\end{align}
		Further using two equivalent decompositions of interaction information,
		\begin{align}
			I[A_t;S_t;M_t] &=I[A_t;M_t] -I[A_t;M_t|S_t]  \\
			&=I[S_t;M_t] -I[S_t;M_t|A_t] 
		\end{align}
		and that $I[S_t;M_t|A_t] =0$ (this follows from the d-separation relations; see \Cref{fig:Bayesian_net_memoryless_env_proof}), we arrive at $I[A_t;M_t]\geq I[S_t;M_t]$ (which is nothing else than a data processing inequality). Thus, $0
		\leq \cesaro{I[S_t;M_t]}_t \leq \cesaro{I[A_t;M_t]}_t=0$, which implies
		\begin{align}
			\cesaro{H(S_t)}_t = \cesaro{H(S_t|M_t)}_t.
		\end{align}
		The work rate $W_{\mathrm{mea}}$ of any agent model in $\bigcup_d\mathbb{A}^{\lrstacksmall \texttt{env}}_{\mathrm{mea}}(d)$ is thus given by
		\begin{align}
			W_{\mathrm{mea}}=\cesaro{H(A_t|M_t)}_t-\cesaro{H(S_t|M_t)}_t=1-\cesaro{H(S_t)}_t. \label{eq:lower_bound_mea}
		\end{align}
		In order to understand how to compute $\cesaro{H(S_t)}_t$ first note that if $H(A_t)=1$, then a simple calculation using \cref{eq:dist_for_proof} gives $H(S_t)=\log_2\left[256/27\right]/\log_2(16)$. Using \Cref{cor:Markov_chains_asymptotic}(iv), the Ces\`aro limits $\cesaro{H(A_t)}_t$ and $\cesaro{H(S_t)}_t$ become finite sums with respect to the asymptotic process such that likewise $\cesaro{H(A_t)}_t=1$ implies $\cesaro{H(S_t)}_t=\log_2\left[256/27\right]/\log_2(16)$.
		
		Since $\cesaro{H(A_t)}_t=1$ is fulfilled for all agent models in $\bigcup_d\mathbb{A}^{\lrstacksmall \texttt{env}}_{\mathrm{mea}}(d)$, we obtain
		\begin{align}
			W_{\mathrm{mea}}=1-\frac{\log_2\left[256/27\right]}{\log_2(16)} = \frac{\log_2\left(27/16\right)}{\log_2 16}.
		\end{align}
		We conclude that  $\bigcup_d\mathbb{A}^{\lrstacksmall \texttt{env}}_{\mathrm{eff}}(d)$, $\bigcup_d\mathbb{A}^{\lrstacksmall \texttt{env}}_{\mathrm{pred}}(d)$, and $\bigcup_d\mathbb{A}^{\lrstacksmall \texttt{env}}_{\mathrm{mea}}(d)$ are mutually exclusive since the work rates achieved by agent models in these sets are incompatible, and the statement of the theorem follows from this.
	\end{proof}
	
	\section{Details on the Example} \label{supp:7}
	
	This appendix provides supplementary details for the results presented in \Cref{sec:6} of the main text. The environment channel  $\texttt{env}^{(r)}$ considered in \Cref{sec:6} of the main text builds on the golden-mean process, a binary stochastic process whose information-thermodynamic properties were previously analyzed in \cite{boyd2018thermodynamics}. We extend this process by incorporating feedback: while the environment continues to generate binary percepts following the golden-mean structure, where every \texttt{1} is followed by a \texttt{0}, the transition probabilities now depend on the agent’s previous action. The finite-state machine of an environment model giving rise to the environment channel $\texttt{env}^{(r)}$ is shown in \Cref{fig:predictive_agent}a and in \Cref{fig:8}b of the main text.
	
	\subsection{Bayesian network} \label{supp:7.BN}
	Let $\texttt{env}^{(r)}_{\mathcal{Z}} = \left(\Phi^{\mathrm{env}}, p^{\mathrm{env}}_{Z_0}\right)$ be the environment model given by the finite-state machine in \Cref{fig:predictive_agent}a.
	It is a feature of $\texttt{env}^{(r)}_{\mathcal{Z}}$ that it perfectly correlates its memory states with the emitted percept symbol. Notably, the coefficients of its transition matrix admit a decomposition of the form
	\begin{align}
		\phi^{\mathrm{env}}(s_t, z_{t+1}| a_t, z_t) = \delta_{z_{t+1}, s_t} \,\tilde{\phi}^{\mathrm{env}}(s_t| a_t, z_t),   \label{eq:TMPhi}
	\end{align}
	where
	\begin{align}
		\tilde{\phi}^{\mathrm{env}}(s_t| a_t, z_t) = \sum_{z'_{t+1}\in \mathcal{Z}} \phi^{\mathrm{env}}(s_t, z'_{t+1}| a_t, z_t).
	\end{align}
	This factorization can be directly verified from, e.g., the explicit matrix representation of $\Phi^{\mathrm{env}}$. Plugging this into the global distribution of a percept-action loop (\cref{eq:global_distribution} of the main text) yields
	\begin{align}
		p_{\bm{MASZ}}  (\bm{m}, \bm{a}, \bm{s}, \bm{z}) =  p^{\mathrm{agt}}_{M_0 A_0}(m_0, a_0) p^{\mathrm{env}}_{Z_{0}}(z_{0})  \prod_{t=0}^{\infty}\,\delta_{z_{t+1}, s_t}\tilde{\phi}^{\mathrm{env}}(s_t| a_t, z_t)
		\theta^{\mathrm{agt}}
		\left(a_{t+1},m_{t+1}|s_t,m_t\right). \label{eq:global_distribution_case_study}
	\end{align}

	\begin{figure}[h!]
		\begin{tikzpicture}[>=stealth',shorten >=1pt,node distance=2cm,on grid,auto,state/.style={circle, draw, minimum size=.6cm, inner sep=1pt},font=\tiny]
			\def \S {S}
			\def \A {A}
			\def \M {M}
			\def \Z {Z}
			\def \V {V}
			\def \W {W}
			
			\node at (-0.4,2) {};

			\node at (-0.4,0) {};
			\def \n {4}
			\foreach \s in {0,...,\n}
			{
				\pgfmathtruncatemacro{\x}{3*\s}
				\ifnum \s = 0
				\node[state] (Xn-\s) at (0.75,0.75) {$\M_{0}$};
				\node[state] (An-\s) at (0,-1.5) {$\A_{0}$};
				\node[state] (Bn-\s) at (1.5,-1.5) {$\S_{0}$};
				\node[state] (Zn-\s) at (0.0,-3.0) {$\Z_{0}$};
				\node[state,black!50,minimum size=.5cm] (Fn-0) at (-0.75,-0.75) {$V_{0}$};
				
				\else
				\pgfmathtruncatemacro{\prevstate}{\s-1}
				\pgfmathtruncatemacro{\ms}{\s+1}
				\pgfmathtruncatemacro{\aftestate}{-\s+1}
				\node[state] (Xn-\s) at (3*\s+0.75,0.75) {$\M_{\s}$};
				\node[state] (An-\s) [right=3cm of An-\prevstate] {$\A_{{\s}}$};
				\node[state] (Bn-\s) [right=3cm of Bn-\prevstate] {$\S_{{\s}}$};
				\node[state] (Zn-\s) [right=3cm of Zn-\prevstate] {$\Z_{\s}$};
				\node[state,black!50,minimum size=.5cm] (Fn-\s) at (-0.75+3*\s,-0.75) {$\V_{\s}$};
				\path[->] (Xn-\prevstate) edge node {} (Fn-\s);
				\path[->] (Bn-\prevstate) edge node {} (Fn-\s);
				\path[->] (Fn-\prevstate) edge node {} (Xn-\prevstate);
				\path[->] (Fn-\prevstate) edge node {} (An-\prevstate);
				\path[->] (An-\s) edge node {} (Bn-\s);
				\path[->] (Bn-\prevstate) edge node {} (Zn-\s);
				\path[->] (Zn-\prevstate) edge node {} (Bn-\prevstate);
				\fi
			}
			\path[->] (Fn-\n) edge node {} (Xn-\n);
			\path[->] (Fn-\n) edge node {} (An-\n);
			\path[->] (Zn-\n) edge node {} (Bn-\n);
			\foreach \s in {1,...,\n}
			{
				\pgfmathtruncatemacro{\prevstate}{\s-1}
				\pgfmathsetmacro{\ms}{int(\s-1)}
			}
			
			\path[->] (An-0) edge node {} (Bn-0);
			\path[dotted, very thick, black!50] (Xn-\n) edge node {} (3*\n+1.25,0.25);
			\path[dotted, very thick, black!50] (Bn-\n) edge node {} (3*\n+2.0,-1.0);
		\end{tikzpicture}
		\caption{Bayesian network of a percept–action loop with an environment obeying \cref{eq:TMPhi}.}
		\label{fig:Bayes_net_for_Toy}
	\end{figure}
	\begin{lemma} \label{eq:compatibitity_case_study}
		The distribution in \cref{eq:global_distribution_case_study} is compatible with the Bayesian network in \Cref{fig:Bayes_net_for_Toy}.
	\end{lemma}
	\begin{proof}
		The proof proceeds in analogy to \Cref{lem:Bayes_net_compatibility}.
		
		By \Cref{lem:equivalent_markov_coniditions}, a distribution over the variables in the Bayesian network shown in \Cref{fig:Bayes_net_for_Toy} fulfills the global Markov conditions if and only if it factorizes as
		\begin{align}
			p_{\bm{M}\bm{A}\bm{S}\bm{Z}\bm{V}} = p_{V_0}\,p_{Z_0}\prod_{t=0}^{\infty} p_{M_t|V_t}\,p_{A_t|V_t}\,p_{V_{t+1}| S_t M_t}\,p_{S_t|A_t Z_t}\,p_{Z_{t+1}|S_t}. \label{eq:global_distribution_case_study_ansatz}
		\end{align}
		Let $p_{\bm{M}\bm{A}\bm{S}\bm{Z}\bm{V}}$ be of the form in \cref{eq:global_distribution_case_study_ansatz}. Then, in particular those d-separation relations which involve only variables $M_t$, $A_t$, $S_t$, and $Z_t$, $t\in\mathbb{N}_0$, must hold. All that is left to show is that the distribution in \cref{eq:global_distribution_case_study} can be recovered through marginalization from a distribution of the form in \cref{eq:global_distribution_case_study_ansatz}.
		
		We set $V_{t} = (M_t,A_t)$. This implies that $p_{M_t|V_t}$ and $p_{A_t|V_t}$ in \cref{eq:global_distribution_case_study_ansatz} are delta distributions. Further, we define
		\begin{alignat}{4}
			p_{V_{t+1}| S_t M_t}(v_{t+1}|s_t,m_t) &= \theta^{\mathrm{agt}} (a_{t+1},m_{t+1}|s_t, m_t) \quad &&\text{for } v_{t+1}=(a_{t+1},m_{t+1}), \label{eq:proof_Toy_BN1}\\
			p_{S_t| A_t Z_t}(s_t|a_t,z_t) &= \tilde{\phi}^{\mathrm{env}} (s_t|a_t, z_t),  \label{eq:proof_Toy_BN2}\\
			p_{Z_{t+1}|S_t}(z_{t+1}|s_t) &= \delta_{z_{t+1}, s_t}.
		\end{alignat}
		For each $t\in\mathbb{N}$, we consider all terms in \cref{eq:global_distribution_case_study_ansatz} that contain $V_t$ and marginalize:
		\begin{align}
			\label{eq:proof_Toy_BN3}
			\sum_{v_t\in\mathcal{V}}p_{M_t|V_t}(m_t|v_t)\,p_{A_t|V_t}(a_t|v_t)\,p_{V_t| S_{t-1} M_{t-1}}(v_t|s_{t-1}m_{t-1}) = \theta^{\mathrm{agt}} (a_{t},m_{t}|s_{t-1}, m_{t-1}).
		\end{align}
		Further, we choose $p_{Z_0}=p^{\mathrm{env}}_{Z_{0}}$ and $p_{V_0}= p^{\mathrm{agt}}_{M_0A_0}$.
		Substituting these definitions back into \cref{eq:global_distribution_case_study_ansatz} recovers \cref{eq:global_distribution_case_study}.
	\end{proof}
	
	\subsection{Predictive gap} \label{supp:7.Predicitvity}
	
	The Bayesian network above allows us to simplify the general expression for the predictive gap (\cref{eq:predictiveness_forgetfulness} of the main text) when the environment obeys \cref{eq:TMPhi}.
	For every $t\geq1$, by the chain rule for mutual information (\cref{eq:chain_rule_mutual_information_one_step}),
	\begin{align}
		I[A_{0:t+1}S_{0:t}; S_t |M_t] = I[S_{t-1}A_t;S_t|M_t] + I[A_{0:t}S_{0:t-1};S_t|M_tS_{t-1}A_t].
	\end{align}
	However, the second term vanishes due to the d-separation relation illustrated in \Cref{fig:ToyModel_dsep}, yielding
	\begin{align}
		\label{eq:ToyModel_predictivity}
		I[A_{0:t+1}S_{0:t}; S_t |M_t] = I[S_{t-1}A_t;S_t|M_t].
	\end{align}
	
	\begin{figure}[htbp!]
		\centering
		\begin{tikzpicture}[>=stealth',shorten >=1pt,node distance=2.5cm, scale= 0.83,on grid,auto, state/.style={circle, draw, minimum size=0.8cm, inner sep=1pt},font=\tiny]
			\def \n {3}
			
			\def \S {S}
			\def \A {A}
			\def \M {M}
			\def \Z {Z}
			\def \t {t}
			
			\def\XCol{{0,0,0,2,0,0,0}}
			\def\SCol{{1,1,2,3,0,0,0}}
			\def\ACol{{1,1,1,2,0,0,0}}
			\def\ZCol{{0,0,0,2,0,0,0}}
			
			\foreach \step in {0,...,\n} {
				\pgfmathtruncatemacro{\ms}{-\step}
				\pgfmathtruncatemacro{\aftestate}{-\step+1}
				\pgfmathtruncatemacro{\prevstate}{\step-1}
				
				\ifnum \step= 0
				\pgfmathsetmacro{\xc}{\XCol[\n]}
				\pgfmathsetmacro{\sc}{\SCol[\n]}
				\pgfmathsetmacro{\ac}{\ACol[\n]}
				\pgfmathsetmacro{\zc}{\ZCol[\n]}
				
				\node[state, fill={\ifnum\xc=1 mred!50\else\ifnum\xc=2 mblue!50\else\ifnum\xc=3 mgreen!50\else white\fi\fi\fi}] (Xn-\step) at (0.75,0.75) {$\M_\t$};
				\node[state, fill={\ifnum\ac=1 mred!50\else\ifnum\ac=2 mblue!50\else\ifnum\ac=3 mgreen!50\else white\fi\fi\fi}] (An-\step) at (0,-1.5) {$\A_\t$};
				\node[state, fill={\ifnum\sc=1 mred!50\else\ifnum\sc=2 mblue!50\else\ifnum\sc=3 mgreen!50\else white\fi\fi\fi}] (Bn-\step) at (1.5,-1.5) {$\S_\t$};
				\node[circle, draw=black, fill=black!40, inner sep=2pt] (Fn-0) at (-0.75,-0.75) {};
				\else
				
				\pgfmathsetmacro{\xc}{\XCol[\n+\step]}
				\pgfmathsetmacro{\sc}{\SCol[\n+\step]}
				\pgfmathsetmacro{\ac}{\ACol[\n+\step]}
				\pgfmathsetmacro{\zc}{\ZCol[\n+\step]}
				
				\node[state, fill={\ifnum\xc=1 mred!50\else\ifnum\xc=2 mblue!50\else\ifnum\xc=3 mgreen!50\else white\fi\fi\fi}] (Xn-\step) [right of = Xn-\prevstate] {$\M_{\t+\step}$};
				\node[state, fill={\ifnum\ac=1 mred!50\else\ifnum\ac=2 mblue!50\else\ifnum\ac=3 mgreen!50\else white\fi\fi\fi}] (An-\step) [right of = An-\prevstate] {$\A_{\t+{\step}}$};
				\node[circle, draw=black, fill=black!40, inner sep=2pt] (Fn-\step) [right of = Fn-\prevstate] {};
				
				\ifnum \step< \n
				\node[state, fill={\ifnum\sc=1 mred!50\else\ifnum\sc=2 mblue!50\else\ifnum\sc=3 mgreen!50\else white\fi\fi\fi}] (Bn-\step) [right of = Bn-\prevstate] {$\S_{\t+{\step}}$};
				\fi
				
				\pgfmathsetmacro{\xc}{\XCol[\n-\step]}
				\pgfmathsetmacro{\sc}{\SCol[\n-\step]}
				\pgfmathsetmacro{\ac}{\ACol[\n-\step]}
				\pgfmathsetmacro{\zc}{\ZCol[\n-\step]}
				\node[state, fill={\ifnum\xc=1 mred!50\else\ifnum\xc=2 mblue!50\else\ifnum\xc=3 mgreen!50\else white\fi\fi\fi}] (Xn-\ms) [left of = Xn-\aftestate] {$\M_{\t-\step}$};
				\node[state, fill={\ifnum\ac=1 mred!50\else\ifnum\ac=2 mblue!50\else\ifnum\ac=3 mgreen!50\else white\fi\fi\fi}] (An-\ms) [left of = An-\aftestate] {$\A_{\t-{\step}}$};
				\node[state, fill={\ifnum\sc=1 mred!50\else\ifnum\sc=2 mblue!50\else\ifnum\sc=3 mgreen!50\else white\fi\fi\fi}] (Bn-\ms) [left of = Bn-\aftestate] {$\S_{\t-{\step}}$};
				\node[circle, draw=black, fill=black!40, inner sep=2pt] (Fn-\ms) [left of = Fn-\aftestate] {};
				
				\ifnum \step< \n
				\path[->, thick] (Bn-\prevstate) edge[bend right = 60] (Bn-\step);
				\fi
				\path[->, thick] (Fn-\step) edge (Xn-\step);
				\path[->, thick] (Fn-\step) edge (An-\step);
				\path[->, thick] (Bn-\ms) edge[bend right = 60] (Bn-\aftestate);
				\path[->, thick] (An-\ms) edge (Bn-\ms);
				\path[->, thick] (Xn-\ms) edge (Fn-\aftestate);
				\path[->, thick] (Bn-\ms) edge (Fn-\aftestate);
				\fi
				
				\path[->, thick] (Fn-\ms) edge (Xn-\ms);
				\path[->, thick] (Fn-\ms) edge (An-\ms);
			}
			
			\foreach \step in {1,...,\n} {
				\pgfmathtruncatemacro{\prevstate}{\step-1}
				\path[->, thick] (Xn-\prevstate) edge (Fn-\step);
				\path[->, thick] (An-\prevstate) edge (Bn-\prevstate);
				\path[->, thick] (Bn-\prevstate) edge (Fn-\step);
				
			}
			
			\path[dotted, very thick, black!50] (Xn-\n) edge ++(0.7,-0.7);
			\path[dotted, very thick, black!50] (An-\n) edge ++(1.0,0.0);
			\path[dotted, very thick, black!50] (Fn--\n) edge ++(-0.7,-0.7);
			\path[dotted, very thick, black!50] (Fn--\n) edge ++(-0.7,0.7);
			\path[dotted, very thick, black!50] (Bn--\n) edge ++(-0.7,-0.7);
			\pgfmathtruncatemacro{\prevstate}{\n-1}
			\path[dotted, very thick, black!50] (Bn-\prevstate) edge ++(0.7,-0.7);
			
		\end{tikzpicture}
		\caption{Bayesian network for the environment channel used in the numerical example  (see \Cref{fig:predictive_agent}a and \Cref{fig:8}b of the main text) with colorized d-separation (blue d-separates red and green) used in the derivation of the simplified expression for the predictive gap in \Cref{supp:7.Predicitvity}. The compatibility of the Bayesian network is proven in \Cref{supp:7.BN}. When compared with \Cref{fig:Bayes_net_for_Toy}, the $Z_t$ variables are absorbed into the curved arrows.} \label{fig:ToyModel_dsep}
	\end{figure}
	\subsection{Optimal work rate of predictive agents} \label{supp:7.PredConstruction}
	
	We now determine the maximal work rate achievable by predictive agent models for all $r\in[0,1]$. We first consider $0<r<1$ and then treat the endpoint cases separately.
	
	For every $t\geq1$, for the percept--action loop $(\texttt{agt}_{\mathcal M},\texttt{env}^{(r)})$, we have $I[A_{0:t+1}S_{0:t}; S_t | S_{t-1}A_t] = 0$ by d-separation (see \Cref{fig:ToyModel_dsep}), i.e., the pair $(S_{t-1},A_t)$ is a sufficient
	statistic of the full interaction past $A_{0:t+1}S_{0:t}$ for
	predicting $S_t$. Therefore, the predictive equivalence classes of
	full pasts are determined by the equivalence classes of the pairs
	$(s_{t-1},a_t)$:
	\begin{align}
		(s_{t-1},a_t)\sim(s'_{t-1},a'_t)
		\quad\Longleftrightarrow\quad
		\tilde{\phi}^{\mathrm{env}}(\,\cdot\mid a_t,z=s_{t-1})
		=
		\tilde{\phi}^{\mathrm{env}}(\,\cdot\mid a'_t,z=s'_{t-1}).
	\end{align}
	The environment transition-matrix values defining this equivalence relation are shown in the following table:
	\begin{center}
		\begin{tabular}{cccc}
			\toprule
			$s_{t-1}$ & $a_t$
			& $\tilde{\phi}^{\mathrm{env}}(\texttt{0}\mid a_t,z=s_{t-1})$
			& $\tilde{\phi}^{\mathrm{env}}(\texttt{1}\mid a_t,z=s_{t-1})$ \\
			\midrule
			\texttt{0} & \texttt{0}
			& $\frac{1+r}{2}$ & $\frac{1-r}{2}$ \\
			\texttt{0} & \texttt{1}
			& $\frac{1-r}{2}$ & $\frac{1+r}{2}$ \\
			\texttt{1} & \texttt{0}
			& $1$ & $0$ \\
			\texttt{1} & \texttt{1}
			& $1$ & $0$ \\
			\bottomrule
		\end{tabular}
	\end{center}
	
	For $0<r<1$, the table gives exactly three \emph{likelihood-equivalence classes} over $(S_{t-1},A_t)$,
	\begin{align}
		\mathcal{C}_0 &= \{(\texttt{0},\texttt{0})\},\label{eq:predictive_work_classes_1} \\
		\mathcal{C}_1 &= \{(\texttt{0},\texttt{1})\}, \label{eq:predictive_work_classes_2} \\
		\mathcal{C}_2 &= \{(\texttt{1},\texttt{0}),
		(\texttt{1},\texttt{1})\}. \label{eq:predictive_work_classes_3}
	\end{align}
	Note that, depending on the agent's channel, some combinations $(s_{t-1},a_t)$ may have zero probability and thus some classes may not be realized.
	
	\begin{figure}[htbp!]
		\centering
		\begin{tikzpicture}
			
			\node at (-3,3) {(a)};
			\node[state] (s0env) at (-1.5,1) {$\texttt{0}$};
			\node[state] (s1env) at (1.5,1) {$\texttt{1}$};
			
			\path[-stealth, thick]
			(s0env) edge[loop left]
			node[left,align=center]
			{$\texttt{0}\mid\texttt{0}:\frac{1+r}{2}$\\
				$\texttt{0}\mid\texttt{1}:\frac{1-r}{2}$}
			(s0env)
			
			(s0env) edge[bend left=20]
			node[above,align=center]
			{$\texttt{1}\mid\texttt{0}:\frac{1-r}{2}$\\
				$\texttt{1}\mid\texttt{1}:\frac{1+r}{2}$}
			(s1env)
			
			(s1env) edge[bend left=20]
			node[below,align=center]
			{$\texttt{0}\mid\texttt{0}:1$\\
				$\texttt{0}\mid\texttt{1}:1$}
			(s0env);
			\node at (5.5,3) {(b)};
			\node[state] (m0interior) at (7,1) {$m_0$};
			\node[state] (m1interior) at (10,1) {$m_1$};
			
			\path[-stealth, thick]
			(m0interior) edge[loop left]
			node[left,align=center]
			{$\texttt{1}\mid\texttt{0}:1$}
			(m0interior)
			
			(m0interior) edge[bend left=20]
			node[above,align=center]
			{$\texttt{0}\mid\texttt{1}:\frac12$\\
				$\texttt{1}\mid\texttt{1}:\frac12$}
			(m1interior)
			
			(m1interior) edge[bend left=20]
			node[below,align=center]
			{$\texttt{1}\mid\texttt{0}:1$}
			(m0interior)
			
			(m1interior) edge[loop right]
			node[right,align=center]
			{$\texttt{0}\mid\texttt{1}:\frac12$\\
				$\texttt{1}\mid\texttt{1}:\frac12$}
			(m1interior);
			
			\node at (-3,-2) {(c)};
			\node[state] (m0zero) at (-1.5,-3.5) {$m_0$};
			\node[state] (m1zero) at (1.5,-3.5) {$m_1$};
			
			\path[-stealth, thick]
			(m0zero) edge[loop left]
			node[left,align=center]
			{$\texttt{0}\mid\texttt{0}:\frac12$\\
				$\texttt{1}\mid\texttt{0}:\frac12$}
			(m0zero)
			
			(m0zero) edge[bend left=20]
			node[above,align=center]
			{$\texttt{0}\mid\texttt{1}:\frac12$\\
				$\texttt{1}\mid\texttt{1}:\frac12$}
			(m1zero)
			
			(m1zero) edge[bend left=20]
			node[below,align=center]
			{$\texttt{0}\mid\texttt{0}:\frac12$\\
				$\texttt{1}\mid\texttt{0}:\frac12$}
			(m0zero)
			
			(m1zero) edge[loop right]
			node[right,align=center]
			{$\texttt{0}\mid\texttt{1}:\frac12$\\
				$\texttt{1}\mid\texttt{1}:\frac12$}
			(m1zero);

			\node at (5.5,-2) {(d)};
			\node[state] (m0one) at (7,-3.5) {$m_0$};
			\node[state] (m1one) at (10,-3.5) {$m_1$};
			
			\path[-stealth, thick]
			(m0one) edge[loop left]
			node[left,align=center]
			{$\texttt{0}\mid\texttt{0}:\varphi^{-1}$\\
				$\texttt{1}\mid\texttt{1}:1~~$}
			(m0one)
			
			(m0one) edge[bend left=20]
			node[above,align=center]
			{$\texttt{1}\mid\texttt{0}:\varphi^{-2}$}
			(m1one)
			
			(m1one) edge[bend left=20]
			node[below,align=center]
			{$\texttt{0}\mid\texttt{0}:\varphi^{-1}$\\
				$\texttt{1}\mid\texttt{1}:1$}
			(m0one)
			
			(m1one) edge[loop right]
			node[right,align=center]
			{$\texttt{1}\mid\texttt{0}:\varphi^{-2}$}
			(m1one);
			
		\end{tikzpicture}
		
		\caption{(a) Finite-state-machine representation of the environment
			model realizing $\texttt{env}^{(r)}$. An edge from $z$ to $z'$
			labeled $s\mid a:p$ represents the transition-matrix entry
			$\phi^{\mathrm{env}}(s,z'\mid a,z)=p$.
			In this model, $Z_{t+1}=S_t$, so the environment state records the
			most recently emitted percept. 
			Panels~(b)--(d) show two-memory-state predictive agents attaining
			the maximal work rate among predictive agents for, respectively,
			$0<r<1$, $r=0$, and $r=1$, where
			$\varphi=(1+\sqrt5)/2$.
			In these panels, an edge from $m$ to $m'$ labeled $a'\mid s:p$
			represents the transition-matrix entry
			$\theta^{\mathrm{agt}}(a',m'\mid s,m)=p$.
			Transitions not shown have probability zero.
			For definiteness, in panels~(b)--(d) we choose
			$p^{\mathrm{agt}}_{M_0A_0}(m_0,\texttt{1})=1$. }	\label{fig:predictive_agent}
	\end{figure}

	In the following, let
	\begin{align}
		h_2(x)\coloneqq -x\log_2x-(1-x)\log_2(1-x)
	\end{align}
	denote the binary entropy, with the convention $0\log_2 0=0$.
	\begin{lemma}[Optimal work rate of predictive agents]
		\label{lem:predictive_work_optimum}
		For every $r\in[0,1]$ and every finite memory size $d\geq2$,
		\begin{align}
			\normalfont
			\max_{\texttt{agt}_{\mathcal M}\in
				\mathbb{A}^{\lrstacksmall\texttt{env}^{(r)}}_{\mathrm{pred}}(d)}
			W(\texttt{agt}_{\mathcal M},\texttt{env}^{(r)})
			=
			\begin{cases}
				\frac13,
				& r=0,\\[1ex]
				\frac{\rho-h_2(\rho)}{1+\rho},
				& 0<r<1,\\[1ex]
				\log_2\varphi,
				& r=1,
			\end{cases}
		\end{align}
		where $\rho=(1+r)/2$ and
		$\varphi=(1+\sqrt5)/2$.
		Moreover, the minimum a.m.\ predictive gap is zero, so every
		$\texttt{agt}_{\mathcal M}\in
		\mathbb{A}^{\lrstacksmall\texttt{env}^{(r)}}_{\mathrm{pred}}(d)$
		satisfies
		$\cesaro{I[A_{0:t+1}S_{0:t};S_t\mid M_t]}_t=0$.
	\end{lemma}
	
	\begin{proof}
		\emph{Case $0<r<1$.}		
		We first show that an agent with two memory states can attain a
		vanishing a.m.\ predictive gap. Consider the agent model $\texttt{agt}_{\mathcal{M}}^*$
		given by the finite-state machine in   \Cref{fig:predictive_agent}b with $p^{\mathrm{agt}}_{M_0A_0}(m_0,\texttt{1})=1$.
		
		For every $t\geq 1$, the agent transition rules imply
		\begin{align}
			M_t=m_0
			&\quad\Longleftrightarrow\quad
			(S_{t-1},A_t)\in \mathcal{C}_1, &
			M_t=m_1
			&\quad\Longleftrightarrow\quad
			(S_{t-1},A_t)
			\in \mathcal{C}_2.
		\end{align}
		Meanwhile, the transition rules of $\texttt{env}^{(r)}$ imply (see the table above \cref{eq:predictive_work_classes_1})
		\begin{align}
			p_{S_t\mid S_{t-1}A_t}
			(\texttt{0}\mid s,a)=1
			\qquad
			\forall (s,a)\in\mathcal C_2.
		\end{align}
		Conditional on $M_t=m_0$, the only possible pair is
		$(S_{t-1},A_t)=(\texttt{0},\texttt{1})$. Conditional on
		$M_t=m_1$, the possible pairs belong to $\mathcal C_2$ and for every
		such pair the environment produces $S_t=\texttt{0}$ with probability
		one. Thus, for every $(s,a,m)$ satisfying $p_{S_{t-1}A_tM_t}(s,a,m)>0$ we have $p_{S_t\mid S_{t-1}A_tM_t}
		(\cdot\mid s,a,m)
		=
		p_{S_t\mid M_t}(\cdot\mid m)$	and therefore
		\begin{align}
			I[S_{t-1}A_t;S_t\mid M_t]=0
			\quad 	\forall t\in\mathbb{N}.
		\end{align}
		Using \Cref{eq:ToyModel_predictivity}, it follows that
		the agent model has vanishing a.m.\ predictive gap. For any $d>2$, this two-state memory can be embedded into an alphabet of
		size $d$ by adjoining $d-2$ memory states that are never reached.
		Thus, for every $d\geq2$, the predictive gap of agent models in
		$\mathbb{A}^{\lrstacksmall\texttt{env}^{(r)}}_{\mathrm{pred}}(d)$
		equals zero.
		
		We next derive an upper bound on the work rate of an arbitrary agent model in $\mathbb{A}^{\lrstacksmall\texttt{env}^{(r)}}_{\mathrm{pred}}(d)$ for any $d\geq 2$.
		
		By \Cref{lem:global_Markov_chain}, the corresponding global process
		$\bm{U}=(M_t,A_t,S_t,Z_t)_{t=0}^{\infty}$ is a homogeneous
		finite-state Markov chain. Let
		$\overline{\bm{U}}
		=
		(\overline{M}_t,\overline{A}_t,
		\overline{S}_t,\overline{Z}_t)_{t=0}^{\infty}$
		be its asymptotic process, as defined in
		\Cref{cor:Markov_chains_asymptotic}(iv), and let $\ell$ be its period.
		
		Using \cref{eq:ToyModel_predictivity} and
		\Cref{cor:Markov_chains_asymptotic}(iv), the vanishing a.m.\
		predictive gap can be written as
		\begin{align}
			0
			&=
			\frac{1}{\ell}
			\sum_{c=1}^{\ell}
			I[\overline{S}_{c-1}\overline{A}_c;
			\overline{S}_c|\overline{M}_c],
			\label{eq:predictive_work_phase_average}
		\end{align}
		and, by the nonnegativity of conditional mutual information,
		\begin{align}
			I[\overline{S}_{c-1}\overline{A}_c;
			\overline{S}_c|\overline{M}_c]
			=0
			\quad
			\forall c\in\{1,\dotsc,\ell\}.
			\label{eq:predictive_work_phase_zero}
		\end{align}
		Furthermore, we have 
		\begin{align}
			I[\overline{M}_c;\overline{S}_c|
			\overline{S}_{c-1}\overline{A}_c]
			=0	\quad
			\forall c\in\{1,\dotsc,\ell\}
			\label{eq:predictive_work_environment_independence}
		\end{align}
		which is easily verified using the Bayesian network in \Cref{supp:7.BN}.
		Then, for any values $s,a,m$ with $p_{\overline{S}_{c-1}\overline{A}_c\overline{M}_c}(s,a,m)>0$, 
		\cref{eq:predictive_work_phase_zero,eq:predictive_work_environment_independence} imply
		\begin{align}
			p_{\overline{S}_c|\overline{M}_c}(\,\cdot\,|m)
			&=
			p_{\overline{S}_c|
				\overline{S}_{c-1}\overline{A}_c\overline{M}_c}
			(\,\cdot\,|s,a,m)
			=
			p_{\overline{S}_c|
				\overline{S}_{c-1}\overline{A}_c}
			(\,\cdot\,|s,a).
			\label{eq:predictive_work_likelihood_refinement}
		\end{align}
		Therefore, all percept--action pairs $(s,a)$ occurring with the same memory
		value $m$ belong to the same likelihood-equivalence class in 	\cref{eq:predictive_work_classes_1,eq:predictive_work_classes_2,eq:predictive_work_classes_3}.

		Define $	\overline{T}_c
		\coloneqq
		\tau(\overline{S}_{c-1},\overline{A}_c)$
		where $\tau$ assigns to each percept--action pair the class in
		\cref{eq:predictive_work_classes_1,eq:predictive_work_classes_2,eq:predictive_work_classes_3}. 
		Thus, for each
		phase $c$, there exists a deterministic map $g_c$ such that
		\begin{align}
			\overline{T}_c=g_c(\overline{M}_c)
			\qquad\text{almost surely}.
			\label{eq:predictive_work_T_function_M}
		\end{align}
		
		Because $\overline{T}_c$ is a function of $\overline{M}_c$,
		we have
		\begin{align}
			H(\overline{A}_c|\overline{M}_c)
			\leq
			H(\overline{A}_c|\overline{T}_c).
			\label{eq:predictive_work_action_entropy_bound}
		\end{align}
		Furthermore, \cref{eq:predictive_work_likelihood_refinement} gives
		\begin{align}
			H(\overline{S}_c|\overline{M}_c)
			=
			H(\overline{S}_c|\overline{T}_c).
			\label{eq:predictive_work_percept_entropy_equality}
		\end{align}
		
		Let $\pi_c
		\coloneqq
		p_{\overline{S}_{c-1}}(\texttt{1})$. For the classes $\mathcal{C}_0$ and $\mathcal{C}_1$, the entropy of the conditional
		percept distribution equals $h_2(\rho)$, while for $\mathcal{C}_2$ it vanishes.
		Thus,
		\begin{align}
			H(\overline{S}_c|\overline{T}_c)
			=
			(1-\pi_c)h_2(\rho).
			\label{eq:predictive_work_percept_entropy}
		\end{align}
		Moreover, $\overline{T}_c$ determines the action in classes $\mathcal{C}_0$ and
		$\mathcal{C}_1$. Thus, only class $\mathcal{C}_2$ can contribute to  $H(\overline{A}_c|\overline{T}_c)$, and therefore
		\begin{align}
			H(\overline{A}_c|\overline{T}_c)
			&=
			\pi_c
			H(\overline{A}_c|\overline{T}_c=\mathcal{C}_2)
			\leq
			\pi_c.
			\label{eq:predictive_work_action_entropy}
		\end{align}
		
		By \cref{eq:work_rate} of the main text and
		\Cref{cor:Markov_chains_asymptotic}(iv),
		\begin{align}
			W(\texttt{agt}_{\mathcal M},\texttt{env}^{(r)})
			&=
			\frac{1}{\ell}
			\sum_{c=1}^{\ell}
			\left[
			H(\overline{A}_c|\overline{M}_c)
			-
			H(\overline{S}_c|\overline{M}_c)
			\right]
			\\
			&\leq
			\frac{1}{\ell}
			\sum_{c=1}^{\ell}
			\left[
			\pi_c-(1-\pi_c)h_2(\rho)
			\right]
			\label{eq:predictive_work_first_upper_bound}\\
			&= \overline{\pi}
			-
			(1-\overline{\pi})h_2(\rho).
			\label{eq:predictive_work_second_upper_bound}
		\end{align}
		where, for \cref{eq:predictive_work_second_upper_bound}, we defined $\overline{\pi}
		\coloneqq
		\frac{1}{\ell}\sum_{c=1}^{\ell}\pi_c$.	
		
		It remains to bound $\overline{\pi}$. We find
		\begin{align}
			\pi_{c+1}
			&=
			(1-\rho)\,
			p_{\overline S_{c-1}\overline A_c}
			(\texttt{0},\texttt{0})
			+
			\rho\,
			p_{\overline S_{c-1}\overline A_c}
			(\texttt{0},\texttt{1})
			\\
			&\leq
			\rho\left[
			p_{\overline S_{c-1}\overline A_c}
			(\texttt{0},\texttt{0})
			+
			p_{\overline S_{c-1}\overline A_c}
			(\texttt{0},\texttt{1})
			\right]
			\\
			&=
			\rho\,p_{\overline S_{c-1}}(\texttt{0})
			=
			\rho(1-\pi_c), \label{eq:pi_inequality}
		\end{align}
		where the first equality follows from the environment transition table since $\overline S_c=\texttt{1}$ occurs with probability $1-\rho$ following $(\overline S_{c-1}\overline A_c)=(\texttt{0},\texttt{0})$, with probability $\rho$ following $(\overline S_{c-1}\overline A_c)=(\texttt{0},\texttt{1})$, and never following a pair whose first component is $\texttt{1}$. Further, the inequality follows from $\rho\geq 1-\rho$.
		
		Averaging the inequality from \cref{eq:pi_inequality} over one period and using the
		$\ell$-cyclostationarity of $\overline{\bm U}$ gives $	\overline{\pi}
		\leq
		\rho(1-\overline{\pi})$, 
		and therefore
		\begin{align}
			\overline{\pi}
			\leq
			\frac{\rho}{1+\rho}.
			\label{eq:predictive_work_percept_probability}
		\end{align}
		The right-hand side of
		\cref{eq:predictive_work_second_upper_bound} is increasing in
		$\overline{\pi}$. Substituting
		\cref{eq:predictive_work_percept_probability} then yields
		\begin{align}
			W(\texttt{agt}_{\mathcal M},\texttt{env}^{(r)})
			\leq
			\frac{\rho-h_2(\rho)}{1+\rho}.
			\label{eq:predictive_work_final_upper_bound}
		\end{align}
		Finally, consider again the two-state agent model $\texttt{agt}_{\mathcal{M}}^*$.	For this agent model, the percept process $\bm{S}$ is itself a two-state Markov chain. If $S_{t-1}=\texttt{0}$, the agent chooses
		$A_t=\texttt{1}$, and thus $p_{S_t|S_{t-1}}(\texttt{1}\mid\texttt{0})=\rho$. If $S_{t-1}=\texttt{1}$, the environment produces $S_t=\texttt{0}$ irrespective of the agent's action, and thus $p_{S_t|S_{t-1}}(\texttt{1}\mid \texttt{1})=0$.
		
		Setting $q_t\coloneqq p_{S_t}(\texttt{1})$,
		we therefore obtain, for every $t\in\mathbb{N}$,
		\begin{align}
			q_t
			&=
			\rho\,p_{S_{t-1}}(\texttt{0})
			=
			\rho(1-q_{t-1}). \label{eq:qt}
		\end{align}
		The unique fixed point of this recursion satisfies 
		\begin{align}
			q_*=\rho(1-q_*) \label{eq:fixed_point_equation}
		\end{align}
		and is thus given by $q_*=\frac{\rho}{1+\rho}$.
		Furthermore,
		\begin{align}
			q_t-q_*&=-\rho(q_{t-1}-q_*)\\
			&=(-\rho)^t(q_0-q_*)
		\end{align}
		where the first equality follows from \cref{eq:qt,eq:fixed_point_equation}, and the second equality follows from recursion in $t$. 
		Since $0<r<1$ implies $0<\rho<1$, this difference converges to zero.
		Thus,
		\begin{align}
			\lim_{t\to\infty}p_{S_t}(\texttt{1})
			&=
			\frac{\rho}{1+\rho},
			&
			\lim_{t\to\infty}p_{S_t}(\texttt{0})
			&=
			\frac{1}{1+\rho}.
		\end{align}
		The agent's update rule implies, for every $t\geq1$,
		\begin{align}
			p_{M_t}(m_1)
			&= p_{S_{t-1}}(\texttt{1})
			= q_{t-1},
			&
			p_{M_t}(m_0)
			=
			p_{S_{t-1}}(\texttt{0})
			&=
			1-q_{t-1},
		\end{align}
		and thus,
		\begin{align}
			\lim_{t\to\infty}p_{M_t}(m_1)
			&=
			\frac{\rho}{1+\rho},
			&
			\lim_{t\to\infty}p_{M_t}(m_0)
			&=
			\frac{1}{1+\rho}.
		\end{align}
		
		Conditional on $M_t=m_0$, the agent chooses $A_t=\texttt{1}$
		deterministically, while $S_t$ has probabilities $1-\rho$ and
		$\rho$. Conditional on $M_t=m_1$, the agent chooses its action
		uniformly, while $S_t=\texttt{0}$ deterministically. Thus,
		\begin{align}
			H(A_t\mid M_t=m_0)&=0,
			&
			H(A_t\mid M_t=m_1)&=1,\\
			H(S_t\mid M_t=m_0)&=h_2(\rho),
			&
			H(S_t\mid M_t=m_1)&=0.
		\end{align}
		It follows that, for every $t\geq1$,
		\begin{align}
			H(A_t\mid M_t)
			&=
			q_{t-1},\\
			H(S_t\mid M_t)
			&=
			(1-q_{t-1})h_2(\rho).
		\end{align}
		Taking $t\to\infty$ therefore gives
		\begin{align}
			\lim_{t\to\infty}H(A_t\mid M_t)
			&=
			\frac{\rho}{1+\rho}, &
			\lim_{t\to\infty}H(S_t\mid M_t)
			&=
			\frac{h_2(\rho)}{1+\rho}.
		\end{align}
		Since ordinary convergence implies convergence in the Ces\`aro mean,
		\begin{align}
			\cesaro{H(A_t\mid M_t)}_t
			&=
			\frac{\rho}{1+\rho},
			&
			\cesaro{H(S_t\mid M_t)}_t
			&=
			\frac{h_2(\rho)}{1+\rho}.
		\end{align}
		Thus,
		$W(\texttt{agt}^*_{\mathcal M},\texttt{env}^{(r)})
		=\frac{\rho-h_2(\rho)}{1+\rho}$.
		This saturates \cref{eq:predictive_work_final_upper_bound}.

		\emph{Case $r=0$.}
		For every $t\geq1$, the chain rule gives
		\begin{align}
			I[A_tM_t;S_t\mid S_{t-1}]
			&=
			I[A_t;S_t\mid S_{t-1}]
			+
			I[M_t;S_t\mid S_{t-1}A_t].
		\end{align}
		At $r=0$, the environment ignores actions, i.e., $p_{S_t\mid S_{t-1}A_t}(s'\mid s,a)
		=
		p_{S_t\mid S_{t-1}}(s'\mid s)$
		for every $s'\in\mathcal S$ and every $(s,a)$ satisfying
		$p_{S_{t-1}A_t}(s,a)>0$. This is equivalent to $I[A_t;S_t\mid S_{t-1}]=0$.
		Moreover, a d-separation relation (see
		\Cref{fig:Bayes_net_for_Toy}) gives
		$I[M_t;S_t\mid S_{t-1}A_t]=0$.
		Substituting these equalities into the chain-rule identity gives
		\begin{align}
			I[A_tM_t;S_t\mid S_{t-1}]=0.
			\label{eq:predictive_endpoint_zero_conditional_independence}
		\end{align}
		
		Now consider the two-state agent shown in
		\Cref{fig:predictive_agent}c. For every $t\geq1$, its memory $M_t$
		determines $S_{t-1}$, and hence
		\begin{align}
			I[S_{t-1};S_t\mid M_t]=0.
		\end{align}
		Moreover, applying the chain rule to
		\cref{eq:predictive_endpoint_zero_conditional_independence} gives
		\begin{align}
			0
			=
			I[A_tM_t;S_t\mid S_{t-1}]
			=
			I[M_t;S_t\mid S_{t-1}]
			+
			I[A_t;S_t\mid S_{t-1}M_t].
		\end{align}
		Both terms are nonnegative, so
		$I[A_t;S_t\mid S_{t-1}M_t]=0$. Consequently,
		\begin{align}
			I[S_{t-1}A_t;S_t\mid M_t]
			&=
			I[S_{t-1};S_t\mid M_t]
			+
			I[A_t;S_t\mid S_{t-1}M_t]
			=
			0.
		\end{align}
		By \cref{eq:ToyModel_predictivity}, this agent has vanishing a.m.\
		predictive gap. Since the predictive gap is nonnegative, its minimum is
		zero for every $d\geq2$, where for $d>2$ the agent can be embedded by
		adjoining unreachable memory states.
		
		We next determine the percept-entropy contribution to the work rate.
		By
		\cref{eq:predictive_endpoint_zero_conditional_independence}, every
		agent model satisfies
		\begin{align}
			H(S_t\mid S_{t-1}A_tM_t)
			=
			H(S_t\mid S_{t-1}).
		\end{align}
		At $r=0$, $S_t$ is uniform when
		$S_{t-1}=\texttt{0}$ and is deterministically $\texttt{0}$ when
		$S_{t-1}=\texttt{1}$. Therefore,
		\begin{align}
			H(S_t\mid S_{t-1}A_tM_t)
			&=
			p_{S_{t-1}}(\texttt{0})h_2\!\left(\frac12\right)
			+
			p_{S_{t-1}}(\texttt{1})h_2(0)
			\nonumber\\
			&=
			p_{S_{t-1}}(\texttt{0}).
			\label{eq:predictive_endpoint_zero_conditional_entropy}
		\end{align}
		
		To determine the asymptotic value of this probability, set
		$q_t\coloneqq p_{S_t}(\texttt{1})$. The law of total probability gives
		\begin{align}
			q_t
			&=
			\frac12p_{S_{t-1}}(\texttt{0})
			+
			0\cdot p_{S_{t-1}}(\texttt{1})
			=
			\frac12(1-q_{t-1}).
		\end{align}
		This is the recursion in \cref{eq:qt} with $\rho=\frac12$. The argument
		following that equation therefore yields $q_t\longrightarrow\frac13$ and, thus, $p_{S_t}(\texttt{0})\longrightarrow\frac23$. Together with
		\cref{eq:predictive_endpoint_zero_conditional_entropy}, this implies
		\begin{align}
			\cesaro{H(S_t\mid S_{t-1}A_tM_t)}_t
			=
			\frac23.
			\label{eq:predictive_endpoint_zero_irreducible_entropy}
		\end{align}
		
		Now let
		$\texttt{agt}_{\mathcal M}\in
		\mathbb{A}^{\lrstacksmall\texttt{env}^{(0)}}(d)$
		be arbitrary. Since conditioning reduces entropy,
		\begin{align}
			\cesaro{H(S_t\mid M_t)}_t
			&\geq
			\cesaro{H(S_t\mid S_{t-1}A_tM_t)}_t
			=
			\frac23.
		\end{align}
		Together with $H(A_t\mid M_t)\leq1$, this gives
		\begin{align}
			W(\texttt{agt}_{\mathcal M},\texttt{env}^{(0)})
			&\leq
			1-\frac23
			=
			\frac13.
		\end{align}
		
		For the agent model in \Cref{fig:predictive_agent}c,
		$I[S_{t-1}A_t;S_t\mid M_t]=0$ for every $t\geq1$, as shown above.
		Hence \cref{eq:predictive_endpoint_zero_irreducible_entropy} gives
		$\cesaro{H(S_t\mid M_t)}_t=2/3$.
		Moreover, this agent model chooses $A_t$ uniformly conditional on $M_t$, so
		$H(A_t\mid M_t)=1$ for every $t\geq1$. It therefore attains the
		upper bound.
		
		Hence, for every $d\geq2$,
		\begin{align}
			\max_{\texttt{agt}_{\mathcal M}\in
				\mathbb{A}^{\lrstacksmall\texttt{env}^{(0)}}_{\mathrm{pred}}(d)}
			W(\texttt{agt}_{\mathcal M},\texttt{env}^{(0)})
			=
			\frac13.
		\end{align}
		
		\emph{Case $r=1$.}
		For any agent model, let
		$\overline{\bm U}$ be its asymptotic process and let $\ell$ be its
		period. The chain rule and the fact that conditioning reduces entropy give
		\begin{align}
			H(\overline A_c\mid\overline M_c)
			-
			H(\overline S_c\mid\overline M_c)
			&=
			H(\overline A_c\mid\overline S_c\overline M_c)
			-
			H(\overline S_c\mid\overline A_c\overline M_c)\leq
			H(\overline A_c\mid\overline S_c).
		\end{align}
		It follows that
		\begin{align}
			W(\texttt{agt}_{\mathcal M},\texttt{env}^{(1)})
			\leq
			\frac1\ell
			\sum_{c=1}^{\ell}
			H(\overline A_c\mid\overline S_c).
			\label{eq:predictive_endpoint_one_first_bound}
		\end{align}
		
		Let $(S^-,A,S^+)$ be auxiliary random variables with joint
		distribution
		\begin{align}
			p_{S^-AS^+}(s,a,s')
			\coloneqq
			\frac1\ell\sum_{c=1}^{\ell}
			p_{\overline S_{c-1}\overline A_c\overline S_c}(s,a,s').
			\label{eq:phase_averaged_variables}
		\end{align}
		Cyclostationarity implies
		$	p_{S^-}=	\frac1\ell\sum_{c=1}^{\ell}p_{\overline S_{c-1}}
		=
		\frac1\ell\sum_{c=1}^{\ell}p_{\overline S_c}
		=
		p_{S^+}$.
		By concavity of conditional entropy, we have
		\begin{align}
			\frac1\ell\sum_{c=1}^{\ell}
			H(\overline A_c\mid\overline S_c)
			\leq
			H(A\mid S^+). \label{eq:concavity_bound}
		\end{align}
		
		Set
		\begin{align}
			q\coloneqq p_{S^+}(\texttt{1}).
		\end{align}
		At $r=1$, the environment transition is deterministic and satisfies
		\begin{align}
			S^+
			=
			\begin{cases}
				\texttt{1},
				&(S^-,A)=(\texttt{0},\texttt{1}),\\
				\texttt{0},
				&\text{otherwise}.
			\end{cases} \label{eq:s_plus_cases}
		\end{align}
		The transition rule and $p_{S^-}=p_{S^+}$ imply
		\begin{align}
			q
			=p(S^+=\texttt{1})
			=p(S^-=\texttt{0},A=\texttt{1})
			\leq p(S^-=\texttt{0})
			=1-q.
		\end{align}
		Hence $p(S^-=\texttt{0})=1-q\geq1/2$, so the conditioning event in the following definition of $\alpha$ has positive probability. Define
		\begin{align}
			\alpha&\coloneqq
			p_{A\mid S^-}(\texttt{1}\mid\texttt{0}),
			&
			\beta&\coloneqq
			p_{A\mid S^-}(\texttt{1}\mid\texttt{1}),
		\end{align}
		where $\beta$ may be chosen arbitrarily if $p_{S^-}(\texttt{1})=0$. The transition rule also gives
		\begin{align}
			p_{S^+}(\texttt{1})
			&=
			p_{S^-A}(\texttt{0},\texttt{1}), \label{eq:p_+1}\\
			p_{AS^+}(\texttt{1},\texttt{0})
			&=
			p_{S^-A}(\texttt{1},\texttt{1}).
			\label{eq:p_+2}
		\end{align}
		Using the definitions of $q$ and $\alpha$,
		\cref{eq:p_+1} gives
		\begin{align}
			q
			&=
			p_{S^+}(\texttt{1})=
			p_{S^-A}(\texttt{0},\texttt{1})=
			p_{S^-}(\texttt{0})\,
			p_{A\mid S^-}(\texttt{1}\mid\texttt{0})=
			(1-q)\alpha,
		\end{align}
		where we used $p_{S^-}(\texttt{0})=1-q$ which follows from  $p_{S^-}=p_{S^+}$.
		Solving for $q$ gives
		\begin{align}
			q=\frac{\alpha}{1+\alpha}. \label{eq:q_solution}
		\end{align}
		Furthermore, using the definitions of $q$, $\alpha$, and $\beta$,
		\cref{eq:p_+2,eq:q_solution} gives
		\begin{align}
			p_{A\mid S^+}(\texttt{1}\mid\texttt{0})
			&=
			\frac{
				p_{AS^+}(\texttt{1},\texttt{0})
			}{
				p_{S^+}(\texttt{0})
			}=
			\frac{
				p_{S^-A}(\texttt{1},\texttt{1})
			}{
				1-q
			}=
			\frac{q\beta}{1-q}
			=
			\alpha\beta.
		\end{align}
		Writing $H(A\mid S^+)
		=
		p_{S^+}(\texttt{0})\,
		H(A\mid S^+=\texttt{0})
		+
		p_{S^+}(\texttt{1})\,
		H(A\mid S^+=\texttt{1})$ where $H(A\mid S^+=\texttt{1})=0$ since conditional on $S^+=\texttt{1}$, the action is deterministically
		$A=\texttt{1}$ (see \cref{eq:s_plus_cases}), we find
		\begin{align}
			H(A\mid S^+)
			&=
			(1-q)h_2(\alpha\beta)
			=
			\frac{h_2(\alpha\beta)}{1+\alpha}.
			\label{eq:alpha_beta_rewrite}
		\end{align}
		For fixed $\alpha$, maximization over $\beta$ gives
		\begin{align}
			\max_{0\leq\beta\leq1}
			\frac{h_2(\alpha\beta)}{1+\alpha}
			=
			\begin{cases}
				\displaystyle\frac{h_2(\alpha)}{1+\alpha},
				&0\leq\alpha\leq\frac12,\\[2ex]
				\displaystyle\frac{1}{1+\alpha},
				&\frac12<\alpha\leq1.
			\end{cases}
		\end{align}
		For $0<\alpha<1/2$, the derivative of the first branch is
		\begin{align}
			\frac{\mathrm d}{\mathrm d\alpha}
			\frac{h_2(\alpha)}{1+\alpha}
			&=
			\frac{\log_2\!\left((1-\alpha)^2/\alpha\right)}
			{(1+\alpha)^2}.
		\end{align}
		The derivative vanishes at the unique solution of
		$(1-\alpha)^2=\alpha$ in $(0,1/2)$. The derivative is positive
		for $\alpha<\alpha_*$ and negative for $\alpha>\alpha_*$. Thus, the unique maximum is attained at $\alpha_* =	\frac{3-\sqrt5}{2}$ which equals $\varphi^{-2}$.	At this point, we find $\frac{h_2(\alpha_*)}{1+\alpha_*}
		=
		\log_2\varphi$.
		
		The second branch is at most $2/3$. This equals the value of the first branch at $\alpha=1/2$ and is
		strictly smaller than its unique maximum since $\alpha_*\neq1/2$.
		Thus, \cref{eq:predictive_endpoint_one_first_bound,eq:concavity_bound,eq:alpha_beta_rewrite} give
		\begin{align}
			W(\texttt{agt}_{\mathcal M},\texttt{env}^{(1)})
			\leq
			\log_2\varphi.
			\label{eq:predictive_endpoint_one_final_bound}
		\end{align}
		To attain this bound, consider the agent model $\texttt{agt}_{\mathcal{M}}^*$
		given by the finite-state machine in   \Cref{fig:predictive_agent}d with $p^{\mathrm{agt}}_{M_0A_0}(m_0,\texttt{1})=1$. It has valid transition probabilities since
		$\varphi^{-1}+\varphi^{-2}=1$. At $r=1$, the transition rules imply,
		for every $t\geq1$,
		\begin{align}
			M_t=m_0
			\quad&\Longleftrightarrow\quad
			S_t=\texttt{0},
			&
			M_t=m_1
			\quad&\Longleftrightarrow\quad
			S_t=\texttt{1}.
		\end{align}
		Thus $H(S_t\mid M_t)=0$ and therefore
		$I[S_{t-1}A_t;S_t\mid M_t]=0$. By
		\cref{eq:ToyModel_predictivity}, this agent model has vanishing a.m.\
		predictive gap.
		
		For this agent model, set $q_t\coloneqq p_{S_t}(\texttt{1})$. Its transition
		rules give
		\begin{align}
			q_t
			=
			\varphi^{-2}(1-q_{t-1}),
		\end{align}
		which is the recursion in \cref{eq:qt} with $\rho$ replaced by
		$\varphi^{-2}$. The same argument therefore gives
		\begin{align}
			q_t\longrightarrow
			q_*
			\coloneqq
			\frac{\varphi^{-2}}{1+\varphi^{-2}}.
		\end{align}
		
		For every $t\geq1$, $M_t=m_0$ if and only if
		$S_t=\texttt{0}$, while $M_t=m_1$ if and only if
		$S_t=\texttt{1}$. Thus $H(S_t\mid M_t)=0$. In particular, since this implies	$I[S_{t-1}A_t;S_t\mid M_t]=0$, by
		\cref{eq:ToyModel_predictivity}, this agent model has vanishing a.m.\
		predictive gap. Furthermore, conditional on
		$M_t=m_1$, the action is deterministically $\texttt{1}$, and the transition rules imply that
		$(A_t,M_t)=(\texttt{1},m_0)$ occurs precisely when
		$S_{t-1}=\texttt{1}$. Thus
		\begin{align}
			p_{A_t\mid M_t}(\texttt{1}\mid m_0)
			=
			\frac{p_{A_tM_t}(\texttt{1},m_0)}
			{p_{M_t}(m_0)}= \frac{p_{S_{t-1}}(\texttt{1})}
			{p_{S_t}(\texttt{0})}
			=
			\frac{q_{t-1}}{1-q_t},
		\end{align}
		and therefore
		\begin{align}
			H(A_t\mid M_t)
			&=
			(1-q_t)
			h_2\left(\frac{q_{t-1}}{1-q_t}\right)\\
			&\longrightarrow
			\frac{h_2(\varphi^{-2})}{1+\varphi^{-2}}
			=
			\log_2\varphi.
		\end{align}
		Ordinary convergence implies convergence in the Ces\`aro mean, and
		therefore
		\begin{align}
			W(\texttt{agt}^*_{\mathcal M},\texttt{env}^{(1)})
			=
			\log_2\varphi.
		\end{align}
		This saturates \cref{eq:predictive_endpoint_one_final_bound}, and
		therefore
		\begin{align}
			\max_{\texttt{agt}_{\mathcal M}\in
				\mathbb{A}^{\lrstacksmall\texttt{env}^{(1)}}_{\mathrm{pred}}(d)}
			W(\texttt{agt}_{\mathcal M},\texttt{env}^{(1)})
			=
			\log_2\varphi.
		\end{align}
		
		The attaining agents use two memory states. For every $d>2$, they can
		be embedded into an alphabet of size $d$ by adjoining $d-2$ memory
		states that are never reached.
	\end{proof}
	
	For the endpoint cases $r=0$ and $r=1$, the upper bounds established
	above hold for all finite-memory agent models, irrespective of
	predictivity. Since the corresponding two-state predictive agent models,
	and their embeddings for $d>2$, attain these bounds, they are globally
	work-efficient. Consequently,
	$W^{(d)}(\texttt{env}^{(0)})=1/3$ and
	$W^{(d)}(\texttt{env}^{(1)})=\log_2\varphi$ for every $d\geq2$.
	
	OpenAI ChatGPT (GPT-5.6 Sol) was used to assist with the derivation of the proof of \Cref{lem:predictive_work_optimum}. The authors supplied the statement, assumptions, proof techniques, and relevant definitions and used the tool to explore possible proof strategies. The authors independently verified all proposed proof elements and wrote the final proof.
	
	\subsection{Numerical results} \label{supp:7.Numerics}
	
	In this section we present numerical results for the work rate optimization. The objective is to numerically estimate the maximal achievable work rate for agent models in (i) the set of memoryless agent models, $\mathbb{A}^{\lrstacksmall \texttt{env}}(1)$, and (ii) the set of arbitrary agent models, $\mathbb{A}^{\lrstacksmall \texttt{env}}(d)$, with memory size $d=2,3,4$. For each agent type we numerically optimize the work rate as a function of the agent parameters.
	
	Every agent model, together with the environment model under consideration, gives rise to a percept-action loop with global Markov chain $\bm{U}=((M_0,A_0,S_0,Z_0),(M_1,A_1,S_1,Z_1),\dotsc)$ (see \Cref{supp:4}). Work rate is then a function of the form $\cesaro{f(p_{U_t})}_t$ with $f(p_{U_t})=H(A_t|M_t)-H(S_t|M_t)$. As we have seen in \Cref{supp:2}, such Ces\`aro limits of finite-state Markov chains are determined by the asymptotic properties of the finite-state Markov chain. Functions of the type $\cesaro{f(p_{U_t})}_t$ can be computed using the initial distribution of the Markov chain and the spectral properties of the Markov chain's transition matrix (for further details, we refer the reader to \cite{ash2008basic, iosifescu2014finite, kemeny1960finite}).

	We employ numerical optimization routines from \texttt{scipy} to maximize the work rate with respect to the agent parameters. Each optimization is initialized with $1000$ random parameter vectors uniformly distributed over the feasible domain. The results are shown in \Cref{fig:9} of the main text.

	In the following, we comment on the discontinuity in the predictive gap of the numerically optimized general agent as the environment’s feedback parameter $r$ is varied (see \Cref{fig:9}b of the main text). To understand the origin of this discontinuity, we inspect the numerically optimized general agent models in more detail. Examining the transition matrices of the best two-state numerical solutions, we identify two restricted ansätze (obtained by setting certain transition probabilities to zero). In the following, we refer to these as \emph{ansatz~1} and \emph{ansatz~2} (see \Cref{fig:TM_Transition}a,b for their finite-state machine representations).
	
	When the remaining free parameters are numerically optimized, the best value obtained with ansatz~1 for $r\lesssim0.71$ (respectively, with ansatz~2 for $r\gtrsim0.71$) is indistinguishable from that of the unconstrained numerical optimization shown in \Cref{fig:9}a of the main text.
	
	The corresponding predictive gaps of ansatz~1 and ansatz~2, each optimized with respect to work rate, are shown in \Cref{fig:TM_Transition}d. For $r\lesssim0.71$, the predictive gap of the best unconstrained numerical solution is indistinguishable within numerical precision from that of ansatz~1, whereas for $r\gtrsim0.71$ it is indistinguishable from that of ansatz~2. Crucially, the predictive gaps of ansatz~1 and ansatz~2 differ by a finite amount, and the discontinuous jump in the predictive gap of the best unconstrained numerical solution at $r\approx0.71$ matches this difference within numerical precision.
	
	These observations support the interpretation that the discontinuity in the predictive gap originates from a change in the agent design---from ansatz~1 to ansatz~2---as the feedback strength $ r $ crosses a critical value.

	\begin{figure}[htbp!]
		\centering
		\begin{tikzpicture}
			\node at (0,7) {(a)};
			\node at (3.6,7) {agent model ansatz 1};
			\node[state] (q0a) at (2.5,4.9) {$m_0$};
			\node[state, right of=q0a] at (3.5,4.9) (q1a) {$m_1$};
			
			\path[-stealth, thick]
			(q0a) edge[loop left]
			node[left,align=left]
			{$\texttt{0}|\texttt{0}:
				\theta_{0\mid 0}$\\
				$\texttt{1}|\texttt{0}:
				1-\theta_{0\mid 0}$}
			()
			
			(q0a) edge[bend left]
			node[above,align=left]
			{$\texttt{0}|\texttt{1}:
				\theta_{2\mid 1}$\\
				$\texttt{1}|\texttt{1}:
				1-\theta_{2\mid 1}$}
			(q1a)
			
			(q1a) edge[bend left]
			node[below,align=left]
			{$\texttt{0}|\texttt{0}:
				\theta_{0\mid 2}$\\
				$\texttt{0}|\texttt{1}:
				\theta_{0\mid 3}$\\
				$\texttt{1}|\texttt{0}:
				1-\theta_{0\mid 2}$\\
				$\texttt{1}|\texttt{1}:
				\theta_{1\mid 3}$}
			(q0a)
			
			(q1a) edge[loop right]
			node[right,align=left]
			{$\texttt{0}|\texttt{1}:
				\theta_{2\mid 3}$\\
				$\texttt{1}|\texttt{1}:
				1-\theta_{0\mid 3}
				-\theta_{1\mid 3}
				-\theta_{2\mid 3}$}
			();
			
			\node at (9.3,7) {(b)};
			\node at (12.6,7) {agent model ansatz 2};
			\node[state] (q0b) at (12.2,4.9) {$m_0$};
			\node[state, right of=q0b] at (13.2,4.9) (q1b) {$m_1$};
			
			\path[-stealth, thick]
			(q0b) edge[loop left]
			node[left,align=left]
			{$\texttt{0}|\texttt{0}:
				\theta_{0\mid 0}$\\
				$\texttt{1}|\texttt{1}:1$}
			()
			
			(q1b) edge[bend left]
			node[below,align=left]
			{$\texttt{0}|\texttt{0}:
				\theta_{0\mid 2}$\\
				$\texttt{1}|\texttt{1}:1$}
			(q0b)
			
			(q0b) edge[bend left]
			node[above,align=left]
			{$\texttt{0}|\texttt{0}:
				\theta_{2\mid 0}$\\
				$\texttt{1}|\texttt{0}:
				1-\theta_{0\mid 0}
				-\theta_{2\mid 0}$}
			(q1b)
			
			(q1b) edge[loop right]
			node[right,align=left]
			{$\texttt{0}|\texttt{0}:
				\theta_{2\mid 2}$\\
				$\texttt{1}|\texttt{0}:
				1-\theta_{0\mid 2}
				-\theta_{2\mid 2}$}
			();
			\node[anchor= west] at (0,-0.3) {\includegraphics[width=0.4\textwidth]{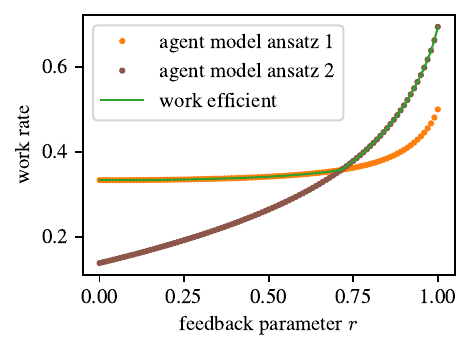}};
			\node[anchor= west] at (9,-0.3) {\includegraphics[width=0.4\textwidth]{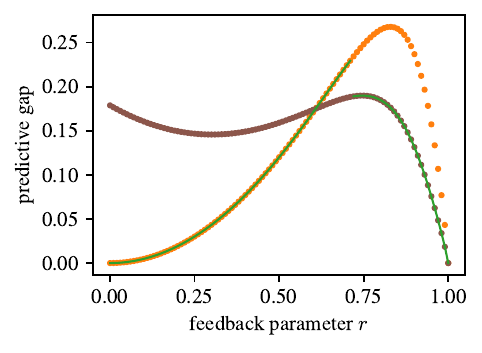}};
			\node at (0.,2.2) {(c)};
			\node at (9.3,2.2) {(d)};
		\end{tikzpicture}
		\caption{Two restricted ansätze for the agent model, and their work rates and predictive gaps. 
			We consider two ansätze for 2-state agent models. Panel (a) shows the finite-state machine of ansatz~1; panel (b) shows the finite-state machine of ansatz~2. In panels~(a) and~(b), an edge from $m_k$ to $m_\ell$
			labeled $\mathtt{a}|\mathtt{s}:q$ represents the transition
			$(S_t,M_t)=(\mathtt{s},m_k)$ to
			$(A_{t+1},M_{t+1})=(\mathtt{a},m_\ell)$ with probability $q$.
			We use the ordering
			$0=(\mathtt{0},m_0)$, $1=(\mathtt{1},m_0)$,
			$2=(\mathtt{0},m_1)$, and $3=(\mathtt{1},m_1)$ for both
			the input and output alphabets of the transition matrix. Thus,
			$\theta_{j\mid i}$ abbreviates
			$\theta^{\mathrm{agt}}(a_{t+1},m_{t+1}\mid s_t,m_t)$,
			where $i$ indexes $(s_t,m_t)$ and $j$ indexes
			$(a_{t+1},m_{t+1})$. Panel (c) displays the best unconstrained numerical work rate (green, as in \Cref{fig:9}a of the main text), together with the numerically optimized work rates of ansatz~1 (orange) and ansatz~2 (brown). Panel (d) shows the corresponding predictive gaps. For $r\lesssim0.71$, the unconstrained numerical result is indistinguishable within numerical precision from ansatz~1 in both quantities. Near $r\approx0.71$, the two ansatz work-rate curves intersect; for $r\gtrsim0.71$, the unconstrained result is indistinguishable within numerical precision from ansatz~2. The difference between the two ansatz predictive gaps matches the observed jump within numerical precision.}  \label{fig:TM_Transition}
	\end{figure}
	
\end{document}